\documentclass[a4paper,10pt]{article}
 \usepackage{amsmath}
 \usepackage{amsthm}
 \usepackage{amsfonts}
 \usepackage{amssymb}
 \usepackage{graphicx}
 \usepackage{hyperref}
 \usepackage{xcolor}

\usepackage[top=2.00cm,bottom=3.00cm,left=2.00cm,right=2.00cm]{geometry}

 \newcommand{\mac}{\mathcal}
 \newcommand{\f}{\mathbb}
 \newcommand{\vf}{\varphi}
 \newcommand{\ol}{\overline}
 \newcommand{\cu}{\subseteq}
 \newcommand{\wt}{\widetilde}
 \newcommand{\ve}{\varepsilon}
 \newcommand{\R}{\mathbb{R}}
 \newcommand{\lp}{\left(}
 \newcommand{\rp}{\right)}

 \DeclareMathOperator{\rank}{rank}
 \DeclareMathOperator{\conv}{conv}
 \DeclareMathOperator{\vol}{vol}
 \DeclareMathOperator{\cone}{cone}

 \newtheorem{definition}{Definition}
 \newtheorem{theorem}{Theorem}
 \newtheorem{assumption}{Assumption}
 \newtheorem{notation}{Notation}
 \newtheorem{corollary}{Corollary}
 \newtheorem{lemma}{Lemma}
 \newtheorem{remark}{Remark}

\definecolor{brightpink}{rgb}{1.0, 0.0, 0.5}

\title{Robustness of Minimum-Volume Nonnegative 
Matrix Factorization \\ under an Expanded Sufficiently Scattered Condition} 

\author{Giovanni Barbarino \qquad  Nicolas Gillis\qquad Subhayan Saha \thanks{Emails: \{giovanni.barbarino, nicolas.gillis, subhayan.saha\}@umons.ac.be. Authors acknowledge the support  by the European Union (ERC consolidator, eLinoR, no 101085607).  
 } \\ 
 Department of Mathematics and Operational Research \\ 
University of Mons, Rue de Houdain 9, 7000 Mons, Belgium   
	}
\date{}

\begin{document}

\maketitle

\begin{abstract}
Minimum-volume nonnegative matrix factorization (min-vol NMF) has been used successfully in many applications, such as hyperspectral imaging, chemical kinetics, spectroscopy, topic modeling, and audio source separation. However, its robustness to noise has been a long-standing open problem. In this paper, we prove that min-vol NMF identifies the groundtruth factors in the presence of noise  under a condition referred to as the expanded sufficiently scattered condition which requires the data points to be sufficiently well scattered in the latent simplex generated by the basis vectors.  
\end{abstract}

\section{Introduction}

Let $\{x_1,x_2,\dots,x_n\}\cu \f R^m$ be a dataset, and 
let  $X = [x_1,x_2,\dots,x_n]\in \f R^{m\times n}$ be the corresponding matrix whose columns are the data points $x_i$'s. 
An approximation of $X$ as the product of two smaller matrices, $W\in \f R^{m\times r}$ and  $H\in \f R^{n\times r}$  with $r\ll \min\{n,m\}$, such that $X \approx WH^\top$ 
  gives us insight on the information contained in $X$. 
  In fact, low-rank approximations are a central tool in data analysis, being equivalent to linear dimensionality reductions techniques, with PCA and the truncated SVD as the workhorse approaches~\cite{vidalgeneralized16, udell2016generalized, markovsky2012low}.  
  
  However, due to the sheer number of possible such decompositions, the information provided is hardly interpretable. This motivated researchers to introduce more constrained low-rank approximations. Among them, nonnegative matrix factorization (NMF) focuses on nonnegative input matrices $X$ and imposes the factors, $W$ and $H$, to be nonnegative entry-wise. 
Nonnegativity is motivated by \emph{physical constraints}, such as nonnegative sources and activations in 
hyperspectral imaging~\cite{bioucas2012hyperspectral}, 
chemometrics~\cite{de2006multivariate} and 
audio source separation~\cite{ozerov2009multichannel}, 
and 
by \emph{probabilistic modeling}, 
such as topic modeling~\cite{lee1999learning, AGHMMSWZ13} and unmixing of independent distributions~\cite{kubjas2015fixed}.  
Moreover, NMF leads to an easily-interpretable and part-based representation of the data~\cite{lee1999learning}. 
See also~\cite{cichocki2009nonnegative, xiao2019uniq, Gil20} and the references therein. 

\paragraph{Geometric interpretation of NMF} In the exact case, when $X = WH^\top$, up to a preprocessing of the matrix $X \geq 0$ that normalizes the columns of $X$ to have unit $\ell_1$ norm, it is possible to assume without loss of generality that $H$ has stochastic rows, that is, $He=e$ where $e$ is the vector of all ones of appropriate dimension. 
This condition allows a simple geometric interpretation of the decomposition: every data point, $x_i = WH(i,:)^\top$, is a convex combination of the $r$ columns of $W$, since $H(i,:) \geq 0$ and $\sum_k H(i,k) = 1$. 
Hence the convex hull of the $x_i$'s, denoted as $\conv(X)$, is contained in $\conv(W)$. 
Notice that even if the number of vertices of $\conv(X)$ may be as large as $n$, the number of vertices of $\conv(W)$ is instead at most $r\ll n$. Such a decomposition is called a simplex-structured matrix factorization (SSMF) \cite{lin2018maximum, abdolali2021simplex}.  

\paragraph{Minimum-volume NMF} The existence of an exact SSMF alone is not sufficient to ensure the uniqueness of a polytope $\conv(W)$ with $r$ vertices containing all the $x_i$'s. 
In fact, we can typically generate an infinite number of such decompositions just by enlarging $\conv(W)$, as long as it remains within the nonnegative orthant. 
As a consequence, 
researchers have looked for solutions with additional constraints,  sparsity being among the most popular one~\cite{hoyer2004non, kim2007sparse, gillis2012sparse}. 
Another approach, motivated by geometric considerations, looks for a 
minimum-volume solution, trying to make the basis vectors as close as possible to the data points. 
In particular, it considers  the following optimization problem, referred to as  minimum-volume (min-vol) NMF:   
\begin{equation}\label{eq:optminvol} 
    \min_{W \in \mathbb{R}^{m \times r}, H \in \mathbb{R}^{n \times r}} \; \vol(W) 
    \quad \text{such that } 
    \quad X = W H^\top, 
    H e = e,  
     \text{ and } H \geq 0, 
\end{equation}
where $\vol(W) := \det(W^\top W)$ is the squared volume of the polytope whose vertices are the columns of $W$ and the origin, up to the factor $1/r!$. 
Note that this problem is equivalent to 
\begin{equation*} 
    \min_{W \in \mathbb{R}^{m \times r}} \; \vol(W) 
    \quad \text{such that} \quad
    \conv(X)\cu \conv(W). 
\end{equation*} 


\begin{remark}[Nonnegativity of $X$ and $W$] \label{rem:nonneg} The nonnegativity of $X$ and $W$ has been removed from~\eqref{eq:optminvol}. 
The main reason is twofold: 
\begin{enumerate}
    \item It makes the problem more general, sometimes referred to as semi-NMF~\cite{ding2008convex, gillis2015exact}, or 
    ``finding a latent simplex''~\cite{bhattacharyya2020near, bakshilearning} or ``learning high-dimensional simplices''~\cite{najafi2021statistical, saberi2025fundamental}.  

    \item Nonnegativity of $W$ is not useful in most identifiability proofs of min-vol NMF; see the next paragraph. 
\end{enumerate} 
Hence, there is a slight abuse of language when referring to min-vol NMF since, in such decompositions, $W$ could potentially have negative entries, although the variant where $W$ is imposed to be nonnegative is often used in practice. 
The reason is that these models appeared in the NMF literature, hence authors kept the name NMF, although using the term semi-NMF would have been more appropriate. We refer the interested reader to the discussions in~\cite[Chapter~4]{Gil20} for more details. 
We will focus in this paper on the case where \emph{$W$ is not imposed to be nonnegative}. 
\end{remark}



\paragraph{Identifiability of min-vol NMF} 

Identifiability for min-vol NMF was proved in~\cite{FMHS15, lin2015identifiability}: 
if $X \in \mathbb R^{m \times n}$ admits a decomposition  $X = W^{\#} (H^{\#})^{\top}$ where $H^{\#} \in \f R^{n \times r}_+$ satisfies the \textit{sufficiently scattered condition (SSC)} and $r = \rank(X)$, then the optimal solution $(W^*,H^*)$ of \eqref{eq:optminvol} is  \textit{identifiable}, that is, it is unique and $W^*$ coincides with $W^\#$ up to a permutation of its columns. 
In particular, this implies that there exists a unique minimum-volume  polytope $\conv(W^*)$ with $r$ vertices containing $\conv(X)$, and it coincides with $\conv(W^\#)$. 

We will provide a formal and detailed definition of SSC in 
Section~\ref{sec:ssc}. 
The geometric intuition is that  a row stochastic matrix $H \in \f R^{n \times r}_+$ is SSC whenever $\conv(H^\top)$ contains the hyper-sphere $\mac Q_{\sqrt{r-1}}$ that is internally tangent to the unit simplex $\Delta^r := \{ x \ | \ x \geq 0, e^\top x = 1 \}$, that is, 
\[
\mac Q_{\sqrt{r-1}} = \left\{   x\in \Delta^r \ \Big | \ x = \frac er + w, \ \| w\|^2\le \frac{1}{r-1}-\frac 1r       \right\} \ \cu \ 
\conv(H^\top).
\]
This is illustrated on the left image of Figure~\ref{fig:SSCvsPSSC}.  
In this case, we say that $H$ is sufficiently scattered inside $\Delta^r$. Equivalently, this requires that the data points $x_i$'s are sufficiently scattered inside $\conv(W)$. 
The SSC requires some sparsity in $H$, since rows of $H$ must be located on the boundary of the unit simplex; in fact, one can show that $H$ requires at least $r-1$ zeros per column~\cite{xiao2019uniq}.  
\begin{figure}
    \centering
    \includegraphics[width=.9\linewidth]{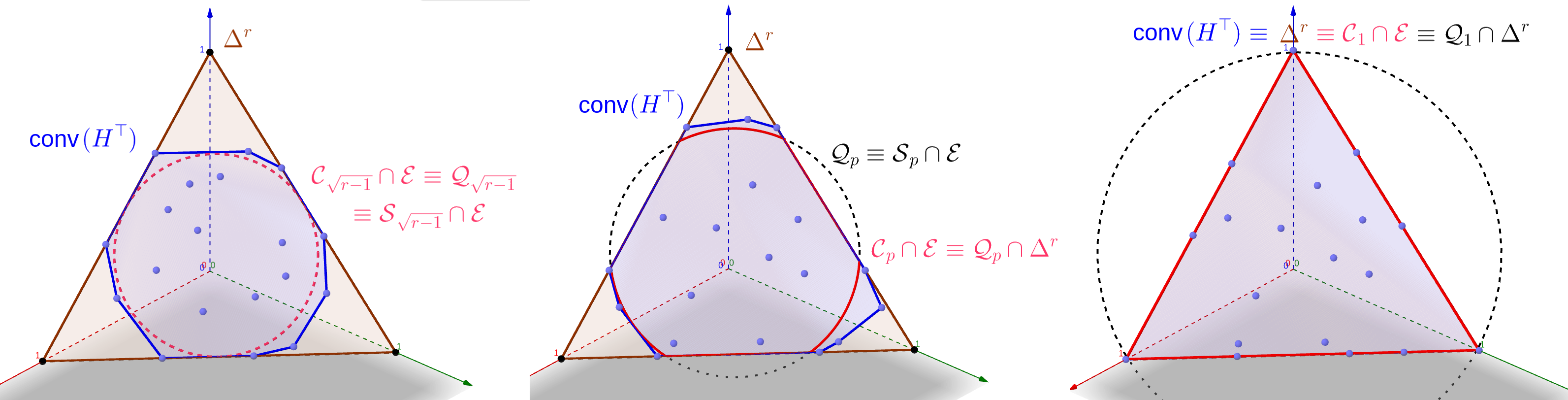}
    \caption{
    Geometric intuition for SSC on the left, $p$-SSC on the center with $1< p< \sqrt{r-1}$, and separability on the right. 
    Visualization on the unit simplex $\Delta^r$ in the case $r=3$ and for $H$ row stochastic. 
    }
    \label{fig:SSCvsPSSC}
\end{figure}

\begin{remark}[Relaxation of the sum-to-one constraint]
The sum-to-one constraint on the rows of $H$, $He=e$, 
can be relaxed to the normalization $H^\top e=e$~\cite{fu2018identifiability}, or to $W^\top e = e$~\cite{leplat2020blind}; see the discussion in~\cite[Chapter~4]{Gil20} for more details. In this paper, we focus on the sum-to-one constraint $He=e$, that is, we focus on simplex-structured matrix factorizations. 
\end{remark}

\paragraph{Importance of min-vol NMF} 

This intuitive idea behind min-vol NMF was introduced 
in hyperspectral unmixing~\cite{full1981extended, craig1994minimum} and analytical chemistry~\cite{perczel1991convex, perczel1992analysis}; 
see also~\cite{miao2007endmember, chan2009convex, rajko2010comments, zhang2017robust} and the references therein.  
Given a set of spectral signatures (that is, fractions of light reflected  depending on the wavelength), 
 the goal is to recover the spectra of the materials present in the image or chemical reaction (the columns of $W$) and their abundances in these signatures (the rows of $H$). 
Since then, it has been used in many different contexts, including 
 topic modeling~\cite{Fuetal19, jang2019minimum}, 
 blind audio source separation~\cite{leplat2020blind, wang2021minimum}, 
crowd sourcing~\cite{ibrahim2019crowdsourcing}, 
recovering joint probability~\cite{ibrahim2021recovering}, 
label-noise learning~\cite{li2021provably}, deep constrained clustering~\cite{nguyen2023deep},  dictionary learning~\cite{hu2023dico}, and tensor decompositions~\cite{sun2023volume, saha2025identifiability}.  

Min-vol NMF is also motivated by statistical considerations: 
if we assume that the rows of $H$ follow a uniform Dirichlet distribution, min-vol NMF is the maximum likelihood estimator~\cite{nascimento2011hyperspectral, jang2019minimum, wu2021probabilistic}; 
this is closely related to the latent Dirichlet allocation model in topic modeling~\cite{blei2003latent}.

\paragraph{Open question: presence of noise and robustness of min-vol NMF} 

Despite its importance in applications, the identifiability of min-vol NMF has only been studied in noiseless scenarios. 
In the presence of noise, one can ask for an approximated decomposition, that is, where the norm of $X-WH^\top$ is smaller than a certain tolerance level $\ve\ge 0$. We thus turn to the following  min-vol NMF problem
   \begin{equation}\label{eq:apprminvol} 
   \min_{W \in \mathbb{R}^{m \times r}, H \in \mathbb{R}^{n \times r}} \  \det(W^\top W)  
   \quad 
   \text{such that} 
    \quad \|X-WH^\top\|_{1,2}\le \ve, \  
    He = e,   
    \text{ and } H\ge 0,   
\end{equation} 
where the norm $\|A\|_{1,2} = \max_j \|a_j\|$ is the maximum Euclidean norm of the columns of $A$.

The main objective of this article is to study the solution of \eqref{eq:apprminvol} and characterize under which conditions it is possible to recover $W^\#$ and $H^\#$, up to some controlled error, 
from 
  \begin{equation}\label{eq:assumption1} 
   X = W^\# (H^\#)^\top + N^\#,
  \end{equation}
where $ (H^\#)^\top$ is column stochastic, $W_\#$ is full rank, and $\|N^\#\|_{1,2}\le \ve$. 
This is, to the best of our knowledge, an important open question in the NMF literature~\cite{xiao2019uniq, Gil20}. Let us quote~\cite{lin2015identifiability}: 
\begin{quote}
    The whole work has so far assumed the noiseless case, and sensitivity in the
noisy case has not been touched. These challenges are left as future work. 
\end{quote}

It is known that the SSC alone is not enough to  \textit{robustly} recover $W^\#$  by solving~\eqref{eq:apprminvol}: for any $\ve > 0$, there exist a matrix $X_\ve$ respecting \eqref{eq:assumption1}, but for which the optimization problem \eqref{eq:apprminvol} has an optimal  solution $(W^*,H^*)$ far from the ground truth $(W^\#,H^\#)$~\cite{lin2015identifiability}.

This problem is closely related to the problem of learning high-dimensional simplices in noisy regimes~\cite{najafi2021statistical, saberi2025fundamental}. In~\cite{saberi2025fundamental}, it is mentioned that 
\begin{quote}
    the minimum-volume simplex estimator
proposed by Najafi et al.~(2021)~\cite{najafi2021statistical} can become highly inaccurate in the presence of noise. In
high dimensions (that is, when $r \gg 1$), the corrupted samples are likely to fall outside the
true simplex, leading to significant estimation errors. 
\end{quote}  
In this paper, we mitigate this issue by allowing approximate solutions, via the constraint $\|X-WH^\top\|_{1,2}\le \ve$, while proving robustness of the solution recovered by min-vol NMF~\eqref{eq:apprminvol}.



\paragraph{Expanded SSC} 

Since the SSC is not enough in the presence of noise, we must define a more restrictive condition for the matrix $H^\#$, and we use the expanded SSC or $p$-SSC. 
We say that $H$ is $p$-SSC with $1 \leq p\leq \sqrt{r-1}$ if 
\[
\mac C_p:= \left\{x \in \f R_+^r \ \big| \ e^{\top}x \geq p \|x\| \right\} \; \subseteq \; \cone\left(H^\top\right). 
\] 
We will discuss in depth this property in Section~\ref{sec:pssc}, but the geometric intuition is that 
a row stochastic matrix $H \in \f R^{n \times r}_+$ is $p$-SSC whenever $\conv(H^\top)$ contains $\mac Q_p\cap \Delta^r$, where $\mac Q_p$ is an enlarged version  of the hyper-sphere $\mac Q_{\sqrt{r-1}}$. We have 
\[
\mac Q_p \cap \Delta^r = \left\{   x\in \Delta^r \ \Big | \ x = \frac er + w, \ \| w\|^2\le \frac{1}{p^2}-\frac 1r       \right\} \ \cu \ 
\conv(H^\top).
\]
This is illustrated on the right image of Figure~\ref{fig:SSCvsPSSC}.  
It is possible to prove that for any $p< \sqrt {r-1}$, a $p$-SSC matrix is in particular SSC, and any SSC matrix  $H$ is a limit of $p$-SSC matrices for $p\to \sqrt {r-1}$; see Section~\ref{sec:pssc} for more details.  
Note that the notion of $p$-SSC is equivalent to the notion of \emph{uniform pixel purity level} introduced in~\cite{lin2015identifiability}; see Section~\ref{subs:uniformpurlev}.

\paragraph{Summary of our main contributions}

Our main results show that if $X = W^\#(H^\#)^\top + N^\#$ admits a decomposition as in \eqref{eq:assumption1} where $H^\#$ is $p$-SSC for $p< \sqrt {r-1}$, then the solution of min-vol NMF~\eqref{eq:apprminvol}  \textit{robustly} identifies $(W^\#,H^\#)$, up to some error depending on 
the perturbation level $\ve$, 
the value of $p$, 
and the conditioning of $W^\#$. 

In order to formally write our main results (Theorems~\ref{th:main} and~\ref{th:main2} below), 
let us define our assumptions rigorously.  
\begin{assumption}
   \label{ass:perturbed_pSSC} The matrix $X\in \f R^{m\times n}$ admits the decomposition   
   \[
   X = W^\#(H^\#)^\top + N^\#,
   \]
   where the involved matrices satisfy the following:
   \begin{itemize}
       \item $H^\#\in \f R^{n\times r}$ is row stochastic and $p$-SSC with $r\ge 2$,
       \item   $1\le p<\sqrt {r-1}$ for $r > 2$, and $p=1$ for $r=2$,
       \item the rank of $W^\#\in \f R^{m\times r}$ is $r\ge 2$, that is, the $r$th singular values of $W^\#$ is positive,  $\sigma_r(W^\#) > 0$, 
       \item $N^\#\in \f R^{m\times n}$ and $\|N^\#\|_{1,2}\le \ve$ for a constant $\ve >0$.
   \end{itemize}
 
\noindent We denote $(W^*,H^*)$ an optimal solution of the following min-vol NMF problem
\[
\min_{W \in \mathbb{R}^{m \times r}, H \in \mathbb{R}^{n \times r}} \  \det(W^\top W)  \quad \text{such that} 
    \quad \|X-WH^\top\|_{1,2}\le \ve, \  He = e, 
    \text{ and } H\ge 0,  
\]
and let $N^* := X-W^*H^{*\top}$ and  $q := \sqrt{r-p^2}$.
\end{assumption}

Note that the case $r=1$ is trivial, since every column of $X$ is equal to the unique column of $W^\#$, up to the noise level. 
We can now state our main results as follows.

\begin{theorem}
  \label{th:main}  Under Assumption~\ref{ass:perturbed_pSSC}, there exist absolute positive constants $C_\ve,C_e>0$   such that if the  level of perturbation $\ve$ satisfies
    \[
    \ve  \; \le \;  C_\ve
\lp \min\{q,\sqrt 2\} -1\rp^2
\frac {\sigma_r(W^\#)}{r^{9/2}}\frac {q^2}{p^2}, 
    \]
    then    
    \begin{align*}    
 \min_{\Pi\in \mac P_r} \|W^\#-W^*\Pi\|_{1,2} \; \le  \;    C_e \ \|W^\#\|     \sqrt {
     \frac { \ve}{\min\{q^2-1,1\}}
   \frac{ r^{7/2}}{\sigma_r(W^\#)}\frac {p^2}{q^2}, 
   }
\end{align*}
where $\|W^\#\|$ is the matrix $\ell_2$-norm of $W^\#$, and 
$\mac P_r$ is the set of $r\times r$ permutation matrices.
\end{theorem}

For $p^2$ approaching $r-1$, that is, when we approach the classical SSC, the parameter $q^2$  tends to $1$, and therefore the allowed level of perturbation $\ve$ goes to zero because of the term $\left(\min\{q,\sqrt 2\} -1\right)$, meaning that any small  perturbation might totally modify the solution of min-vol NMF~\eqref{eq:apprminvol}. Moreover all the bounds get better as $p$ gets smaller, that is, as the $p$-SSC gets stronger.

 The case $p=1$ is the best and strongest assumption that we can impose on the ground truth solution, and in the literature this is called the \textit{separability} condition~\cite{donoho2004does, AGKM11}. 
 In geometrical terms,  a row stochastic matrix $H \in \f R^{n \times r}_+$ is separable (or $1$-SSC) whenever $\conv(H^\top)=\Delta^r$, meaning that $\conv(X) = \conv(W)$, that is, the columns of $W$ are samples from the columns of $X$. This is the so-called pure-pixel assumption in hyperspectral imaging~\cite{bioucas2012hyperspectral}, 
 and the anchor-word assumption in topic modeling~\cite{AGKM11}. 
In this case,  when $H^\#$ is $p$-SSC with $p$ close enough to $1$, the error dependence on the perturbation improves from $\sqrt\ve$ to $\ve$, as shown in our second main theorem. 
\begin{theorem}
   \label{th:main2} Under Assumption~\ref{ass:perturbed_pSSC}, there exist absolute positive constants $C_\ve,C_p, C_e>0$   such that
   if the level of perturbation $\ve$ and the parameter $p$ satisfy
    \[
    \ve \le C_\ve
    \frac{\sigma_r(W^\#)   }{r\sqrt r}, 
    \quad \text{ and } \quad p\le 1 + C_p
    \frac 1r,
    \]
    then 
    \[   \min_{\Pi\in \mac P_r} \|W^\#-W^*\Pi\|_{1,2}  \; \le \;  C_e  \|W^*\| \left(     \frac{r\sqrt r}{\sigma_r(W^\#)}   \ve  + r(p-1)\right), 
\] 
where $\mac P_r$ is the set of $r\times r$ permutation matrices.
\end{theorem}

We will compare these bounds with the error bounds of separable NMF algorithms specifically designed for the case $p=1$; see Section~\ref{sec:sepNMFcompa}.

\paragraph{Recovery of $H^\#$} We focus in this paper on the identifiability of $W^\#$, as most previous works. The reason is that once $W^\#$ is identified, $H^\#$ can be recovered by solving a linearly constrained least squares since $W^\#$ is full rank. 
In our case, since $W^\#$ and $W^*$ are close to each other, and  
$W^* (H^*)^\top + N^* = W^\# (H^\#)^\top + N^\#$, we have  
\[
(H^\#)^\top = (W^\#)^\dagger \left( W^* (H^*)^\top \right) + (W^\#)^\dagger \left(N^\# - N^* \right), 
\] 
where $(W^\#)^\dagger \in \mathbb{R}^{r \times m}$ denotes the pseudoinverse of $W^\#$, hence  
\[
(H^\# - H^*)^\top
= 
(W^\#)^\dagger \left( W^{*} - W^\# \right) (H^*)^\top  
+ (W^\#)^\dagger \left(N^\# - N^* \right), 
\] 
so 
\[
\left\|(H^* - H^\#)^\top \right\|_{1,2}  
\; \leq \; \frac{1}{\sigma_r(W^\#)} \left(\|W^\# - W^{*}\|_{1,2} + 2\ve \right),  
\]  
using the facts that $\|H^*\|_1 = 1$,  
$\|(W^*)^\dagger\| = \frac{1}{\sigma_r(W^\#)}$, 
and  
the matrix norm inequalities from Lemma~\ref{lem:norm_inequalities}, namely $\|ABC\|_{1,2} \le \|A\|\|B\|_{1,2}\|C\|_1$ for any matrices $(A,B,C)$ of appropriate dimensions.

\paragraph{Outline of the paper} 

In Section~\ref{sec:pssc}, we define the $p$-SSC, 
discuss its geometric interpretation, 
show that it trivially implies identifiability of min-vol NMF in the noisless case, 
make the connection between the SSC and separability, 
and 
provide an important necessary condition.  
In Section~\ref{sec:robustnearsep}, we provide a sketch of the proof of Theorem~\ref{th:main2}, and, in Section~\ref{sec:generalrobust}, a sketch of the proof of Theorem~\ref{th:main}. Our goal in these two sections is to provide the high-level ideas of the proofs to make the paper more pleasant to read.  
Most of the technical details of the proofs are postponed to the Appendix.

\paragraph{Notation} 

Given a vector $x \in \mathbb{R}^m$, we denote $\|x\|$ its $\ell_2$ norm. 
Given a matrix 
$X \in \mathbb{R}^{m \times n}$, we denote $X^\top $ its transpose, 
its $i$th column by $x_i$, 
its $i$th row by $\tilde x_i$, 
its entry at position $(k,i)$ by $x_{k,i}$, 
$X(:,\mathcal{K})$ the submatrix of $X$ whose columns are indexed by $\mathcal K$, $X(\mathcal{K},:)$ similarly for the rows,  
$\|X\| = \sigma_{\max}(X)$ its $\ell_2$ norm which is equal to its largest singular value, 
$\|X\|_F^2 = \sum_{i,j} X(i,j)^2$ its squared Frobenius norm, 
$\sigma_r(X)$ its $r$th singular value,  
 $\rank(X)$ its rank. For $m=n$, we denote $\det(X)$ its determinant. 
 
We denote $e_k$ the $k$th unit vector, $I$ the identity matrix, $e$ the vector of all ones, and $E$ the matrix of all ones, 
all of appropriate dimension depending on the context.   
The set $\mathbb{R}^{m \times n}_+$ denotes the $m$-by-$n$ component-wise nonnegative matrices. A matrix $H  \in \mathbb{R}^{n \times r}$ has stochastic rows if $H \geq 0$ and $H  e = e$. 

Given $W \in \mathbb{R}^{m \times r}$, the convex hull generated by the columns of $W$ is denoted $\conv(W) = \{ Wh \ | \ e^\top  h = 1, h \geq 0\}$, the cones it generates by $\cone(W) = \{ Wh \ | \ h \geq 0 \}$, and its volume as $\vol(W) = \det(W^\top W)$; this is a slight abuse of language since $\sqrt{\vol(W)}/r!$ is the volume of the polytope whose vertices are the columns of $W$ and the origin, within the subspace spanned by the columns of $W$, given that $\rank(W) = r$. The pseudoinverse of $W$ is denoted $W^\dagger \in \mathbb{R}^{r \times m}$.

Given an integer $r$, we denote the set of integers from 1 to $r$ as $[r] = \{1,2,\dots,r\}$.  
Given disjoint sets $\mac A_1,\dots, \mac A_t$, their disjoint union is denoted as $\sqcup_i \mac A_i$.

\section{Expanded SSC: definition and properties} 
\label{sec:pssc}

In this section, we first define the expanded SSC and discuss its geometric interpretation including in the dual space  (Section~\ref{sec:pssc_def}). We then link it with the separability condition and the SSC (Section~\ref{subs:separability_SSC}), show how it implies identifiability of min-vol NMF in the noiseless case (Section~\ref{sec:psscminvol}), and finally a necessary condition for the expanded SSC to be satisfied (Section~\ref{subs:suffnec}).

\subsection{Definition and geometry} \label{sec:pssc_def}

Let us formally define the expanded SSC, which was introduced in~\cite{saha2025identifiability} in the context of the identifiability of nonnegative Tucker decompositions in order to show that the Kronecker product of two $p$-SSC matrices is SSC.    
\begin{definition}\label{def:SSCexpanded}[Expanded SSC ($p$-SSC)] 
Let $H \in \R_+^{n \times r}$, $r\ge 2$ and $p \geq 1$. The matrix $H$ satisfies the $p$-SSC if  $$
 \mathcal{C}_p := \left\{x \in \R_+^r \ \big| \ e^{\top}x \geq p \|x\| \right\} \; \subseteq \; \cone\left(H^\top\right). 
 $$
\end{definition} 
In order to explain the geometric intuition behind the definition, we first need a nice way to visualize the cone $\mac C_p$. First of all, $\mac C_p$ is the intersection of an ice-cream cone $\mac S_p$ with the positive orthant $\f R_+^r$, where 
\[
 \mac S_p := \left\{x \in \R^r \ \big| \ e^{\top}x \geq p \|x\| \right\}, \qquad \mac C_p = \mac S_p\cap \f R^r_+.
\]
Noteworthy examples are the cases $p = 1$ and $p = \sqrt{r-1}$: 
\begin{itemize}
    \item for $p = 1$, $\mac S_1$ is the smallest ice cream cone with central axis along the vector $e$ and containing $\f R_+^r$, meaning that $\mac C_1 = \f R^r_+$;
    
    \item  for $p=\sqrt{r-1}$, $\mac S_{\sqrt{r-1}} = \mac C_{\sqrt{r-1}}$ is the largest ice cream cone with central axis along the vector $e$ and contained in $\f R_+^r$. 
\end{itemize}

Notice that any nonzero $x\in \mac S_p$ satisfies $e^\top x\ge p\|x\| > 0$, meaning that $x$ is a positive multiple of a vector $y$ such that $e^\top y  = 1$. Since the same holds for $\mac C_p$, and both are cones, it follows that $\mac S_p = \cone(\mac S_p\cap \mac E)$ and $\mac C_p = \cone(\mac C_p\cap \mac E)$, where $\mac E$ is the affine subspace $\mac E:=\{x\ |\ e^\top x = 1\}$.


By restricting to the space $\mac E$, we can show that $\mac Q_p := \mac S_p \cap \mac E$ is an hyper-sphere relative to the space $\mac E$ with center in the vector $e/r$. Moreover, $\Delta^r = \f R_+^r\cap \mac E$, so $\mac C_p \cap \mac E= \mac S_p\cap \f R^r_+ \cap \mac E = \mac Q_p \cap \Delta^r$, and 
\[
\mac C_p \cu \cone(H^\top) 
\iff  \cone(\mac C_p\cap \mac E)\cu \cone(H^\top)
\iff \mac C_p\cap \mac E \cu \cone(H^\top)
\iff \mac Q_p\cap \Delta^r \cu \cone(H^\top), 
\]
showing that it is possible to test the $p$-SSC of a nonnegative matrix $H$ by looking at what happens on $\mac E$, and, in particular, at the relation between $\mac Q_p$ and $\cone(H^\top)$.

For any nonnegative $H$, if we renormalize the nonzero rows to have unit sum and call the resulting matrix $\wt H$, then  $\cone(H^\top) = \cone(\wt H^\top)$.  The cone $\cone(\wt H^\top)$ is now the conic hull of $\cone(\wt H^\top) \cap\mac E = \conv(\wt H^\top)$, so the above relation is equivalent to  $\mac Q_p \cap \Delta^r \cu \conv(\wt H^\top)$.
In other words, up to a renormalization, we can always rewrite the $p$-SSC as a containment condition between two convex sets on $\Delta^r$. 

We can visualize the relations between the various sets involved in Figure~\ref{fig:SSCvsPSSC}, and we collect some of their properties in the following result. The proof is postponed to Appendix~\ref{app:proof of Lemma on Q_p}.

\begin{lemma}
\label{lem:p-hypersphere}
Define the hyper-sphere $\mac Q_p$ with center $e/r$   and contained in the affine subspace $\mac E=\{x\ |\ e^\top x = 1\}$ as \[
 \mac Q_p := \left\{   x\in \mac E \ \Big | \ x = \frac er + w, \ \| w\|^2\le \frac{1}{p^2}-\frac 1r       \right\}.
 \]
For every $p\ge 1$, it holds that $\mac S_p \cap \mac E =\mac Q_p$, and thus  $\mac C_p \cap \mac E = \mac Q_p \cap \Delta^r$. 
As a consequence,
\begin{itemize}
    \item a row stochastic matrix $H \in \R_+^{n \times r}$ is $p$-SSC if and only if $\mac Q_p\cap \Delta^r \cu \conv(H^\top)$.
    \item a nonnegative matrix $H \in \R_+^{n \times r}$ is $p$-SSC if and only if $\mac Q_p\cap \Delta^r \cu \cone(H^\top)$.
\end{itemize}

\noindent 
The set $\mac Q_p$ shrinks as $p$ gets larger, and makes $\mac C_p \cap \Delta^r$ phase between three different behaviors:
\begin{itemize}
    \item For  $1\le p< \sqrt {r-1}$, the convex set  $\mac C_p \cap \Delta^r$ has  mixed curvilinear-polyhedral boundary. In particular, 
\[ \partial\mac Q_1 \cap \Delta^r = \{e_1,e_2,\dots e_r \}, \] so $\mac Q_1$ is exactly the hyper-sphere circumscribed to the hyper-tetrahedron $\Delta^r$.
    \item For  $ \sqrt {r-1}\le p< \sqrt r$, the hyper-sphere $\mac Q_p$  is contained in $\Delta^r$, so   $\mac C_p \cap \Delta^r=\mac Q_p$ is a hyper-sphere. 
    In particular,   \[ \mac Q_{\sqrt{r-1}} \cap \partial\Delta^r = \left\{\frac{e - e_i}{r-1} \ \Big |\ i=1,\dots, r \right\}, \]  so  $\mac Q_{\sqrt{r-1}}$ is exactly the hyper-sphere inscribed to the hyper-tetrahedron $\Delta^r$.
    \item For  $ \sqrt {r}<  p$, the hyper-sphere $\mac Q_p$  is empty, so $\mac C_p \cap \Delta^r =\emptyset$.
    In particular, $\mac Q_{\sqrt{r}}$ is a degenerate hyper-sphere consisting only of the point $e/r$. 
\end{itemize}
\end{lemma}

The $p$-SSC for $1 \leq p \leq \sqrt{r-1}$ has been introduced in order to bridge between the classical SSC and the separability condition. In fact, we have SSC for any $p<\sqrt{r-1}$ and we have separability when $p=1$. 
In Section~\ref{subs:separability_SSC}, we reintroduce the two concepts and discuss in detail the relations between the different conditions.
Before doing so, we explore the geometric interpretation of the $p$-SSC in the dual space.

\subsubsection{Geometric interpretation in the dual space}
\label{sec:geom_int_dual_pSSC}

Let us recall the notion of dual cone. 
\begin{definition}[Dual Cone]
     For any cone $\mac F$, its dual is defined as
\begin{align}
    \mac F^* 
    & = 
    \left\{ y \ |\ x^{\top}y \geq 0 \text{ for all } x \in \mac F   \right\}.\label{eq:dual_general_cone} 
\end{align}
If $\mac F$ is the cone generated by the columns of a  matrix $A \in  \R^{m \times n}$, then
\[
\mac F^* =  \cone^*(A) 
     = 
    \left\{y \ |\ A^{\top}y \geq 0 \right\}.
\]
\end{definition}
Some key properties of duality of closed convex cones are as follows: 
\begin{itemize}
    \item The dual of a closed convex cone is a closed convex cone.
    \item It inverts the containment relations, that is,  $\mac F\cu \mac G \iff \mac G^*\cu \mac F^*$.
    \item The dual of intersection is the sum of the duals, that is, $(\mac F\cap \mac G)^* = \mac F^* + \mac G^*$.
\end{itemize}

It is easy to show that the dual of the cone $\mac S_p$ is the cone $\mac S_q$ where $r = p^2+q^2$. Since $\mac C_p$ in the interval of interest $p^2\in (1,r-1)$ is a convex cone with partly linear and partly curvilinear boundary, its dual will have the same kind of boundary. Since $\mac C_p = \mac S_p \cap \f R^r_+$, by the property of duality,  $\mac C_p^* = \mac S_q + \R^r_+ = \cone(\mac Q_q \cup \{e_1,\dots, e_r\})$. In particular, we can visualize $\mac C_p^*$ on $\mac E$ as the convex hull of $\mac Q_p$ and the vectors $e_1,\dots,e_r$. 

 Figure~\ref{fig:dual_pSSC} shows the shape of $\mac C_p^*,\mac S_p^*$ and $\mac C_p, \mac S_p$ on  $\mac E$ in dimension $r=3$. In the following lemma, we summarize the above discussion about the dual cones and we refer 
 to Appendix~\ref{app:proof_of_lem_CpSp_dual_cones} for the proof.

\begin{figure}
    \centering
    \includegraphics[width=.9\linewidth]{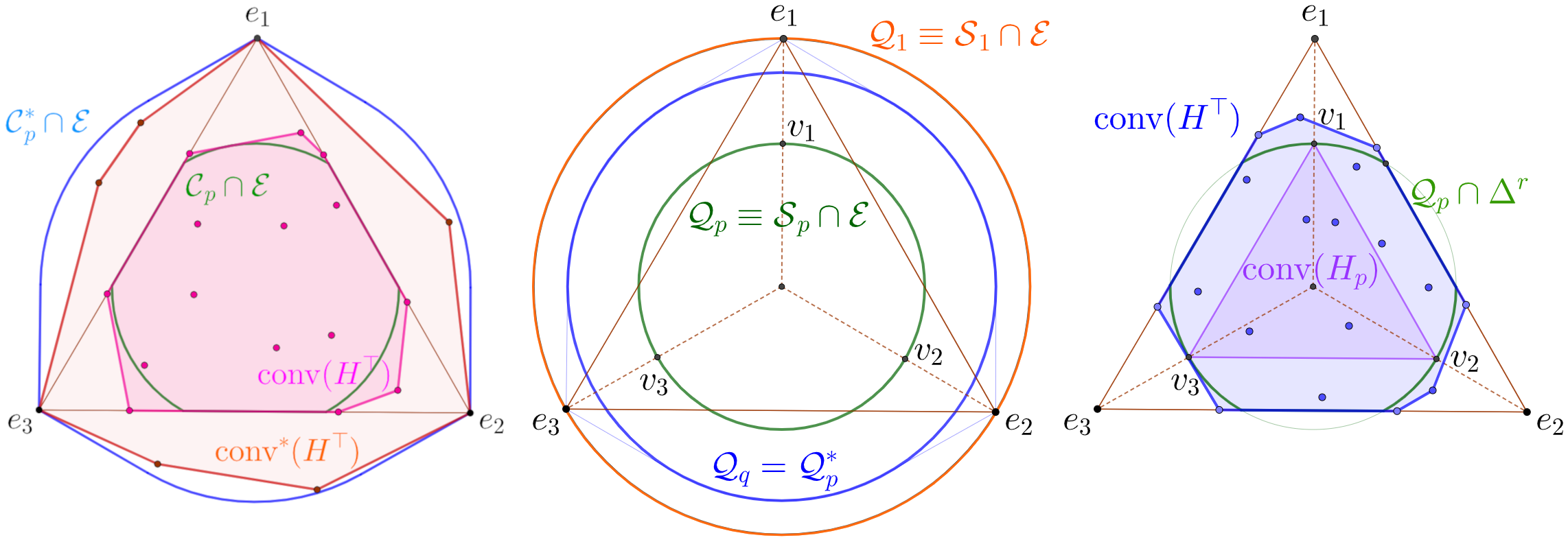}
    \caption{On the left and center, $\mac C_p^*,\mac S_p^*$ and $\mac C_p, \mac S_p$ on  $\mac E$ in dimension $r=3$ for $1<p<\sqrt{r-1}$. 
    On the left, the containments between $\conv(H^\top)$, $\conv^*(H^\top)$, $\mac C_p\cap\mac E$ and $\mac C_p^*\cap \mac E$  for a row stochastic and $p$-SSC matrix $H$. 
    On the right, the points $v_i$, their convex hull $\conv(H_p)$  and $\conv(H^\top)$ for a row stochastic and $p$-SSC $H$. }
    \label{fig:dual_pSSC}
\end{figure}

\begin{lemma}
\label{lem:Cp_Sp_dual_cones_and_inclusions}    
Suppose $r> 2$ and $p\in [1,\sqrt{r-1}]$, with $q = \sqrt{r-p^2}$. Given $\mac Q_p$ from Lemma~\ref{lem:p-hypersphere} and the cones 
\begin{align*}
               \mac S_p &= \left\{x \in \R^r \ \big| \ e^{\top}x \geq p \|x\| \right\} , \qquad 
                \mac C_p = \left\{x \in \R_+^r \ \big| \ e^{\top}x \geq p \|x\| \right\}= \mac S_p\cap \R_+^r,
    \end{align*}
their dual cones according to \eqref{eq:dual_general_cone} are
\begin{align*}
              \mac S_p^* = \mac S_q, \qquad 
               \mac C_p^* = \mac S_q + \f R^r_+ = \cone(\mac Q_q \cup \{e_1,\dots, e_r\})
               , \qquad 
               \mac C_p^*\cap \mac E= \conv(\mac Q_q \cup \{e_1,\dots, e_r\}).
    \end{align*}
\end{lemma}

Using the properties of duality, we can formulate an equivalent definition for $p$-SSC. 
\begin{corollary}
    \label{cor:def_pSSC_dual}
   A matrix $H \in \R_+^{n \times r}$ satisfies the $p$-SSC if and only if   
   $$
    \cone^*\left(H^\top\right) \; \cu \;  \mac C_p^*, 
 $$
 or,  equivalently, 
 $$
    \cone^*\left(H^\top\right) \cap \mac E \; \cu \;  \conv(\mac Q_q \cup \{e_1,\dots, e_r\}).
 $$
\end{corollary}

\subsubsection{An equivalent formulation: uniform pixel purity level} \label{subs:uniformpurlev}

The $p$-SSC condition is equivalent to the so-called \textit{uniform pixel purity level} $\gamma$ defined in \cite{lin2015identifiability}. 
Given a row-stochastic matrix $H$, its uniform pixel purity level  $\gamma$ is  defined as follows: 
\[
\gamma := \sup \Big\{ s\le 1 \ \Big| \    B_s \cap \Delta^r \cu \conv\big(H^\top \big) \Big\} \quad \text{ where } \quad B_s= \{x\in \f R^r \ | \ \|x\|\le s \}.
\]
Notice that, by Lemma~\ref{lem:p-hypersphere}, 
\[
B_s\cap \Delta^r = \left\{  x\in \Delta^r \ | \ \|x\| \le s  \right\}
= \left\{  x\in \Delta^r \ \Big| \ x = \frac er + w, \  \|w\|^2 \le s^2-\frac 1r  \right\}
= \mac Q_{1/s}\cap \Delta^r
= \mac C_{1/s}\cap \Delta^r.
\]
Thus $B_s\cap \Delta^r\cu \conv(H^\top)\iff \mac C_{1/s}\cu \cone(H^\top)$, meaning that a  row-stochastic matrix $H$ satisfies  $p$-SSC if and only if its uniform pixel purity level is at least $\gamma \ge 1/p$.

\subsection{Links with SSC and Separability} \label{subs:separability_SSC}

We now link in more details $p$-SSC with two key conditions in the NMF literature: separability and the SSC.

\subsubsection{Separability}

The notion of separability dates back to the hyperspectral community where it is called the pure-pixel assumption~\cite{boardman1995mapping}. It requires that for each pure material present in the image, there exists a pixel containing only that material. 
The terminology was introduced by Donoho and Stodden~\cite{donoho2004does}, and it was later used by Arora et al.~\cite{AGKM11} to obtain unique and polynomial-time solvable NMF problems; see~\cite[Chapter~7]{Gil20} for a survey on separable NMF. In the context of topic modeling, it was referred to as the anchor-word assumption~\cite{AGHMMSWZ13} and requires that, for each topic, there exists a word that is only used by that topic. 
Let us formally define separability. 
\begin{definition}\label{def:separablematricesfirst}
A matrix $H \in \R^{n \times r}_+$ is called separable if there exists an index set $\mathcal{K} \subseteq [n]$  where $|\mathcal{K}| = r$ such that $H(\mathcal{K},:)$ is a diagonal matrix with positive diagonal elements. 
\end{definition}
Equivalently, a matrix $H$ is \textit{separable} if the convex cone generated by its rows spans the entire nonnegative orthant, that is, $\cone(H^\top) = \mathbb{R}^r_+$. 
 See the right image on Figure~\ref{fig:SSCvsPSSC} for a visualization.
 We say that $X$ admits a \textit{separable NMF} $(W,H)$ of size $r$ if there exists a decomposition of $X$ of the form $X = WH^\top$ of size $r$ such that $H \in \mathbb{R}^{n \times r}$ is a separable matrix. 
 This implies that, up to scaling, $W = X(:,\mathcal{K})$ for some index set $\mathcal K$, that is, the columns of $W$ are a subset of the columns of $X$. 
 In geometrical terms, if $H$ is row stochastic, then $r$ of its rows must be the vectors $e_1,\dots, e_r$, that is, the vertices of the unitary simplex $\Delta^r$. Equivalently, we would have $\conv(H^\top) = \Delta^r$ or $\cone(H^\top) = \f R^r_+$. It is possible to prove that it is also equivalent to say that $H$ is $1$-SSC. 
 \begin{corollary}
     A matrix $H \in \R^{n \times r}_+$ is separable if and only if it is $1$-SSC. \end{corollary}
 \begin{proof}
     For any $x\ge 0$, $(e^\top x)^2 \ge \|x\|^2$. As a consequence, 
     $$  \mathcal{C}_1 = \left\{x \in \R_+^r \ \big| \ e^{\top}x \geq  \|x\| \right\} = \f R^r_+.
     $$    
     It follows that $H$ is $1$-SSC if and only if $\f R ^r_+ = \mac C_1 \cu \cone (H^\top) \cu \f R^r_+$, that is, $\cone(H^\top) = \f R^r_+$. But the conic hull of a set of nonnegative points is $ \f R^r_+$ if and only if $r$ of the points coincide with positive multiples of $e_1,\dots,e_r$, that is, there must exists an index set $\mathcal{K} \subseteq [n]$  where $|\mathcal{K}| = r$ such that $H(\mathcal{K},:)$ is a diagonal matrix with positive diagonal elements.  \end{proof}

\subsubsection{The sufficiently scattered condition (SSC)} 
\label{sec:ssc}

The separability assumption is relatively strong. To relax it, a crucial notion is the \textit{sufficiently scattered condition (SSC)}, which was introduced in~\cite{donoho2004does}; see also~\cite{huang2013non}.   

\begin{definition}[Sufficiently scattered condition (SSC)\footnote{Slight variants of the SSC exist in the literature. Refer to Section 4.2.3.1 in \cite{Gil20} for a more detailed description and the relation between these variants.}]
\label{def:SSCalternate} A matrix $H \in \R_+^{n \times r}$ with $r\ge 2$ satisfies the SSC if the following two conditions hold:  
\begin{enumerate}
    \item SSC1: $\mathcal{S}_{\sqrt{r-1}} \subseteq \cone(H^\top)$.
    
    \item SSC2: $\cone^*(H^\top) \cap \partial \mathcal{S}_1 =  \{\lambda e_k  \ | \  \lambda \geq 0  \text{ and } k \in [r]\}$.
\end{enumerate}
\end{definition}

SSC1 requires that $\cone(H^\top)$ contains the ice-cream cone $\mac S_{\sqrt{r-1}}$ that is tangent to every facet of the nonnegative orthant, or equivalently it requires that $\cone(H^\top)$ contains the hypersphere $\mac Q_{\sqrt {r-1}}$ inscribed to the unit simplex $\Delta^r$. See the left image on Figure~\ref{fig:SSCvsPSSC} for a visualization. 

SSC2 is typically satisfied if SSC1 is, and allows one to avoid pathological cases; see~\cite[Chapter 4.2.3]{Gil20} for more details. 
Using duality, we prove in 
Appendix~\ref{app:proof_equivalence_SSC2} that it is possible to rewrite SSC2 as follows: 
\[
 \partial \cone (H^\top) \cap \mac S_{\sqrt{r-1}} 
 =  \left\{\lambda \frac{e-e_k}{r-1}  \ | \  \lambda \geq 0  \text{ and } k \in [r]\right\}. 
\]
Restricting to the unit simplex $\Delta^r$ and considering a row stochastic matrix $H$, the above formula can be interpreted geometrically as follows: the boundary of $\conv(H^\top)$ must intersect $\mac Q_{\sqrt {r-1}}$  uniquely on the boundary of $\Delta^r$. 
Notice that from Lemma~\ref{lem:p-hypersphere}, we know that  $\Delta^r\cap \mac Q_{\sqrt {r-1}}$  is exactly the set of  $\frac{e-e_i}{r-1}$ for $i=1,\dots,r$.   In particular this means that $\mac Q_{\sqrt {r-1}}$ can be enlarged to $\mac Q_{p}\cap \Delta^r$ for some $p<\sqrt{r-1}$ and it will still be contained in $\conv(H^\top)$. 

Putting together the two conditions of SSC with the definition of $p$-SSC, the following relation holds.

\begin{lemma}[\cite{expandedSSC}] 
    \label{lem:pSSC_SSC}  
      For $r>2$,  a matrix $H \in \R^{n \times r}_+$ is SSC if and only if  $H$ is $p$-SSC for some $p<\sqrt {r-1}$. 
      
      \noindent For $r=2$, SSC, separability and $1$-SSC coincide.      
\end{lemma}

\subsection{Min-vol NMF and identifiability in the noiseless case} \label{sec:psscminvol} 

To take advantage of the SSC and $p$-SSC, we need the notion of volume. 
Given a matrix $W \in \mathbb{R}^{m \times r}$ with $\rank(W) = r$, the quantity $\vol(W) = \det(W^\top W)$ is a measure of the volume of the columns of $W$; namely,  
$\frac{1}{r!} \sqrt{\vol(W)}$ is the volume of the convex hull of the columns of $W$ and the origin in the linear subspace spanned by the columns of $W$. 
We have the following identifiability result. 
\begin{theorem}[Identifiability of min-vol NMF]\cite{FMHS15, lin2015identifiability} \label{thm:idenminvol}
Let $X \in \mathbb R^{m \times n}$ admit the decomposition  $X = W_{\#} H_{\#}^{\top}$ where $H_{\#} \in \R^{r \times n}_+$ is row stochastic and  satisfies the \textit{SSC}, and $r = \rank(X)$. 
Then for any optimal solution  $(W_*,H_{*})$ of \begin{equation}\tag{\ref{eq:optminvol}} 
    \min_{W \in \mathbb{R}^{m \times r}, H \in \mathbb{R}^{n \times r}} \; \vol(W)  
    \quad \text{such that } 
    \quad X = W H^\top, 
    H e = e,  
     \text{ and } H \geq 0, 
\end{equation}
there exists a permutation matrix $\Pi$ such that $W_* = W_\# \Pi$ and $H_* = H_\#  \Pi$.
\end{theorem}
In simple terms, the SSC of $H_\#$ in an NMF $X = W_\#H_\#^\top$ with $H_\# e = e$ implies that there exist no other factorization where the first factor has a smaller volume. 
In particular, this also means that for any SSC decomposition of $X=WH^\top$ with row stochastic $H$, we have $\conv(W) \equiv \conv(W_\#)$, that is, the matrices $W$ and $W_\#$ coincide up to a permutation of the columns. 

By Lemma~\ref{lem:pSSC_SSC}, the same holds whenever $H$ is row stochastic and $p$-SSC for $p^2<r-1$ (or for $p=1$ when $r=2$). 


\begin{corollary}
\label{cor:pSSC_identifiable_with_minvol}    Let $X \in \mathbb R^{m \times n}$ admit the decomposition  $X = W_{\#} H_{\#}^{\top}$ where $H_{\#} \in \R^{r \times n}_+$ is row stochastic, satisfies the $p$-SSC with $p\in [1,\sqrt{r-1})$ (or $p=1$ for $r=2$),  and $r = \rank(X)$. 
Then for any optimal solution  $X = W_*H_{*}^{\top}$ of \eqref{eq:optminvol}
there exists a permutation matrix $\Pi$ such that $W_* = W_\# \Pi$ and $H_* = H_\#  \Pi$.
\end{corollary}

\subsection{Necessary conditions for $p$-SSC and the matrix $H_p$} \label{subs:suffnec}

We now provide a necessary condition for the $p$-SSC to hold which will be instrumental in our robustness proofs.

Define the vectors $v_i$ as the intersection between the boundary of the cone $\mac S_p$ and the segment connecting $e/r$ and $e_i$; see Figure~\ref{fig:dual_pSSC} for an illustration.  
Their coordinates can be computed as follows: 
\[
v_i = \alpha_p e + (1-r\alpha_p)e_i, \qquad  \alpha_p 
     = \frac 1r \left(1 - \frac 1{\sqrt{r-1}}\frac qp\right).
\]
We define $H_p\in \f R^{r\times r}$ the matrix whose columns are the vectors $v_i$, that is, 
\[H_p^\top = H_p = 
\begin{pmatrix}
    v_1 &\dots & v_r
\end{pmatrix}
= \alpha_p E + (1- r\alpha_p) I, 
\] 
where $E = ee^\top$ is the matrix of all-ones of appropriate dimension. 
Observe that, by construction, the columns of $H_p$ are contained both in $\mac S_p = \left\{x \in \R^r \ \big| \ e^{\top}x \geq p \|x\| \right\}$ and in $\Delta^r$, so 
\[
\conv(H_p) \; \cu \; \mac S_p \cap \Delta^r \; = \;  \mac Q_p\cap \Delta^r.  
\] 
Moreover, $\conv(H_1) \equiv  \Delta^r$. As a consequence, for a row stochastic $p$-SSC matrix $H\in \f R^{n\times r}$, one has 
\[
\conv(H_p)  \; \cu \; \mac Q_p\cap \Delta^r  \; \cu \;  \conv(H^\top)  \; \cu \;  \Delta^r. 
\]
This implies that $\conv(H^\top)$ must necessarily contain all vectors $v_i$, and the containment $\conv(H_p)\cu \conv(H^\top)$ becomes an equality for $p=1$, that is, when $H$ is separable. 
The following lemma, whose proof is in Appendix~\ref{app:Hp}, 
summarizes the above discussion. 
\begin{lemma}
\label{lem:def_of_Hp}  
Fix  $p\in [1, \sqrt{r-1}]$ and let $q = \sqrt{r-p^2}$. Let $H_p\in \f R^{r\times r}$ be the matrix whose columns are  the intersection between  the boundary of the cone $\mac S_p$ and the segment connecting $e/r$ and $e_i$. Then 
    \[H_p^\top = H_p 
= \alpha_p E + (1- r\alpha_p) I, \qquad  \alpha_p 
     = \frac 1r \left(1 - \frac 1{\sqrt{r-1}}\frac qp\right).
\]
Any $p$-SSC matrix $H\in \f R^{n\times r}$ must necessarily satisfy $\cone(H_p)\cu \cone(H^\top)$ and if $H$ is also row stochastic then $\conv(H_p)\cu \conv(H^\top)$.
\end{lemma}

Due to its simple structure, the last singular value of $H_p$ and the norm of its inverse can be calculated exactly, and they will be central quantities in the proofs for our main results; 
see Appendix~\ref{app:Hp} for their closed form and some useful lower and upper bounds. 

Note that, as opposed to SSC and $p$-SSC, the condition $\conv(H_p)\cu \conv(H^\top)$ is easy to check: we just need to verify whether each column of $H_p$ can be written as a convex combination of the rows of $H$, which is a linear system of equalities and inequalities. 
Another necessary condition for the SSC was proposed 
in~\cite[p.~119]{Gil20} (see also~\cite{Gillis_2024}): it requires $\cone(H^\top)$ to contain the tangent points of $\mathcal C_{\sqrt{r-1}}$ on $\Delta^r$, that is, the columns of $(E-I)/(r-1)$.


\section{Robustness under near-separability ($p$-SSC with $p$ close to 1)} \label{sec:robustnearsep}

In this section, we provide a sketch of the proof for Theorem~\ref{th:main2}, giving the basic intuitions behind it. The full proof with all the details is postponed to Appendix~\ref{app:proof_of_theo_sep}. 
We start with some initial results that will be also useful to prove the other main theorem, Theorem~\ref{th:main}, in Section~\ref{sec:generalrobust}.

\subsection{First steps for the proofs of Theorem~\ref{th:main} and~\ref{th:main2}} \label{sec:initial_common_steps}


Let us recall the main notation from Assumption~\ref{ass:perturbed_pSSC}.
Our data matrix $X\in \f R^{m\times n}$ admits the $p$-SSC decomposition
\[
   X = W^\#(H^\#)^\top + N^\#,
   \]
where $W^\#\in \f R^{m\times r}$ is full rank, $H^\#\in \f R^{n\times r}$ is $p$-SSC and row stochastic, $p\in [1,\sqrt{r-1})$  (or $p=1$ for $r=2$) and $\|N^\#\|_{1,2}\le \ve$. The matrix $X$ also admits a different decomposition
\[
   X = W^*(H^*)^\top + N^*,
   \]
where the pair $(W^*,H^*)$ is an optimal solution to the min-vol problem
\begin{equation}
    \tag{\ref{eq:apprminvol}}
   \min_{W \in \mathbb{R}^{m \times r}, H \in \mathbb{R}^{n \times r}} \  \det(W^\top W)  \quad \text{such that} 
    \quad \|X-WH^\top\|_{1,2}\le \ve, \  He = e,   \text{ and } H\ge 0, 
\end{equation}
and hence $\|N^*\|_{1,2} = \|X - W^*(H^*)^\top\|_{1,2}\le \ve$.

Our main results will bound $\min_{\Pi \in \mac P_r}\|W^\# - W^* \Pi\|_{1,2}$, where 
 $\Pi\in\f R^{r\times r}$ is a permutation matrix used to permute the columns of $W^*$ in order to match them with the closest columns of $W^\#$.

\subsubsection{The matrix $R$ linking $W^\#$ and $W^*$}

In order to find a relation linking $W^\#$ and $W^*$, we will use the fact that each column of the invertible matrix $H_p$ introduced in Lemma~\ref{lem:def_of_Hp} is a convex combination of the rows of the $p$-SSC  matrix $H^\#$, that is, $H_p = (H^\#)^\top V$, where $V\in\R_+^{n\times r}$ is column stochastic. 
As a consequence, 
\begin{equation}\label{eq:relation_W*_Whash_with_V}
    W^*(H^*)^\top V + N^*V= XV = W^\#(H^\#)^\top V + N^\# V =  W^\# H_p + N^\# V, 
\end{equation}
and hence 
\[
W^\#  = W^*(H^*)^\top V H_p^{-1} + (N^*-N^\#)VH_p^{-1}  = W^* R + M, 
\]
where we define $R := (H^*)^\top V H_p^{-1}$ and $M := (N^*-N^\#)VH_p^{-1}$. 
Using the inequality $\|AB\|_{1,2} \le \|A\|_{1,2}\|B\|_1$ for any matrix $A$ and $B$ of appropriate dimension (see Lemma~\ref{lem:norm_inequalities}), we obtain 
\[
\| W^\# - W^*R\|_{1,2} = \|M\|_{1,2} \le \|N^*-N^\#\|_{1,2} \|V\|_1 \|H_p^{-1}\|_1 \le  2\ve \|H_p^{-1}\|_1, 
\]
where, by Lemma~\ref{lem:last_sv_of_Hp}, $\|H_p\|^{-1}\le 2\sqrt r \frac pq\le 2r$ for every  $p\in [1,\sqrt{r-1}]$. 
We summarize the above discussion in the following result. 
\begin{lemma}
 \label{lem:definition_of_R}   Under Assumption~\ref{ass:perturbed_pSSC}, there exists a column stochastic $V\in\f R^{n\times r}$ such that $(H^\#)^\top V = H_p^\top$ and 
 \[
W^\# = W^*R + M, 
\]
where  $R := (H^*)^\top V H_p^{-1}\in \f R^{r\times r}$ and $M := (N^*-N^\#)VH_p^{-1}\in \f R^{m\times n}$. Moreover,
\[
e^\top R =e^\top, \qquad \|M\|_{1,2} \le   2\ve \|H_p^{-1}\|_1 \le 2\ve\lp 2\sqrt{r-1} \frac pq - 1\rp\le 2\ve (2r-3) \leq 4 r \ve.
\]
\end{lemma}

Both the remainders of the proofs for the main results focus uniquely on estimating how far the matrix $R$ is from a permutation matrix. 
In fact, this can then be used to compute a bound on the target error $\min_{\Pi \in \mac P_r} \|W^\# - W^*\Pi\|_{1,2}$ as follows. 
\begin{corollary}
    \label{cor:Error_estimation_given_R_is_permutation}
Under the assumptions and the notation of Lemma~\ref{lem:definition_of_R}, 
    \begin{align*}   
    \min_{\Pi \in \mac P_r} \|W^\#-W^*\Pi\|_{1,2} &\le    \|W^*\|\min_{\Pi \in \mac P_r} \|R-\Pi\|_{1,2}    +4r\ve.  
\end{align*}
\end{corollary}

\begin{proof}
  We have  
  \begin{align*} 
  \min_{\Pi \in \mac P_r} \|W^\#-W^*\Pi\|_{1,2} & =   
   \min_{\Pi \in \mac P_r} \|W^*R + M-W^*\Pi\|_{1,2} \le    \|W^*\|\min_{\Pi \in \mac P_r} \|R-\Pi\|_{1,2}    +4r\ve , 
\end{align*}
where we used $\|M\|_{1,2} \leq 4r \ve$ (Lemma~\ref{lem:definition_of_R}), 
and the matrix inequality $\|W^* (R-\Pi)\|_{1,2} \leq \|W^*\| \|R-\Pi\|_{1,2}$ (see Lemma~\ref{lem:norm_inequalities}). 
\end{proof}

\subsubsection{The volume of $R$ is lower bounded }

The matrix $W^*$ is the optimal solution of the min-vol NMF problem \eqref{eq:apprminvol}, while the $p$-SSC decomposition $X =  W^\#(H^\#)^\top + N^\#$ is  a feasible solution. 
As a consequence, the volume of $W^*$ will necessarily be smaller than the volume of $W^\#$. 

From Lemma~\ref{lem:definition_of_R}, $W^\# = W^*R + M$ where $M$ is a small perturbation. As a consequence, $R$ will need to act as an enlarger of the volume of $W^*$ in order to get it on par with the volume of $W^\#$. In particular, the volume of $R$ itself cannot be less than $1$ minus a perturbation term coming from $M$.  
This means that 
\[
\vol(W^\#)= \vol(W^*R + M)\approx \vol(W^*R)  = \vol(W^*)\vol(R)\le  \vol(W)\vol(R)
\implies \vol(R) \ \wt{\ge}\  \frac{\vol(W^\#)}{\vol(W^*)} \ge 1.
\]
The precise result is as follows. Its proof can be found in Appendix~\ref{app:proof_of_lower_bound_detR}. 

\begin{lemma}
    \label{lem:lower_bound_detR}
    Under Assumption~\ref{ass:perturbed_pSSC}, if
    \[  \ve 
 = \mac O \left(
    \frac{\sigma_r(W^\#)   }{r^2}
\frac {q}p \right) , 
    \]
   then the matrix $R$ in Lemma~\ref{lem:definition_of_R} satisfies
   \[
    \det(R)^2 \ge  1 -\mac O\left(     \frac{r^2}{\sigma_r(W^\#)}   \frac pq \ve  \right).
   \]
\end{lemma}

\subsubsection{$W^*$ is full rank}

The min-vol NMF problem \eqref{eq:apprminvol} aims to find the decomposition $X = WH^\top + N$ with the minimum volume $\vol(W) = \det(W^\top W)$. 
In particular, it may lower the rank of $W$ in order to get the volume equal to zero. 
When $\ve = 0$, that is, no noise, 
 Corollary~\ref{cor:pSSC_identifiable_with_minvol} shows that the $p$-SSC solution $X = W^\#(H^\#)^\top$ is also the only optimal solution to the min-vol NMF problem \eqref{eq:apprminvol}, up to permutation of columns. 
 In particular  
\[
 W^\# \ \text{ full rank } \implies 
\det((W^*)^\top W^*) = \det((W^\#)^\top W^\#) > 0.
\]
The introduction of a small perturbation $\|N\|_{1,2} = \ve > 0$ does not usually impact the rank of $W^*$, except in the case when $W^\#$ is already close to be rank deficient, that is, when its last singular value $\sigma_r(W^\#)$ is close to zero. As a consequence, we need an upper bound on $\ve$ in function of $\sigma_r(W^\#)$ to ensure that $W^*$ is full rank. 
\begin{lemma}
    \label{lem:full_rank_W} Under Assumption~\ref{ass:perturbed_pSSC}, 
    \[
    \sigma_r(W^*) \ge 
     \frac{\sigma_r(W^\#)   }{\sqrt {r(r-1)}}
\frac qp
-  2\ve. 
    \]
As a consequence  if  
\[
\ve 
< 
    \frac{\sigma_r(W^\#)   }{2\sqrt {r(r-1)}}
\frac {q}p, 
    \] 
    then $W^*$ is full rank. 
\end{lemma}

\noindent Its proof is in Appendix~\ref{app:full_rank_W} and uses~\eqref{eq:relation_W*_Whash_with_V} that gives  
\[
\sigma_r(W^*) \ge \frac{\sigma_r(W^*(H^*)^\top V)}{\|(H^*)^\top V\|} 
\ge  \frac{\sigma_r(W^\# H_p^\top) -  \| (N^*-N^\#)V \| }{\sqrt{r}}
\ge \frac{\sigma_r(W^\#)\sigma_r( H_p) }{\sqrt{r}}
-2\ve.
\]
Notice that the bound on $\ve$ in Lemma~\ref{lem:full_rank_W} gets stricter the more we approach $p=\sqrt{r-1}$ (the original SSC condition). 
In fact, 
 \begin{itemize}
    \item for $H^\#$ SSC and $p = \sqrt{r-1}$, the bound reads $ \ve <
    \frac{\sigma_r(W^\#)   }{2(r-1)\sqrt {r}}$, 
  \item for $H^\#$ separable and $p = 1$, the bound reads  $  \ve<     \frac{\sigma_r(W^\#)   }{2\sqrt {r}}$.
\end{itemize}
This shows that for larger $p$, it is easier for the perturbation matrix $N^\#$ to induce a min-vol NMF optimal solution $W^*$ that is rank deficient, thus making the model less robust.

\subsection{Sketch of the proof of Theorem~\ref{th:main2}}

By Corollary~\ref{cor:Error_estimation_given_R_is_permutation}, it remains to prove that $R$ is close to a permutation matrix $\Pi$, or, equivalently, 
that every column $r_i$ of $R$ is close to some canonical basis vector $e_j$, and that two distinct columns are not close to the same $e_j$.

By Lemma~\ref{lem:definition_of_R}, the matrix $R$ is equal to $(H^*)^\top VH_p^{-1}$. For $p=1$, we have $H_p = H_p^{-1} = I$, and in particular all entries of $R$ are nonnegative. It stands to reason that when $p$ is close to $1$, then $H_p^{-1}$ is still close to $I$, and in particular its negative entries have small magnitude proportional to $p-1$. The same can thus be said for the matrix $R$.

A direct computation is enough to lower bound all entries of $R$ by a constant $-\beta_p = -\mac O(p-1)$. 
The proof, in Appendix~\ref{app:proof_to_lower_bound_R_sepcase}, only makes use of the column stochasticity of the matrix $(H^*)^\top V$ to carry on the computation for the entries of $R$. 
\begin{lemma}
  \label{lem:lower_bound_R_sepcase}  Suppose that $p - 1 = \mac O(1/r)$ where $p\in [1,\sqrt{r-1})$ and $q = \sqrt{r-p^2}$. All the entries of the matrix $R$ defined in Lemma~\ref{lem:definition_of_R} are lower bounded by $-\beta_p\le 0$ and
    \[
    \beta_p = \mac O(p-1) = \mac O\left(\frac 1r\right),  \qquad \|R\|_1\le 1 + 2\beta_p, \qquad \frac pq\sqrt r = \mac O(1)
    .
    \]
\end{lemma}

\noindent By Lemma~\ref{lem:definition_of_R},  the entries of each column $r_i$ of $R$ sum up to $1$, and, by Lemma~\ref{lem:lower_bound_detR}, the volume of $R$ is bounded below by $1$ up to a perturbation. Using the the Hadamard theorem $\det(R)^2\le \prod_i\|r_i\|^2$, we can write the conditions on the columns $r_i$ as follows:
\[
e^\top r_i = 1, \quad r_i\ge -\beta_p, \quad \|r_i\|_1\le 1 +2\beta_p  \quad \forall i, \qquad \prod_i\|r_i\|^2\ge   1 -\mac O\left(     \frac{r^2}{\sigma_r(W^\#)}   \frac pq \ve  \right).
\]
In the separable case $p=1$ with no perturbation $\ve = 0$, we would have that 
\[
e^\top r_i = 1, \ \ r_i\ge 0 \implies 1 = e^\top r_i \ge \|r_i\|^2\ge  \prod_i\|r_i\|^2\ge   1, 
\]
meaning that all inequalities are equalities, and, in particular, each $r_i$ must be a binary vector of norm 1, that is, a canonical basis vector $e_j$. Moreover, the condition $\det(R)^2\ge 1$ prevents having two distinct columns $r_i$  equal to the same $e_j$. 

In the presence of a perturbation ($\ve > 0$) and non-separability ($p> 1$), 
the proof is more involved. 
By Lemma~\ref{lem:lower_bound_R_sepcase}, 
the largest positive entry in $r_i$ is at most $\|R\|_1 \le 1+ 2\beta_p$, and at least
\[
  \|r_i\|_{\infty}
    \ge \frac{\|r_i\|^2}{\|r_i\|_1}
    = \frac{\prod_j\|r_j\|^2}{\|r_i\|_1\prod_{j\ne i}\|r_j\|^2}
    \ge \frac{1 -\mac O\left(     \frac{r^2}{\sigma_r(W^\#)}   \frac pq \ve  \right)}{(1+2\beta_p)^{2r-1}}
     \ge 1 -\mac O\left(     \frac{r^2}{\sigma_r(W^\#)}   \frac pq \ve  + r\beta_p\right).
\]
As a consequence, the largest positive entry in $r_i$ is close to $1$ up to a term depending on $\ve$ and $\beta_p = \mac O(p-1)$. 
Since $\|r_i\|_1\le 1 +2\beta_p$, the rest of its entries must be bounded in magnitude by a similar term depending on $\ve$ and $\beta_p$. This let us conclude that $r_i$ is indeed close to some canonical basis vector $e_j$, and again the condition $\det(R)^2\ \wt{\ge} \ 1$  precludes having two distinct columns $r_i$  equal to the same $e_j$.

This reasoning allows us to have a bound on $\|R-\Pi\|_{1,2}$ (see  Appendix~\ref{app:rest_of_proof_th_main2}), which we can plug in 
Corollary~\ref{cor:Error_estimation_given_R_is_permutation} and obtain the target bound on $\|W^\#-W^*\Pi\|_{1,2}$. The full proof with all the details can be found in Appendix~\ref{app:rest_of_proof_th_main2}.

\subsection{Comparison with robust separable NMF algorithms} \label{sec:sepNMFcompa}

Let us compare the bounds of Theorem~\ref{th:main2} in the case of separability, that is, $p=1$, to robust separable NMF algorithms specifically designed for this situation.  
In that case, Theorem~\ref{th:main2} tells us that 
    \[
\ve \leq \mathcal O\left(  \frac{\sigma_r(W^\#)}{r\sqrt r} \right) 
\quad \Rightarrow \quad 
\min_\Pi \|W^\#-W^*\Pi\|_{1,2} \leq 
 \|W^*\| \, 
 \mathcal O\left(     \frac{r\sqrt r}{\sigma_r(W^\#)}   \right) \, \ve. 
    \] 
There are two main classes of robust separable NMF algorithms: (1) greedy algorithms, and (2) convex-optimization based algorithms. 
Most greedy algorithms rely on the full-rank condition on $W^\#$, that is, $\sigma_r(W^\#) > 0$, like we do in this paper, although this is not a necessary condition for recovering the vertices of the convex full of a set of points. Among these algorithms, let us highlight two of them:  
\begin{itemize}
    \item The most famous and widely used one: the successive projection algorithm (SPA)~\cite{Araujo01} which is the workhorse algorithm and satisfies~\cite{spa2014} 
    \[
\ve \leq \mathcal \kappa(W^\#)^2 O\left( \frac{\sigma_r(W^\#)}{\sqrt{r}} \right) 
\quad \Rightarrow \quad 
\min_\Pi \|W^\#-W^*\Pi\|_{1,2} \leq \mathcal O\left(  \kappa(W^\#)^2   \right) \, \ve , 
    \]
    where $\kappa(W^\#) = \frac{\|W^\#\|}{\sigma_r(W^\#)} \geq 1$ is the condition number of $\kappa(W^\#)$.  
    The squared condition number of $W^\#$ can be relatively large, typically larger than $r$, and hence min-vol NMF will be more robust than SPA in these situations. 

   \item The most robust one: precondition SPA~\cite{mizutani2014ellipsoidal, gillis2015semidefinite} for which first robustness bounds were proved in~\cite{gillis2015semidefinite} and later improved in~\cite{mizutani2016robustness}:  
    \[
\ve \leq \mathcal O\left( \sigma_r(W^\#) \right) 
\quad \Rightarrow \quad 
\min_\Pi \|W^\#-W^*\Pi\|_{1,2} \leq \| W^\# \| \, \mathcal O \left( \frac{1}{\sigma_r(W^\#)}  \right) \, \ve .  
    \] 
    Hence preconditioned SPA is expected to be more robust than min-vol NMF, up to the factor $r\sqrt{r}$.  
\end{itemize}  
However, we do not know whether the bounds of Theorem~\ref{th:main2} are tight; this is a question for further research. 
We refer to~\cite{barbarino2025robustness} for a proof of tightness of the bounds above for SPA and preconditioned SPA, and to~\cite[p.~257]{Gil20} for a comparison of bounds of more robust separable NMF algorithms, including algorithms that do not rely on the full-rankness of $W$. 
Another interesting question for further research is the following: can we adapt min-vol NMF for rank-deficient cases? 
Computing the volume of a polytope which is not a simplex is non-trivial. 
For example, \cite{leplat2019minimum} proposed to use the practical measure $\det(W^\top W+\delta I)$ for some $\delta > 0$, but a proof of recovery and robustness remains elusive.

\section{General robustness under $p$-SSC} \label{sec:generalrobust}


In this section, we prove our second main theorem, Theorem~\ref{th:main}. 
Recall that the two matrices $W^\#$ and $W^*$ are related through the relation $W^\# = W^*R + M$, where the matrix $R$ is introduced in Lemma~\ref{lem:definition_of_R}, 
and $M$ is a perturbation matrix such that $\|M\|_{1,2}\le 4r\ve$. 

The aim of the Theorem~\ref{th:main} is to bound $\min_{\Pi \mac P_r}\|W^\# - W^*\Pi\|_{1,2}$. 
By Corollary~\ref{cor:Error_estimation_given_R_is_permutation}, we have seen that it is enough to estimate how close $R$ is to a permutation matrix $\Pi$ since 
 \begin{align*}   
    \min_\Pi \|W^\#-W^*\Pi\|_{1,2} &\le    \|W^*\|\min_\Pi \|R-\Pi\|_{1,2}    +4r\ve.  
\end{align*}
The focus is thus on the matrix $R$. We already know that $e^\top R = e^\top$,  and, by Lemma~\ref{lem:lower_bound_detR},  a lower bound on its volume is as follows:  
\begin{equation}
    \label{eq:lower_bound_volR}
    \det(R)^2 \ge  1 -\mac O\left(     \frac{r^2}{\sigma_r(W^\#)}   \frac pq \ve  \right).
\end{equation}
Keep in mind that now $p\ge 1$, and the quantity $p-1$ can be of the order of $\sqrt r$, so for example Lemma~\ref{lem:lower_bound_R_sepcase} would only tell us that each element of $R$ is lower bounded by $-\mac O(\sqrt r)$, which is too much since we want $R$ to approach a nonnegative permutation matrix when $\ve \to 0$. 

We thus need a different approach, so we choose to follow and generalize the original proof of Theorem~\ref{thm:idenminvol} in \cite{FMHS15} showing the identifiability of the min-vol solution under the SSC. 

\subsection{Properties of the rows $\wt r_i$ of $R$}

By substituting $W^\# = W^*R + M$ into 
\[ 
       W^*(H^*)^\top + N^* = X = W^\#(H^\#)^\top + N^\#, 
\] 
and multiplying on the left by the pseudoinverse of $W^*$, we get
\[
 R (H^\#)^\top = (H^*)^\top - (W^*)^\dagger(N^\# -  N^* + M(H^\#)^\top)  
 \ge - \|(W^*)^\dagger(N^\# -  N^* + M(H^\#)^\top)  \|_{1,2}\ge -  \gamma_p \ve,
\]
\[
\implies 
(R + \gamma_p\ve ee^\top) (H^\#)^\top \ge 0, \qquad \text{ where } \qquad 
\gamma_p =  \mac O\left( \frac{ r\sqrt r}{\sigma_r(W^\#)}\frac {p^2}{q^2}\right).
\]
Here we used that $H^*\ge 0$, the properties of the $(1,2)$ induced norm in Lemma~\ref{lem:norm_inequalities}, 
$\|(W^*)^\dagger\| = \sigma_r(W^*)^{-1}$, the bound on $\sigma_r(W^*)$ in Lemma~\ref{lem:full_rank_W} 
and the definition of $M$ in Lemma~\ref{lem:definition_of_R}. Denote by $\wt r_i$ the rows of $R$. 
Since $H^\#$ is $p$-SSC, 
by Corollary~\ref{cor:def_pSSC_dual}, 
\[
\cone(R^\top + \gamma_p\ve ee^\top) \cu \cone((H^\#)^\top)^*\cu \mac C_p^*\cu \mac C_1
\implies
\wt r_i \in \mac C_p^* - \gamma_p\ve e, \quad 
e^\top \wt r_i \ge \|\wt r_i\| -2r\gamma_p\ve.
\]
We collect these properties in a first Lemma and prove it in Appendix~\ref{app:proof_upper_bound_detR}. 

\begin{lemma}
\label{lem:upper_bound_detR}
Given the matrix $R$ in Lemma~\ref{lem:definition_of_R}, if $\ve  = \mac O(\frac {\sigma_r(W^\#)}{r}\frac qp)$, then
\[
\|\wt r_i\| \le e^\top \wt r_i + 2r\gamma_p \ve, \qquad 0\le \gamma_p = \mac O\left( \frac{ r\sqrt r}{\sigma_r(W^\#)}\frac {p^2}{q^2}\right), 
\]
where $\wt r_i$ are the rows of $R$. Moreover, for every index $i$, $\wt r_i + \gamma_p\ve e \in  \mac C_p^*$.
\end{lemma}

We now use the inequality between the arithmetic mean (AM) and the geometrical mean (GM) 
on a set of scalars $\{z_i\}_{i=1}^r$, that is,   
$\prod_i z_i \leq \Big(\frac{1}{r}\sum_i z_i\Big)^r$, 
for two different sets: 
\begin{align*}
    z_i = e^\top \wt r_i   + 2r\gamma_p \ve & & \to  & &\prod_i \|\wt r_i\|^2  \le \prod_i( e^\top \wt r_i   + 2r\gamma_p \ve)^2  \le\left( \sum_i \frac{e^\top \wt r_i + 2r\gamma_p\ve }r\right)^{2r}, \\
        z_i = \|\wt r_i\| & & \to  & &\prod_i \|\wt r_i\|^2  \le \left( \frac{\sum_i\|\wt r_i\| }{r}  \right)^{2r} \le \left( \sum_i \frac{e^\top \wt r_i + 2r\gamma_p\ve }r\right)^{2r}. 
\end{align*}
By the Hadamard theorem, $\det(R)^2 \le \prod_i \|\wt r_i\|^2$, and together with  \eqref{eq:lower_bound_volR}, this implies that $\det(R)^2$ is lower bounded by $1$ minus a perturbation. 
On the other hand, the sum of $\sum_i e^\top \wt r_i$ is just the sum of all elements in $R$, but since $e^\top r_j = 1$, $\sum_i e^\top \wt r_i = r$. As a consequence, the AMs above are upper bounded by $1$ plus a perturbation. In equation, this means 
\[
1 \, \; \wt{\le} \, \; \det(R)^2 \; \leq \; \text{GM}(z) \; \leq \; \text{AM}(z)  \, \; \wt{\le} \, \; 1. 
\] 
By Lemma~\ref{lem:AM-GM}, we conclude that the elements $z_i$'s, in both cases, are close to each other, and in particular close to their (arithmetic or geometric) mean. 
This is how we show that both $\|\wt r_i\|$ and $e^\top \wt r_i$ are close to $1$. Moreover, again due to  \eqref{eq:lower_bound_volR}, two distinct $\wt r_i$'s cannot be too close to each other, so we can also lower bound the distance $\|\wt r_i-\wt r_j\|$ as $1$ minus a perturbation. 

These properties of $\wt r_i$ are summarized in the following result and proven in Appendix~\ref{app:proof_of_R_close_to_orthogonal}.

\begin{lemma}
\label{lem:R_is_close_to_orthogonal}
Let $R$ be the matrix in Lemma~\ref{lem:definition_of_R}, and denote $\wt r_i$ the $i$-th row of $R$. 
If $\ve  = \mac O(\frac {\sigma_r(W^\#)}{r^{9/2}}\frac {q^2}{p^2})$, then
\[
\max\{  |  \|\wt r_i\| -1 |,   |e^\top \wt r_i -1 |  \} = \mac O\left( 
    \sqrt{\frac{ r^{7/2}}{\sigma_r(W^\#)}\frac {p^2}{q^2}\ve}
    \right).
\]
Moreover, 
  \[
    \min_{i\ne j} \|\wt r_i-\wt r_j\| \ge  1 - \mac O\lp   \sqrt{\frac{ r^{7/2}}{\sigma_r(W^\#)}\frac {p^2}{q^2}\ve}\rp.
    \]
\end{lemma}

In the next section, we provide the geometric intuition to bound the distance of the $\wt r_i$'s to the unit vectors, $e_j$'s. 
This will allow us to conclude the proof of Theorem~\ref{th:main}.


\subsection{Geometric intuition to bound the distance of the $\wt r_i$'s from the unit vectors}

Lemma~\ref{lem:upper_bound_detR} and Lemma~\ref{lem:R_is_close_to_orthogonal} give a quite complete description of the properties of the rows $\wt r_i$ of $R$. 
In particular, we proved that  there exists a parameter $\vf_p\ge 0$ such that \[
\max\{  |  \|\wt r_k\| -1 |,   |e^\top \wt r_k -1 |  \}\le \vf_p \sqrt \ve, 
\qquad  \min_{i\ne j} \|\wt r_i-\wt r_j\| \ge  1 - \vf_p\sqrt\ve,\qquad 
 \vf_p= \mac O\left( 
    \sqrt{\frac{ r^{7/2}}{\sigma_r(W^\#)}\frac {p^2}{q^2}}
    \right), 
    \]
and that all $\wt r_i$ approximately belong to the dual space $\mac C_p^*$, that is, 
\[
\wt r_k \in  \mac C_p^*-\gamma_p\ve e, \qquad \gamma_p = \mac O\left( \frac{ r\sqrt r}{\sigma_r(W^\#)}\frac {p^2}{q^2}\right).
\]
Let us fix an index $k$ and let $\beta:= e^\top \wt r_k$, 
the above conditions can be visualized on the space  $$\mac H_\beta := \{z\in \f R^r \ | \ e^\top z = \beta \}.$$ 
For $r=3$, Figure~\ref{fig:p-SSC_identifiability} illustrates these bounds, where
\begin{itemize}
    \item The row $\wt r_k$ is outside the ball $\mac B: = \{z\in \f R^r\ | \ \| z\| <1 - \vf_p\sqrt\ve\}\cap \mac H_\beta$, 
    \item The row $\wt r_k$ is inside the set $\mac P := (\mac C_p^*-\gamma_p\ve e)\cap \mac H_\beta$.
\end{itemize} 
In Section~\ref{sec:geom_int_dual_pSSC}, we showed that on the space $\mac E = \mac H_1 = \{x\ |\ e^\top x = 1 \}$, the set $\mac C_p^*$ is the convex hull of the ball $\mac Q_q:= \mac S_q \cap \mac E$ and the canonical basis vectors $e_i$.  
Here, $\mac C_p^*$ is translated by $\gamma_p\ve e$ and we are looking at the space $\mac H_\beta$ where $\beta$ is close to $1$, obtaining the set $\mac P$. As a consequence $\mac P$ has an analogous description: 
\[ 
\mac C^*_p \cap \mac E =  \conv\{\mac Q_q\cup \{ e_1,\dots, e_r\}\} \implies \mac P= \conv\{\wt{\mac S} \cup \{\wt e_1,\dots,\wt e_r\}\}.
\]
where $\wt {\mac S}$ is the ball
$$\wt {\mac S} := (\mac S_q -\gamma_p\ve e)\cap \mac H_\beta\cu (\mac C_p^*-\gamma_p\ve e)\cap \mac H_\beta, $$
and $\wt e_j$ are the points 
\[
\wt e_j := \frac{\beta}{1-r\gamma_p\ve} (e_j - \gamma_p\ve e)\in \mac H_\beta. 
\] 
This is formally proven in Lemma \ref{lem:Pcharacterization}. 
\begin{figure}
    \centering
    \includegraphics[width=0.9\linewidth]{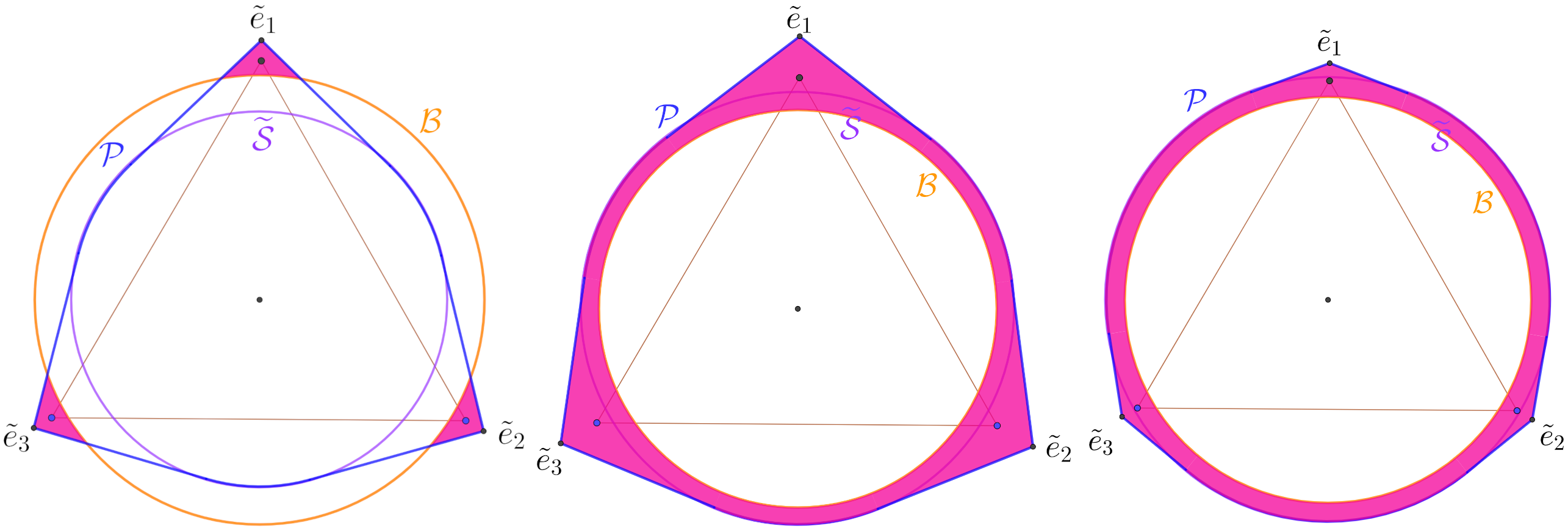}
    \caption{Visualization of the sets $\mac P$, $\mac B$ and $\wt{\mac S}$; the pink set is $\mac P\setminus\mac B$. 
    On the left, the favourable case where the rows of $R$ belong to disjoint regions around the vectors $\tilde e_j$'s which are close to the unit vectors $e_j$'s. 
    On the center, the effect of increasing the level of perturbation $\ve$. On the right, the effect of increasing the value of $p$ to almost $\sqrt{r-1}$. 
    Increasing $\ve$ or $p$ too much makes $R$ potentially far from a permutation matrix, since the rows of $R$ can be anywhere in the pink region.}   
    \label{fig:p-SSC_identifiability}
\end{figure} 
The row $\wt r_k$ thus lies in the set $\mac P\cap \mac B^c =  \mac P\setminus\mac B$. 
When there is no perturbation, that is, $\ve =0$, the set $ \mac P\setminus\mac B$ is exactly equal to $\{e_1,\dots,e_r\}$, meaning that $\wt r_k$ coincides with one of the canonical basis vectors. 
When increasing $\ve$, $\mac P$ gets larger and $\mac B$ gets smaller, thus allowing $\wt r_k$ to distance itself from the canonical basis vectors. 

A similar behaviour occurs when $p\to\sqrt{r-1}$, that is, when we get close to SSC. In fact, in this case, the dual $\mac C_p^*$ gets closer to the ball $\mac S_1$, and the projection $\wt{\mac S}$ will strictly contain $\mac B$ for every level of perturbation $\ve \ge 0$. \\

In cases when $\ve$ is too large and/or  $p\to\sqrt{r-1}$,  
we end up in the situations illustrated by the center and right images on  Figure~\ref{fig:p-SSC_identifiability}. The purple set is $ \mac P\setminus\mac B$, so the rows $\wt r_i$ can in theory be far from all $\wt e_j$. 

What can be proved is that, for small enough $\ve >0$ depending on $p$, we have the containment $\wt{\mac S}\cu \mac B$ that avoids the situation depicted in the center and right images of Figure~\ref{fig:p-SSC_identifiability}. That is,$\mac P \setminus B$ can be decomposed into disjoint regions around the vectors $\wt e_j$, each of the rows of $R$ belongs to one of these regions and no two of them can belong to the same region. 
In higher dimension, that is, for $r> 3$, we also need that all $\wt e_{i,j}:= (\wt e_i + \wt e_j)/2$  belong to $\mac B$. The following result reports the correct upper bound on $\ve$. 

\begin{lemma}
    \label{lem:S_eij_are_in_B2}
Given Assumption~\ref{ass:perturbed_pSSC}, and $\wt{\mac S}$, $\mac B$ and $\wt e_{i,j} = (\wt e_{i}+\wt e_{j})/2$ as defined in this section,   
\[
 \sqrt\ve = \mac O\lp 
  \frac{\min\{q,\sqrt 2\} -1}{\sqrt{\frac{ r^{7/2}}{\sigma_r(W^\#)}\frac {p^2}{q^2}}} 
  \rp 
\implies \conv(\{\wt e_{i,j}\}_{i\ne j}, \mac {\wt S}) \cu \mac B.
\]
\end{lemma}

This is enough to show that for small enough $\ve> 0$, the set  $\mac P\setminus\mac B$ is the disjoint union of small regions around the $\wt e_j$'s, coloured in pink in the left image on Figure~\ref{fig:p-SSC_identifiability}. 

As mentioned above, for too large $\ve$ and/or $p$, 
$\wt{\mac S} \not\cu \mac B$, and $\wt r_k$ can be far from any $e_j$. 
This visualization holds in higher dimensions only for $r-1 > p^2 \ge  r-2$, otherwise we also need that every $\wt e_{i,j}$ is inside $\mac B$.
The upper bound on $\ve$ provided by Lemma~\ref{lem:S_eij_are_in_B2} thus guarantees that  $\mac P\setminus\mac B$  can be written as the disjoint union of small regions $\mac P_i$ around $\wt e_i$. 

Since the row $\wt r_k$ must fall into one of the above mentioned disjoint regions, say $\mac P_j$,  the diameter of $\mac P_j$ is an upper bound over the distance $\|\wt r_k-\wt e_j\|$ and it can be computed with classical Euclidean geometry.  
\begin{lemma}
\label{lem:distance_rk_ei_sscp2}  Given  the above notation, 
\[
   \ve  = \mac O\lp 
\lp \min\{q,\sqrt 2\} -1\rp^2
\frac {\sigma_r(W^\#)}{r^{9/2}}\frac {q^2}{p^2}\rp
\implies 
\min_j\|\wt r_k-\wt e_j\|^2 =  \frac { \ve}{\min\{q^2-1,1\}}
    \mac O\left(  
   \frac{ r^3\sqrt r}{\sigma_r(W^\#)}\frac {p^2}{q^2}
      \right)
  .
    \]
\end{lemma}

The proofs for the last two results are reported 
in Appendix~\ref{app:geometric_intuition_th_pSSC}. 
Now, we are ready to take the last steps in the proof of Theorem~\ref{th:main}.

\subsection{Sketch of the Proof of Theorem~\ref{th:main}} \label{sec:sketchmain}

Lemma~\ref{lem:distance_rk_ei_sscp2} give us a bound on $\|\wt r_k-\wt e_j\|$. Since 
\[
\wt e_k-e_k = 
\lp e^\top \wt r_k  -1\rp e_k + \gamma_p\ve( r e_k -  e), 
\]
the quantity $\|\wt e_k-e_k\|$ will depend mainly on $\beta-1$ and $r\gamma_p\ve$, both asymptotically less than the estimated bound on  $\|\wt r_k-\wt e_j\|$ reported in Lemma~\ref{lem:distance_rk_ei_sscp2}. 
As a consequence, the same bound will also applies to $\|\wt r_k-e_j\|$. 
Finally, the lower bound on $\|\wt r_i-\wt r_j\|$ of Lemma~\ref{lem:R_is_close_to_orthogonal} ensures that no two distinct rows are close to the same $e_j$, and therefore $R$ is close to a permutation matrix $\Pi$. 

Now that we have an estimation on $\|R-\Pi\|_{1,2}$, we can plug it in Corollary~\ref{cor:Error_estimation_given_R_is_permutation} and obtain the bound 
\begin{align*}   
\min_\Pi \|W^\#-W^*\Pi\|_{1,2} \le    \|W^*\| \cdot  \mac O\left(  \sqrt {
     \frac { \ve}{\min\{q^2-1,1\}}
   \frac{ r^{7/2}}{\sigma_r(W^\#)}\frac {p^2}{q^2}
   }
      \right).
\end{align*}
Eventually, we can actually substitute $\|W^*\|$ with $\|W^\#\|$ since $W^\# = W^*R + M$ and $R$ is close to a permutation matrix $\Pi$, so
\[
\|W^*\| \le \|R^{-1}\| \|W^\#-M\|\approx\frac{\|W^\#\|}{\sigma_r(R)}
\approx \frac{\|W^\#\|}{\sigma_r(\Pi)} = \|W^\#\|.
\]
The full proof with all the details can be found in Appendix~\ref{app:proof_of_th_main}.

\section{Conclusion} 

In this paper, we studied the identifiability of the factors $W^\#$ and $H^\#$ in the decomposition $X = W^\# (H^\#)^\top + N$, where 
$W^\#$ is full rank, $H^\#$ is row stochastic and satisfies the $p$-SSC condition (Definition~\ref{def:SSCexpanded}), and $\|N\|_{1,2} \leq \epsilon$. 

We proved that the factors $W^\#$ and $H^\#$ can be recovered from $X$ by solving min-vol NMF~\eqref{eq:apprminvol}. 
We provided two main theorems: one general (Theorem~\ref{th:main}), and one specific  to the near-separable case (Theorem~\ref{th:main2}) which requires a column of $X$ close to each column of $W^\#$ (equivalently, $p$ is close to one). 
This fills an important gap in the literature: although min-vol NMF has been used successfully in many applications, 
a theoretical guarantee in the presence of noise was lacking. Moreover, our results also offer geometric insights on the robustness one can expect. In particular, the noise level allowed depends on how much the data points are well spread; in other terms, the smaller $p$, the more likely min-vol NMF will recover the ground truth factors $W^\#$ and $H^\#$, while for $p \rightarrow \sqrt{r-1}$, which corresponds to the SSC condition, robustness is not possible. 

\paragraph{Further research}  

An interesting question for further work is to follow the line of thought of the papers \cite{najafi2021statistical, saberi2025fundamental}, where authors assume the data follows a statistical model, namely $x_i = W h_i + n_i$ 
where $h_i$ is uniform in the simplex (equivalently, follow the uniform Dirichlet distribution), and $n_i$ is Gaussian. The question is: how many samples do we need to be able to estimate $W$ up to some accuracy with high probability depending on the noise level?  They propose a non-polynomial time algorithm to do this that is close to the sample optimal bound (they derive a lower bound). To adapt this idea to our setting, we would need to assume that the noise is bounded (or at least bounded with high probability) since we require $\|n_i\| \leq \varepsilon$ for all $i$, while we would need to quantify how many samples are needed for the $p$-SSC condition to be satisfied with high probability. 

Another question for further research is to study the tightness of the bounds of Theorems~\ref{th:main} and~\ref{th:main2}: are these bounds tight or can they be improved? Note that, for the most favorable case, that is, the separable case (Theorem~\ref{th:main2} with $p=1$), one cannot do better than $\ve \leq  \sigma_r(W)$, otherwise the noise can make $W$ rank deficient (this is related to Lemma~\ref{lem:full_rank_W}), while the error on $W^\#$ is at least $\min_{\Pi \in \mac P_r} \|W^\#-W^*\Pi\|_{1,2} \geq C_2 \frac{\|W^\#\|}{\sigma_r(W^\#)}$ for some constants $C_2$; see the discussion in~\cite[Chapter~4]{Gil20}.

Last but not least, our identifiability results rely on solving the min-vol NMF optimization problem~\eqref{eq:apprminvol}. 
Finding the minimum-volume simplex containing a given set of data points is NP-hard in general. However, the problem might be easier under the $p$-SSC. In particular, it can be solved in polynomial time for $p=1$, that is, separability~\cite{AGKM11}; and there are polynomial-time algorithms for small dimensions ($r\leq 4$)~\cite{zhou2002algorithms}. 
Hence studying the complexity of min-vol NMF under $p$-SSC would be particularly interesting, possibly leading to the design of polynomial-time algorithms for NMF beyond the separability assumption.


\small 

\bibliographystyle{spmpsci} 
\bibliography{references}

\normalsize 

\appendix

\section{Proof of Preliminary Results}

\subsection{Properties of $p$-SSC}

\subsubsection{Proof of Lemma~\ref{lem:p-hypersphere}}
\label{app:proof of Lemma on Q_p}
Any $x\in \mac E$ can be written as $e/r + w$ for some vector $w$ such that $e^\top w=0$, so $\|x\|^2 = \|e/r\|^2 + \|w\|^2 = 1/r + \|w\|^2$. This is enough to show that
 \[
\mac S_p\cap \mac E
 = \left\{   x\in \mac E \ \Big| \ x = \frac er + w, \ \frac 1{p^2} - \frac 1r\ge \|w\|^2   \right\} = \mac Q_p, 
\]
and, as a consequence,
\[
\mac C_p \cap \mac E= \mac S_p\cap \f R^r_+ \cap \mac E = \mac Q_p \cap \Delta^r.
\]
If $H \in \R_+^{n \times r}$ is row stochastic, then $\cone(H^\top)$ is the disjoint union of all nonnegative multiples of $\cone(H^\top)\cap \Delta^r = \conv(H^\top)$. Analogously, $\mac C_p$  is the disjoint union of all nonnegative multiples of $\mac C_p\cap \mac E$. 
In particular, $\mac C_p\cu \cone(H^\top)\iff \mac C_p\cap \mac E\cu \conv(H^\top)$, and this proves that $H$ is $p$-SSC if and only if $\mac Q_p\cap \Delta^r\cu \conv(H^\top)$.

In the case $H \in \R_+^{n \times r}$ is not row stochastic, we have the similar result 
\[
\mac Q_p\cap \Delta^r\cu \cone(H^\top) 
\iff \cone(\mac Q_p\cap \Delta^r)\cu \cone(H^\top)
\iff \cone(\mac C_p\cap \Delta^r)\cu \cone(H^\top)
\iff \mac C_p\cu \cone(H^\top).
\]

Notice now that
\[
\partial\mac Q_p = 
\left\{   x\in \mac E \ \Big| \ x = \frac er + w, \ \frac 1{p^2} - \frac 1r =  \|w\|^2   \right\},\qquad 
x\in \Delta^r \implies \left\| x-\frac er   \right\|^2 = \|x\|^2  - \frac 1r \le 1- \frac 1r,
\]
since $x\ge 0, \ e^\top x = 1\implies 0 \le x\le e\implies  \|x\|^2 \le e^\top x  = 1$. As a consequence,  $\Delta^r\cu \mac Q_p$ for $ p= 1$, and the intersection between $\partial\mac Q_p$ and $\Delta^r$ is 
\[
\partial \mac Q_1\cap \Delta^r =\{x\in \Delta^r \ | \ \|x\| = 1 \} =    \{e_1,e_2,\dots,e_r\}.
\]
Notice that $x\in \mac Q_p\implies x\in \mac E\implies e^\top x = 1$ and if $x^+:=\max\{0,x\}$, then $e^\top x^+ \ge 1$, and $e^\top x^+=1$ if and only if $x=x^+\ge 0$.  Call $s_x:=\#\{i:x_i > 0\}$.  
When $p^2 \ge r-1$ we find that for every $x\in \mac Q_p$, 
\begin{equation}
    \label{eq:Qp_inside_Delta}
    \frac 1{r-1}\ge  \frac 1{p^2}\ge \|x\|^2 \ge  \|x^+\|^2 \ge \frac {(e^\top x^+)^2}{s_x} \ge \frac 1{s_x}  \implies s_x \ge r-1.
\end{equation}
The only case in which $x\not > 0$ is for $s_x = r-1$, but in that case all the inequalities in \eqref{eq:Qp_inside_Delta} are actual equalities and in particular $e^\top x^+ = 1$, so $x\ge 0$ anyway. This shows that  $p^2 \ge r-1\implies \mac Q_p\cu \Delta^r$.

If $x\in \mac Q_{\sqrt {r-1}}  \cap \partial\Delta^r$ then $s_x \le r-1$ and thus again $s_x = r-1$, meaning that there is exactly one zero entry in $x$. Again,  all the inequalities in \eqref{eq:Qp_inside_Delta} are equalities, and the QM-AM inequality $ \|x^+\|^2 \ge \frac {(e^\top x^+)^2}{s_x}$ achieves equality only for all nonzero elements of $x$ being equal. Since $e^\top x^+ = 1$, all nonzero elements of $x$ must be equal to $1/(r-1)$. We conclude that 
 \[ \mac Q_{\sqrt{r-1}} \cap \partial\Delta^r = \left\{\frac{e - e_i}{r-1} \ \Big |\ i=1,\dots, r \right\}.\] In the case $p^2=r$, we have
 \[
   \mac Q_{\sqrt r}= \left\{   x\in \mac E \ \Big| \ x = \frac er + w, \ 0\ge \|w\|^2   \right\} =\left\{    \frac er    \right\} 
 \]
and in the last case $p^2 > r$, we get the impossible condition $0>\|w\|^2$, meaning that $\mac Q_p$ is empty.

\subsubsection{Proof of Lemma~\ref{lem:Cp_Sp_dual_cones_and_inclusions}}
\label{app:proof_of_lem_CpSp_dual_cones}

Recall that $\mac S_p = \cone(\mac Q_p)$,  $\mac Q_p \cu \mac S_p$, and, by definition, 
\[
\mac S_p = \left\{x \in \R^r \ \big| \ e^{\top}x \geq p \|x\| \right\}, \qquad 
 \mac Q_p = 
\left\{   w + \frac er \ \Big |  \ \|w\|^2\le \frac{1}{p^2}-\frac 1r , \ e^\top w = 0      \right\}.
\]
Notice that the vector $e$ and its nonnegative multiples are in $\mac S_p^*$ since $e^\top x \ge p\|x\| \ge 0$ for every $x\in \mac S_p$. 
Let now $y\in \mac S_p^*$ be not a multiple of the vector $e$, and take $x = \frac er  - \lambda \lp y  - e^\top y \frac er\rp\in \mac Q_p\cu \mac S_p$ where $\lambda = \sqrt{\frac{1}{p^2}-\frac 1r} / \|y  - e^\top y \frac er \| $. 
By the definition of duality, 
\begin{align*}
    0\le x^\top y & = \frac{e^\top y}{r} - \lambda  \lp y  - e^\top y \frac er\rp^\top y
     = \frac{e^\top y}{r} - \lambda  \left\| y  - e^\top y \frac er\right\|^2
      - \lambda  \lp y  - e^\top y \frac er\rp^\top e^\top y \frac er\\
      & =  \frac{e^\top y}{r} - \sqrt{\frac{1}{p^2}-\frac 1r}  \left\| y  - e^\top y \frac er\right\|
       =  \frac{e^\top y}{r} - \sqrt{\frac{1}{p^2}-\frac 1r}  
       \sqrt{\|y\|^2 - \frac{(e^\top y)^2}{r}}.
\end{align*}
In particular, $e^\top y\ge 0$ and 
\begin{align*}
e^\top y \ge  r\sqrt{\frac{1}{p^2}-\frac 1r}  
       \sqrt{\|y\|^2 - \frac{(e^\top y)^2}{r}}
     &  \implies 
      (e^\top y)^2 \ge  \lp \frac{r}{p^2}- 1 \rp
      \lp r\|y\|^2 - (e^\top y)^2\rp  \\
       &  \implies 
      \frac 1{p^2}(e^\top y)^2 \ge  \lp \frac{r}{p^2}- 1 \rp
       \|y\|^2   \\
       &  \implies 
       (e^\top y)^2 \ge  \lp r - p^2 \rp
       \|y\|^2  \implies 
       e^\top y \ge  q
       \|y\|
\end{align*}
This means that $y\in \mac S_q$, so $\mac S_p^*\cu \mac S_q$. In order to show that $\mac S_q \cu \mac S_p^*$, it is enough to prove that $x^\top y \ge 0$ for every $x\in \mac Q_p$ and $y\in \mac Q_q$. 
Recall that,  by Lemma~\ref{lem:p-hypersphere}, 
\[
 \mac Q_p := \left\{   x\in \mac E \ \Big | \ x = \frac er + w, \ \| w\|^2\le \frac{1}{p^2}-\frac 1r       \right\}.
 \]
As a consequence, for every $x\in \mac Q_p$ and $y\in \mac Q_q$,
\begin{align*}
    x^\top y &= \lp  x-\frac er  \rp^\top \lp y-\frac er\rp + \frac 1r \ge - \left\| x-\frac er\right\| \ \left\| y-\frac er\right\|+ \frac 1r\\
    &\ge  - \sqrt{ \lp\frac 1{p^2} - \frac 1r\rp} \sqrt{ \lp\frac 1{q^2} - \frac 1r\rp}
    + \frac 1r =  - \sqrt{ \lp\frac {q^2}{rp^2} \rp} \sqrt{ \lp\frac {p^2}{rq^2} \rp}
    + \frac 1r = 0.
    \end{align*}
The dual of $\mac C_p$ is thus  $ \mac C_p^* = (\mac S_p \cap \f R^r_+)^*=\mac S_p^* + \f R^r_+ =\mac S_q + \f R^r_+ = \cone(\mac Q_q \cup \{e_1,\dots, e_r\})$. Since the set $\mac Q_q \cup \{e_1,\dots, e_r\}$ is already on $\mac E$, $\mac C_p^* \cap \mac E = \conv(\mac Q_q \cup \{e_1,\dots, e_r\})$.

\subsubsection{Equivalence between different SSC2}

\label{app:proof_equivalence_SSC2}
\begin{lemma}
    Given a matrix $H\in \f R_+^{n\times r}$ with $r\ge 2$ satisfying SSC1, that is, $\mac S_{\sqrt{r-1}}\cu \cone(H^\top)$, the following conditions are equivalent:

         Condition 1. $ \partial \cone (H^\top) \cap \mac S_{\sqrt{r-1}} 
 =  \left\{\lambda(e-e_k)  \ | \  \lambda \geq 0  \text{ and } k \in [r]\right\},$
 
 Condition 2.  $\cone^*(H^\top) \cap \partial \mathcal{S}_1 =  \{\lambda e_k  \ | \  \lambda \geq 0  \text{ and } k \in [r]\}$. 
\end{lemma}
\begin{proof}
Given any convex closed cone $\mac F\ne \{0\}$ such that the dual cone $\mac F^*$ is not $\{0\}$, the following properties hold:
\begin{itemize}
    \item $\mac F$ is self-dual, that is, $\mac F^{**} = \mac F$.
    \item For every non-zero $x\in \partial \mac F^*$ there exists a non-zero $y\in \partial \mac F$ such that $x^\top y = 0$.  
\end{itemize}

\noindent Notice that the vector $e$ belongs to all the cones  $\mac S_p$ and their dual $\mac S_p^*$  for $0\le p\le \sqrt r$ since
\[
r = e^\top e \ge p\|e\| = p \sqrt r \implies e\in \mac C_p, \qquad 
e^\top x \ge p \|x\| \ge 0 \ \ \ \forall  x\in\mac C_p \implies e\in \mac C_p^*.
\]
Moreover,  $\mac S_{\sqrt{r-1}}\cu \cone(H^\top)\cu \f R^r_+$, so by duality and Lemma~\ref{lem:Cp_Sp_dual_cones_and_inclusions}, $\R^r_+\cu \cone^*(H^\top)\cu\mac S_1$. As a consequence, the cones $\mac S_1$, $\mac S_{\sqrt {r-1}}$, $\cone(H^\top)$ and $\cone^*(H^\top)$ are all convex closed cones not equal to $\{0\}$ whose duals are again not equal to $\{ 0 \}$. 

Suppose now that Condition 1 holds, and consider a nonzero vector $x\in \cone^*(H^\top) \cap \partial \mathcal{S}_1$. Since $\mac S_1 =\mac S_{\sqrt{r-1}}^*$,  there exists a nonzero $y\in \partial \mac S_{\sqrt{r-1}}$ such that $x^\top y = 0$. If $\mac H_x = \{z \ | \ x^\top z\ge 0\}$, then notice that $y \in \partial  \mac S_{\sqrt{r-1}}
\cu \mac S_{\sqrt{r-1}}\cu \cone(H^\top)\cu \mac H_x$  and $y\in\partial\mac H_x$. As a consequence, $y\in\partial\cone(H^\top)$ and $y\in \partial \mac S_{\sqrt{r-1}}$, so by Condition 1, the vector $y$ must be equal to $\lambda (e-e_k)$ for some $\lambda > 0$ and some $k\in [r]$. 

From $x\in\partial \mac S_1$ we find that $e^\top x = \| x\| > 0$, so 
\begin{align*}
    0 = x^\top y \implies 0 = x^\top(e-e_k) = \|x\| - x^\top e_k\implies \|x\| = x^\top e_k \implies x = \mu e_k, 
\end{align*}
for some $\mu > 0$. This is enough to show that $\cone^*(H^\top) \cap \partial \mathcal{S}_1 =  \{\lambda e_k  \ | \  \lambda \geq 0  \text{ and } k \in [r]\}$, that is, that Condition 2 holds. 

Suppose now that Condition 2 holds, and consider a nonzero vector $x\in  \partial \cone (H^\top) \cap \mac S_{\sqrt{r-1}}$. Since $\cone (H^\top) =\cone^{**}(H^\top)$, there exists a nonzero $y\in \partial \cone^*(H^\top)$ such that $x^\top y = 0$. 
Since $\mac S_1 =\mac S_{\sqrt{r-1}}^*$, the dual of the SSC1 is $\cone^*(H^\top)\cu \mac S_1$. 
As a consequence, if $\mac H_x = \{z \ | \ x^\top z\ge 0\}$, then  $y \in \partial \cone^*(H^\top)
\cu 
 \cone^*(H^\top) \cu 
\mac S_{1}\cu \mac H_x$  and $y\in\partial\mac H_x$. This is enough to show that  $y\in\partial\cone^*(H^\top)$ and $y\in \partial \mac S_{1}$, so by Condition 2, the vector $y$ must be equal to $\lambda e_k$ for some $\lambda > 0$ and some $k\in [r]$. 

From SSC1, $x\in \mac S_{\sqrt{r-1}}\cu \cone(H^\top)$ but since  $x\in  \partial \cone (H^\top)$, we also have that $x\in  \partial \mac S_{\sqrt{r-1}}$, that is, we find that $e^\top x =\sqrt{r-1} \| x\| > 0$, so we can write
\begin{align*}
    0 = x^\top y \implies 0 = x^\top e_k = x^\top (e_k-e) + \sqrt{r-1}\|x\| \implies 
    \|x\| \sqrt{r-1}= x^\top(e-e_k)  \le \|x\| \|e-e_k\| = \|x\| \sqrt{r-1}.
\end{align*}
As a consequence, $x$ must be a positive multiple of $e-e_k$, and this is enough to show that $\partial \cone (H^\top) \cap \mac S_{\sqrt{r-1}} 
 =  \left\{\lambda(e-e_k)  \ | \  \lambda \geq 0  \text{ and } k \in [r]\right\}$, that is, that Condition 1 holds. 
\end{proof}

\subsubsection{The matrix $H_p$}\label{app:Hp}

 \paragraph{Proof of Lemma~\ref{lem:def_of_Hp}} 
Each column $v_i$ of $H_p$ is by definition in the segment connecting $e/r$ to $e_i$, and we claim that $r\alpha_p$ is the coefficient realizing the correct convex combination of $v_i$, meaning $v_i = r\alpha_p e/r + (1-r\alpha_p)e_i$. In order to prove this, it is sufficient to show that $r\alpha_p\in [0,1]$ and that $v_i\in \partial \mac S_p$, that is, $1 = e^\top v_i = p\|v_i\|$. 
Notice that $q/p$ is decreasing in $p$, and $1/\sqrt{r-1}\le q/p\le \sqrt{r-1}$, so 
\[
0\le r\alpha_p = 1-  \frac 1{\sqrt{r-1}}\frac qp \le 1 - \frac{1}{r-1}\le 1.
\]
Moreover,
\begin{align*}
    \left\| r\alpha_p \frac er + (1-r\alpha_p)e_i\right\|^2 &=
   \left\|\frac er + (1-r\alpha_p)\lp e_i - \frac er  \rp\right\|^2
     =\frac 1r + (1-r\alpha_p)^2 \lp  1 + \frac 1r -\frac 2r  \rp   = 
    \frac 1r + \frac     1{r}  \frac{q^2}{p^2}       
    = \frac{1}{p^2}.
\end{align*}
As a consequence, $H_p$ can be written as  $H_p= \alpha_p E + (1- r\alpha_p) I$. Since $v_i\in \mac S_p\cap \Delta^r = \mac Q_p\cap \Delta^r$,  by Lemma~\ref{lem:p-hypersphere} we have that if $H\in \f R^{n\times r}$ is $p$-SSC, then $v_i\in \cone(H^\top)$ and if $H$ is row stochastic then $v_i\in \conv(H^\top)$. This holds for every index $i$, so the same is true for their convex hull $\conv(H_p)$.

\paragraph{Properties and Estimations on $H_p$}

 Here we collect some of the properties of $H_p$ relative to its last singular value and its inverse. Here the norm $\|\cdot\|_1$ is the induced $1$-norm, that is, the maximum $1$-norm among the columns of the matrix, $\|A\|_1 = \max_i \|a_i\|_1$. 

\begin{lemma}
\label{lem:last_sv_of_Hp}    For every $1\le  p \le \sqrt{r-1}$, let $q =\sqrt{ r-p^2}$. The matrix $H_p$ satisfies 
\[ \sigma_r(H_p) =  \frac 1{\sqrt{r-1}}\frac qp \quad \text{ and } \quad 
\|H_p^{-1}\|_1 =       \frac 1r \left[ 2(r-1)^{3/2}\frac pq - (r-2) \right]. \]
Moreover, 
\[
 2r -3 \ge 
      2\sqrt{r-1} \frac pq - 1
      \ge \|H_p^{-1}\|_1 \ge \sqrt{r-1} \frac pq \ge 1.
\]

\end{lemma}
\begin{proof}
 The matrix $H_p$ is an Hermitian matrix whose eigenvalues are $r\alpha_p + (1-r\alpha_p)= 1$ with multiplicity $1$ and $1-r\alpha_p\ge 0$ with multiplicity $r-1$. 
    Since $1\le  p,q \le \sqrt{r-1}$, all eigenvalues are strictly positive,  $H_p$ is positive definite, and its smallest singular value coincides with its smallest eigenvalue $\sigma_r(H_p) = 1-r\alpha_p$.  
     
The inverse of $H_p$ is    \[
    H_p^{-1} 
    = -  \alpha_p  (1- r\alpha_p)^{-1}  E + (1- r\alpha_p)^{-1} I, 
    \]
since 
\[
[\alpha_p E + (1- r\alpha_p) I] [ -  \alpha_p   E +  I] = [-(1- r\alpha_p)\alpha_p+\alpha_p-r\alpha_p^2]E+ (1- r\alpha_p) I =  (1- r\alpha_p) I.
\]
The $1$-norm of the columns of $H_p^{-1}$ are all the same and thus equal to  $\|H_p^{-1}\|_1$, that is, 
\begin{align*}
    \|H_p^{-1}\|_1 & =  (r-1) \alpha_p  (1- r\alpha_p)^{-1}   + (1- r\alpha_p)^{-1} (1-\alpha_p)
    = \frac{1+ (r-2)\alpha_p  }{1- r\alpha_p}
    =\frac{r-2}{r} \frac{\frac {2r-2}{r-2}-1+ r\alpha_p  }{1- r\alpha_p}\\
    &=\frac{r-2}{r}\lp  2\frac {r-1}{r-2}\frac1{\frac 1{\sqrt{r-1}}\frac qp}  - 1\rp
    =   \frac 1r \left[ 2(r-1)^{3/2}\frac pq - (r-2) \right],
\end{align*}
where 
\begin{align*}
   \frac 1r \left[ 2(r-1)^{3/2}\frac pq - (r-2) \right] &=
   \sqrt{r-1}\frac pq \frac 1r \left[ 2(r-1) - \frac {r-2}{\sqrt{ r-1}}\frac qp \right]
   \\& \ge \sqrt{r-1}\frac pq \frac 1r \left[ 2(r-1) - (r-2) \right] = \sqrt{r-1}\frac pq \ge 1, 
\end{align*}
and 
\begin{align*}
   \frac 1r \left[ 2(r-1)^{3/2}\frac pq - (r-2) \right] &=
   \frac {r-1}r \left[ 2\sqrt{r-1}\frac pq - 1 + \frac{1}{r-1} \right]
   =
    2\sqrt{r-1}\frac pq - 1   - 
   2\frac {\sqrt{r-1}}r \frac pq   + \frac{2}{r(r-1)}
   \\& \le 2\sqrt{r-1}\frac pq - 1
   \le 2r-3.
\end{align*}
\end{proof}

\subsection{Common Steps for the Main Results}

\subsubsection{Notation and Prerequisites}

Here is a review of the norms and notation we use:
\begin{itemize}

    \item $\|A\|$, $\|v\|$ is the classical Euclidean norm on matrices (also called spectral norm) and on vectors. 
    
    \item  $\|A\|_1$ is the induced $1$-norm, that coincides with the maximum $1$-norm of the columns of $A$ 
    \[
    \|A\|_1 = \max_i \|a_i\|_1 = \max_i\sum_j |a_{j,i}|. 
    \]
    
    \item $\|A\|_{1,2}$ is the induced $(1,2)$-norm, that coincides with the maximum Euclidean norm of the columns of $A$, 
    \[
    \|A\|_{1,2} = \max_i \|a_i\|. 
    \]
    
    \item $\|A\|_F$ is the Frobenius norm, that is, $\|A\|_F^2 = \text{trace}(A^\top A) = \sum_{i,j} |a_{j,i}|^2.$

    \item When we say ``if $\ve = \mac O(f({\bf x}))$ ...'' it means ``there exists an absolute constant $C>0$ such that for every $\ve \le Cf({\bf x})$\dots'' 
    
    \item Likewise, when we say ``then $g({\bf x}) = \mac O(f({\bf x}))$'', it means ``then there exists an absolute constant $C>0$ such that for every value of the variables $\bf x$ in their respective domains, $g({\bf x})\le C f({\bf x})$ holds''.    
    
\end{itemize}

The following are known results on the relations between the different norms and the singular values. 
\begin{lemma}
    \label{lem:norm_inequalities} Given $A,B,C$ matrices with opportune dimensions, then
    \[
    \|ABC\|_{1,2} \le \|A\|\|B\|_{1,2}\|C\|_1, 
    \]
    and, if $A\in \f R^{m\times n}$,
    then     
    \[ \|A\| \le \|A\|_F\le \sqrt {\min\{m,n\}} \|A\|, \qquad 
    \|A\|_{1,2}\le \|A\|\le \sqrt n \|A\|_{1,2}.
    \]
\end{lemma}
\begin{proof}
  In  \cite[Section 6.2]{higham2002NumAlg} one can find all the above inequalities except for the equivalence constants between $\|A\|$ and $\|A\|_{1,2}$. 
  Notice that
  \[
  \|A\|_{1,2} = \max_i \|a_i\| = \max_i \|Ae_i\|\le \|A\|,
  \]
  \[
  \|A\| = \max_{\|v\|= 1} \|Av\| \le  \max_{\|v\|= 1} \sum_i \|a_i\||v_i|\le \|A\|_{1,2}
  \max_{\|v\|= 1} \|v\|_1 = \sqrt n \|A\|_{1,2}.
  \]
\end{proof}

\begin{lemma}
\cite[Section III.6]{bhatia1997Matrix} \label{lem:lower_bound_last_sv}   Given $A\in  \R^{n\times r}$, $B\in  \R^{m\times r}$ matrices with $r\le n$ and $r\le m$, then
    \[
    \sigma_r(A) \|B\| \ge \sigma_r(AB^\top) \ge \sigma_r(A)\sigma_r(B).
    \]
\end{lemma}

Another result we need is an estimation on how much an additive perturbation  $M$ can alter the volume of a matrix $A$. 

\begin{lemma}
  \label{lem:perturbation_on_volume}  Given $A,M\in \R^{m\times r}$, 
     \begin{align*}
        \det((A+M)^\top (A+M) ) &\le \det(A^\top A) + \prod_i^r  (\sigma_i(A)+\|M\|)^2 -  \prod_i^r  \sigma_i(A)^2\\&\le \det(A^\top A) + (\|A\| +  \|M\|)^{2r} -  \|A\|^{2r}. 
    \end{align*}
\end{lemma}
\begin{proof}
We have 
    \begin{align*}
          \det((A+M)^\top (A+M) ) &= \prod_i^r  \sigma_i(A+M)^2\le  \prod_i^r  (\sigma_i(A)+\|M\|)^2\\
          & = \prod_i^r  \sigma_i(A)^2  + \prod_i^r  (\sigma_i(A)+\|M\|)^2 -  \prod_i^r  \sigma_i(A)^2 . 
    \end{align*}
    Here $\prod_i^r  (\sigma_i(A)+\|M\|)^2 -  \prod_i^r  \sigma_i(A)^2$ is increasing in each of the $\sigma_j(A)$ since its derivative is 
    \[
    2\frac{\prod_i^r  (\sigma_i(A)+\|M\|)^2}{\sigma_j(A)+\|M\|} -  2\frac{\prod_i^r  \sigma_i(A)^2}{\sigma_j(A)} \ge 0, 
    \]
    so we can majorize each $\sigma_j(A)$ with $\|A\|$ to complete the proof. 
\end{proof}

When the Arithmetic Mean (AM) and the Geometric Mean (GM) of some nonnegative elements $x_i$ are equal to each other, then all the elements coincide. If the difference between AM and GM is small, then it is also reasonable to expect that the elements are close to each other, as shown by the following result.

\begin{lemma}
\label{lem:AM-GM}    Given $x_1\ge x_2\ge \dots \ge x_n\ge 0$, let $A$ and $G$ be their arithmetic and geometric means respectively, that is, 
    \[
    A = \frac 1n \sum_ix_i,\qquad G = \sqrt[n]{\prod_i x_i}.
    \]
    The following relations hold:
    \begin{align*}
          (\sqrt {x_1} - \sqrt{x_n})^2 &\le n (A-G),\\
          x_1-x_n& \le  (\sqrt {x_1} + \sqrt{x_n})  \sqrt {n(A-G)}.
    \end{align*}
\end{lemma}
\begin{proof}
    Fix $x_1$ and $x_n$. Let us try to minimize $A-G$. The derivative with respect to $x_i$ is
\begin{align*}
    \frac{\partial }{\partial {x_i}} (A-G)  &= 
    \frac 1n - \frac Gn\frac{x_i^{1/n - 1}}{\sqrt[n]{x_i}} = \frac 1n(1 - G/x_i)
\end{align*}
so $A-G$ has a minimum for $G=x_2 = \dots = x_{n-1}$. In particular, 
\[
G^n = x_1x_n G^{n-2}, \ A = \frac{x_1+x_n + (n-2)G}{n} \implies 
G = \sqrt{x_1x_n}, \   A = \frac 1n(\sqrt {x_1} - \sqrt{x_n})^2 + G. 
\] 
\end{proof}

A last essential result for the estimation is to bound $|(1\pm x)^n -1|$ when $x$ is very small. 

\begin{lemma}
    \label{lem:bound_perturbation_power} Given $0\le x\le \frac 1{nc}\le  1$ for some positive integer $n$ and some positive $c$, then
    \[
    (1+x)^n - 1 \le ncx(e^{1/c}-1), \qquad 1-(1-x)^n  \ge ncx(1- e^{-1/c}).\]
    If $0\le x\le \frac 1{nc}\le  \frac{nc-1}{c}$ then 
    \[
    \qquad 1-(1-x)^{1/n}  \le \frac {xc}{nc- 1}.
    \]
    If instead $0\le x \le 1$ then
    \[
      1 - (1-x)^n \le nx.
    \]
\end{lemma}
\begin{proof} Let us prove the four inequalities above. For the first inequality, we have 
    \begin{align*}
          (1+x)^n - 1 =  x\sum_{k=0}^{n-1} (1+x)^k\le    x\sum_{k=0}^{n-1} \left(1+\frac 1{nc}\right)^k = x\frac{ \left(1+\frac 1{nc}\right)^n-1}{\frac 1{nc}} \le ncx(e^{1/c}-1),
    \end{align*}
  where we used $(1+1/y)^y\le e$ for every $y> 0$. 
  The second inequality is analogous since 
   \begin{align*}
          1 -(1-x)^n  =  x\sum_{k=0}^{n-1} (1-x)^k\ge    x\sum_{k=0}^{n-1} \left(1-\frac 1{nc}\right)^k = x\frac{ 1- \left(1-\frac 1{nc}\right)^n}{\frac 1{nc}} \le ncx(1-e^{-1/c}),
    \end{align*}
     where we used $(1-1/y)^y\le e^{-1}$ for every $y> 0$. 

Notice now that for every $0\le y\le 1$, the function $(1+ny)(1-y)^n$ admits a global maximum for $y=0$ since
\[
\frac{\partial}{\partial y}(1+ny)(1-y)^n = n(1-y)^{n-1}\left( 1-y - (1-ny) \right) = 0 \iff y = 0,1
\]
unless $n=1$. In any case, $(1+ny)(1-y)^n\le 1$. The third inequality is satisfied if and only if $nc \ge 1 + 1/n$ and $0\le ncx\le 1$, so we can take $y = cx/(nc-1)\le 1$ and find
\[
\lp 
1- \frac{cx}{nc-1} 
\rp^n \le \frac{1}{1+n\frac{cx}{nc-1}} 
= 1 - \frac{nc}{nc-1+ncx}x\le 1 - x. 
\]  
       The fourth inequality holds due to
    \begin{align*}
          1- (1-x)^n  =  x\sum_{k=0}^{n-1} (1-x)^k\le    nx.
    \end{align*}
\end{proof}

\subsubsection{Proof of Lemma~\ref{lem:lower_bound_detR}}
\label{app:proof_of_lower_bound_detR}

    Recall from Lemma~\ref{lem:definition_of_R} that 
      \[
W^\# = W^*R + (N^*-N^\#)P = W^*R + M.
\]
  Since $W^\#$ is feasible for the min-vol NMF problem \eqref{eq:apprminvol}, its volume is larger than the volume of $W^*$, so 
\begin{align*}
 \det((W^*)^\top W^*)\le \det((W^\#)^\top W^\#) = \det((W^*R + M)^\top (W^*R+M)).
\end{align*}
Notice that, by Lemma~\ref{lem:norm_inequalities} and Lemma~\ref{lem:definition_of_R}, 
\[
\|M\| \le \sqrt r \|M \|_{1,2} 
\le 4   \sqrt{r(r-1)} \frac pq \ve
 = \mac O\left(\frac{\sigma_r(W^\#)}{r}\right).
\]
Using Lemma~\ref{lem:perturbation_on_volume}, we get
\begin{align*}
   \det((W^\#)^\top W^\#)&\le \det((W^*R)^\top W^*R)+ \prod_i [\sigma_i(W^*R) + \|M  \|]^2  -\prod_i  \sigma_i(W^*R)^2\\
    &=  \det(R^\top R)   \det((W^*)^\top W^*)  +\prod_i [\sigma_i(W^\# - M) + \|M  \| ]^2
    -\prod_i\sigma_i(W^\# - M)^2\\
     &\le \det(R^\top R) \det((W^\#)^\top W^\#) +\prod_i [\sigma_i(W^\#) + 2\|M\| ]^2
    -\prod_i[\sigma_i(W^\#) - \|M\|]^2 , 
    \end{align*} 
where the last inequality holds since $\|M\| = \mac O(\sigma_r(W^\#))$.
One can then obtain a lower bound to $\det(R)^2$ since $W^\#$ is full rank, so $\det((W^\#)^\top W^\#) \ne 0$ and
\begin{align*}
 \det(R^\top R) &\ge 1 -\frac{ 
 \prod_i [\sigma_i(W^\#) + 2\|M\| ]^2
    -\prod_i[\sigma_i(W^\#) - \|M\|]^2
 }{\det((W^\#)^\top W^\#)}\\
   &\ge 1 - \left[ \prod_i \left[1 + \frac{2\|M\|}{\sigma_i(W^\#)} \right]^2
    -\prod_i  \left[1- \frac{\|M\|}{\sigma_i(W^\#)}\right]^2 \right]
    \\
   &\ge 1 - \left[  \left[1 + \frac{2\|M\|}{\sigma_r(W^\#)} \right]^{2r}
    -  \left[1- \frac{\|M\|}{\sigma_r(W^\#)}\right]^{2r} \right].
\end{align*}
Keeping in mind that $\|M\|  = \mac O\left(\frac{\sigma_r(W^\#)}{r}\right)$, and using Lemma~\ref{lem:bound_perturbation_power}, we find 
\begin{align*}
  \det(R)^2 &\ge 1 - \left[  \left[1 + \frac{2\|M\|}{\sigma_r(W^\#)} \right]^{2r}
    -  \left[1- \frac{\|M\|}{\sigma_r(W^\#)}\right]^{2r} \right]\\
    &= 1 - \left[  \left[1 + \mac O\left(     \frac{r\|M\|}{\sigma_r(W^\#)}     \right)\right]
   -  \left[1-\mac O\left(     \frac{r\|M\|}{\sigma_r(W^\#)}     \right)\right] \right]\\
   &=  1 - \mac O\left(     \frac{r\|M\|}{\sigma_r(W^\#)}     \right)
  =  1 -\mac O\left(     \frac{r^2}{\sigma_r(W^\#)}   \frac pq \ve  \right).
  \end{align*}

\subsubsection{Proof of Lemma~\ref{lem:full_rank_W}}
\label{app:full_rank_W}

From Lemma~\ref{lem:definition_of_R}, 
\[
W^\# H_p^\top = W^\#(H^\#)^\top V = W^*(H^*)^\top V+ (N^*-N^\#)V.
\]
Notice that $H^\top V\in \R^{r\times r}$ is also column stochastic, so using Lemma~\ref{lem:norm_inequalities}, Lemma~\ref{lem:lower_bound_last_sv},  Lemma~\ref{lem:last_sv_of_Hp} and the perturbation theorem of singular values,
\begin{align*}
    \sigma_r(W^*) &= \frac{\sigma_r(W^*)\|(H^*)^\top V\| }{\|(H^*)^\top V\|} \ge \frac{\sigma_r(W^*(H^*)^\top V) }{\|(H^*)^\top V\|_F} = 
    \frac{\sigma_r(W^\# H_p-  (N^*\text{$-$}N^\#)V ) }{\sqrt{\sum_{i,j} ((H^*)^\top V)_{i,j}^2 }}\ge 
    \frac{\sigma_r(W^\# H_p) -  \| (N^*\text{$-$}N^\#)V \| }{\sqrt{\sum_{i,j} ((H^*)^\top V)_{i,j} }}\\
    &\ge 
    \frac{\sigma_r(W^\#) \sigma_r(H_p) -  \sqrt r\| (N^*\text{$-$}N^\#)V \|_{1,2} }{\sqrt r}
     \ge 
    \frac{\sigma_r(W^\#) \frac 1{\sqrt{r-1}}\frac qp -  
    \sqrt r\| N^*\text{$-$}N^\#\|_{1,2}\|V\|_1 }{\sqrt r}
    \\
     &\ge 
      \frac{\sigma_r(W^\#)   }{\sqrt {r(r-1)}}
\frac qp -
    \frac{  2\ve \sqrt r }{\sqrt r}
 = 
 \frac{\sigma_r(W^\#)   }{\sqrt {r(r-1)}}
\frac qp
-  2\ve.
\end{align*}
This is enough to conclude that $W^*$ is full rank for $ \frac{\sigma_r(W^\#)   }{\sqrt {r(r-1)}}
\frac qp
\ge 2\ve$.

\section{Proof of Theorem~\ref{th:main2}}
\label{app:proof_of_theo_sep}

\subsection{Proof of Lemma~\ref{lem:lower_bound_R_sepcase}}
\label{app:proof_to_lower_bound_R_sepcase}

    Recall from Lemma~\ref{lem:def_of_Hp} and Lemma~\ref{lem:last_sv_of_Hp} that
\[
H_p=  \alpha_p E + (1- r\alpha_p) I, \qquad 
H_p^{-1}=(1- r\alpha_p)^{-1} (-\alpha_p E +  I),\qquad
1-r\alpha_p = \frac 1{\sqrt{r-1}}\frac qp.
\]
Since $1-\alpha_p> 0$ and $\alpha_p\ge 0$, each column of $H_p^{-1}$ has exactly one positive entry. 
From Lemma~\ref{lem:definition_of_R} recall that $R = (H^*)^\top V H_p^{-1}$, where  $(H^*)^\top V$ is nonnegative and column stochastic. In particular, $0\le (H^*)^\top V\le E$ and \[
R=(H^*)^\top VH_p^{-1} =(1- r\alpha_p)^{-1} (H^*)^\top V(-\alpha_p E +  I).
\]
As a consequence, each entry $r_{i,j}$ of $R$ is bounded from below by
\begin{align*}
    r_{i,j} &= \sum_k (1- r\alpha_p)^{-1}((H^*)^\top V)_{i,k}  (-\alpha_p E +  I)_{k,j}
\\&= (1- r\alpha_p)^{-1}\left[ 
((H^*)^\top V)_{i,j}  (1-\alpha_p) - \alpha_p 
\sum_{k\ne j} ((H^*)^\top V)_{i,k} 
\right]
\ge  -(r-1)\alpha_p (1- r\alpha_p)^{-1} := -\beta_p.
\end{align*}
Notice that in the separable case, that is, for $p=1$, $\alpha_p = \beta_p = 0$. In particular $\beta_p \approx p-1$ for $p\to 1$. Here we show that for $p = 1 + \mac O(1/r)$ then $\beta_p = \mac O(p-1)=  \mac O(1/r)$. In fact, from Lemma~\ref{lem:last_sv_of_Hp},
\begin{align*}
    \beta_p & = \frac{r-1}{r}\frac{r\alpha_p}{1-r\alpha_p}
    = \frac{r-1}{r}\left[ 
    \frac{1}{1-r\alpha_p}
    -1
            \right]
             = \frac{r-1}{r}\left[ 
  p \sqrt{\frac {r-1}{r-p^2} }
    -1 
            \right] , 
\end{align*}
but 
\begin{align*}
    \sqrt{\frac {r-1}{r-p^2} }
 & = \sqrt{\frac {r-1}{r-1 - \mac O(p-1)} }
 = \sqrt{\frac {1}{1 - \mac O(\frac{p-1}{r-1})} }
  = \sqrt{1 +  \mac O\left(\frac{p-1}{r-1}\right)} 
  = 1 +  \mac O\left(\frac{p-1}{r-1}\right) = \mac O(1),
\end{align*}
so
\begin{align*}
    \beta_p & = \frac{r-1}{r}\left[ 
  p \sqrt{\frac {r-1}{r-p^2} }
    -1
            \right]
            = \mac O(p-1) + \frac{r-1}{r}\left[ 
  \sqrt{\frac {r-1}{r-p^2} }
    -1
            \right]
            = \mac O(p-1) + \mac O\left(\frac{p-1}{r}\right)
              = \mac O(p-1).
\end{align*}
In particular, this also shows that
\begin{equation*}
    \frac pq\sqrt r =     \sqrt{\frac {p}{r-p^2} r} = \mac O   \left( \sqrt{\frac {r}{r-1} }\right) = \mac O(1).\label{eq:pq_asymptotic_sepcase}
\end{equation*}
Eventually, since $\|(H^*)^\top V\|_1 = 1$, 
\begin{align*}
   \|R\|_1 = \| (H^*)^\top VH_p^{-1}\|_1\le \|H_p^{-1}\|_1 =
    \frac{1-\alpha_p + (r-1)\alpha_p}{1- r\alpha_p} = 1 + 2\beta_p.
\end{align*}

\subsection{Last steps of the proof} \label{app:rest_of_proof_th_main2}

    Recall from Lemma~\ref{lem:definition_of_R} that $e^\top R = e^\top$, so the entries of each column $ r_i$ of $R$ sum up to $1$.  
Let us now fix a column $r_i$ of $R$ and suppose that $r_{k,i}$ is its largest positive entry. We want to show that $r_{k,i}$ is close to $1$.  According to Lemma~\ref{lem:lower_bound_R_sepcase},  $r_{k,i}\le \|r_i\|_1\le \|R\|_1\le 1+2\beta_p$ and $\beta_p = \mac O(1/r)$, so we have an easy upper bound. 

For the lower bound, we need first to prove that $r_{k,i} = \|r_i\|_\infty$.
Call  $r_{\ell,i}$ the minimum entry of $r_i$ and notice that if $r_{\ell,i}\ge 0$ then $0\le r_{\ell,i}\le r_{k,i} = \|r_i\|_\infty$. We can thus suppose  $0>r_{\ell,i}\ge -\beta_p\ge -1/(r-2)$ and find that $$(r-1)|r_{\ell,i}| \le 1 + |r_{\ell,i}| = \sum_{j\ne \ell} r_{j,i}\le (r-1)r_{k,i}\implies 
 |r_{\ell,i}| \le r_{k,i}\implies r_{k,i} = \|r_i\|_\infty.$$
Notice now that, by Lemma~\ref{lem:lower_bound_R_sepcase}, 
 \[
 \ve = \mac O\left(
    \frac{\sigma_r(W^\#)   }{r\sqrt r}
\right)
 = \mac O\left(
    \frac{\sigma_r(W^\#)   }{r^2}\frac qp
\right)
\frac pq\sqrt r
 = \mac O\left(
    \frac{\sigma_r(W^\#)   }{r^2}\frac qp
\right).
 \]
We can thus apply Lemma~\ref{lem:lower_bound_detR}, and find that the matrix $R$ satisfies
   \[
    \det(R)^2 \ge   1 -\mac O\left(     \frac{r^2}{\sigma_r(W^\#)}   \frac pq \ve  \right),
   \]
so we can use the Hadamard theorem and the above estimate $\|r_i\|_1\le 1+2\beta_p$ to compute the lower bound
\begin{align*}
    r_{k,i} = \|r_i\|_{\infty}
    \ge \frac{\|r_i\|^2}{\|r_i\|_1}
    = \frac{\prod_j\|r_j\|^2}{\|r_i\|_1\prod_{j\ne i}\|r_j\|^2}
    \ge \frac{\det(R)^2}{\|r_i\|_1\prod_{j\ne i}\|r_j\|_1^2}
    \ge \frac{1 -\mac O\left(     \frac{r^2}{\sigma_r(W^\#)}   \frac pq \ve  \right)}{(1+2\beta_p)^{2r-1}}.
\end{align*}
Since $\beta_p = \mac O(1/r)$, we can apply Lemma~\ref{lem:bound_perturbation_power} and find that
 \begin{align*}
    r_{k,i} 
    \ge \frac{1 -\mac O\left(     \frac{r^2}{\sigma_r(W^\#)}   \frac pq \ve  \right)}{(1+2\beta_p)^{2r-1}}
    \ge \frac{1 -\mac O\left(     \frac{r^2}{\sigma_r(W^\#)}   \frac pq \ve  \right)}{1+\mac O(r\beta_p)}
     \ge 1 -\mac O\left(     \frac{r^2}{\sigma_r(W^\#)}   \frac pq \ve  + r\beta_p\right).
\end{align*}
Using both the upper and the lower bound on $r_{k,i}$, we find that
\[
|1-r_{k,i}|  = \mac O\left(     \frac{r^2}{\sigma_r(W^\#)}   \frac pq \ve  + r\beta_p\right). 
\]
This is enough to show that $r_i$ is close to the canonical basis vector $e_k$.
\begin{align}\label{eq:error_r_i_sepcase}
 \nonumber   \|e_k-r_i\| &\le \|e_k-r_i\|_1 = |1-r_{k,i}| + \sum_{j\ne i} |r_{j,i}| = 
    |1-r_{k,i}| + \|r_i\|_1 - r_{k,i} \\&\le
    |1-r_{k,i}| +1+2\beta_p - r_{k,i}
    \le
    2\beta_p  + 2|1-r_{k,i}| = \mac O\left(     \frac{r^2}{\sigma_r(W^\#)}   \frac pq \ve  + r\beta_p\right). 
\end{align}
To conclude, we need to show that two different columns of $R$ are not close to the same $e_k$. 
Let $R=QT$ be a $QR$ decomposition of $R$ with $Q$ orthogonal and $T$ upper triangular. If $t_i$ are the columns of $T$, then $r_i = Qt_i$ and thus $|t_{i,i}|\le \|t_i\| = \|r_i\|\le 1 + 2\beta_p$. Notice moreover that $\det(R)^2 = \det(T)^2=\prod_i|t_{i,i}|^2$. 
If now $j> i$ and $\beta_p=\mac O(1/r)$, then again by Lemma~\ref{lem:bound_perturbation_power}, 
\begin{align*}
    \|r_i-r_j\|^2  = \|t_i - t_j\|^2 \ge |t_{j,j}|^2 = \frac{\det(R)^2}{\prod_{k\ne j}|t_{k,k}|^2} \ge \frac{ 1 -\mac O\left(     \frac{r^2}{\sigma_r(W^\#)}   \frac pq \ve  \right)}{(1+2\beta_p)^{2r-2}} =  1 -\mac O\left(     \frac{r^2}{\sigma_r(W^\#)}   \frac pq \ve  + r\beta_p\right).
\end{align*}
If we suppose that both $r_i$ and $r_j$ are close to the same $e_k$ in the sense of \eqref{eq:error_r_i_sepcase}, then 
\begin{align*}
    \|r_i-r_j\|^2  \le  (\|r_i-e_k\|  + \|e_k-r_j\|)^2  = \mac O\left(     \frac{r^2}{\sigma_r(W^\#)}   \frac pq \ve  + r\beta_p\right) <  1 -\mac O\left(     \frac{r^2}{\sigma_r(W^\#)}   \frac pq \ve  + r\beta_p\right), 
    \end{align*}
 a contradiction. As a consequence, each $r_j$ is close to a different $e_k$ and we can conclude that  
\begin{equation}
    \label{eq:R_close_to_permutation_sepcase}
\min_\Pi \|R-\Pi\|_{1,2} \le \mac O\left(     \frac{r^2}{\sigma_r(W^\#)}   \frac pq \ve  + r\beta_p\right).
\end{equation}
 Due to Corollary~\ref{cor:Error_estimation_given_R_is_permutation},  we get
\begin{align*}  \min_\Pi \|W^\#-W^*\Pi\|_{1,2} &\le    \|W^*\|\min_\Pi \|R-\Pi\|_{1,2}    +4r\ve \\
    &\le    \|W^*\| \cdot \mac O\left(     \frac{r^2}{\sigma_r(W^\#)}   \frac pq \ve  + r\beta_p\right)  +4r\ve.   
\end{align*}
If $\Pi$ is the permutation matrix satisfying \eqref{eq:R_close_to_permutation_sepcase}, then  Lemma~\ref{lem:norm_inequalities}, Lemma~\ref{lem:lower_bound_R_sepcase} and Lemma~\ref{lem:definition_of_R} imply that
\begin{align*}
  \|W^\#\|&\le \|W^*\|\|R\| + \|M\| 
  \le \sqrt r\|W^*\|\|R\|_{1,2} + \sqrt r\|M\|_{1,2}  \\
  &\le \sqrt r\|W^*\|(1 + \|R-\Pi\|_{1,2}) + 2r\sqrt r\ve 
  = \mac O(\sqrt r)\|W^*\| + \mac O(\sigma_r(W^\#))\\
  \implies \frac{\|W^*\|}{\sigma_r(W^\#)} &\ge 
  \frac 1{\mac O(\sqrt r)}
  \left(\frac{\|W^\#\|}{\sigma_r(W^\#)} - \mac O(1)\right)
  = \Omega\left( \frac1{\sqrt r} \right)\\
   \implies 4r\ve &\le 4r\sqrt r \ve \mac O\left(  \frac{\|W^*\|}{\sigma_r(W^\#)} \right) = 
   \|W^*\|\cdot \mac O\left(     \frac{r\sqrt r}{\sigma_r(W^\#)}  \ve \right) , 
\end{align*}
and, as a consequence, from \eqref{eq:pq_asymptotic_sepcase} and $\beta_p=\mac O(p-1)$ we conclude that
\begin{align*}   \min_\Pi \| W^\#-W^* \Pi\|_{1,2}&\le    \|W^*\| \cdot \mac O\left(     \frac{r^2}{\sigma_r(W^\#)}   \frac pq \ve  + r\beta_p\right)  +4r\ve\\
& = 
\|W^*\| \cdot \mac O\left(     \frac{r\sqrt r}{\sigma_r(W^\#)}  \ve  + r(p-1)\right).
\end{align*}

\section{Proof of Theorem~\ref{th:main}}
\label{app:proof_of_theo_pSSC}

\subsection{ $R$ is close to an orthogonal matrix}
 
 \subsubsection{Proof of Lemma~\ref{lem:upper_bound_detR}}
\label{app:proof_upper_bound_detR}

   From Lemma~\ref{lem:definition_of_R} and Assumption~\ref{ass:perturbed_pSSC}, $W^\# = WR + (N-N^\#)P$ and $W^\#(H^\#)^\top + N^\# = WH^\top + N$ with $R = H^\top VH_p^{-1}$ and $P = VH_p^{-1}$. As a consequence,   
\[(WR + (N-N^\#)P)(H^\#)^\top + N^\# = WH^\top + N \implies R(H^\#)^\top = H^\top + W^\dagger (N-N^\#)(I -P(H^\#)^\top).
\]
Since each element of a matrix is bounded in absolute value by the norm $\|\cdot\|_{1,2}$ of the matrix, and since from  Lemma~\ref{lem:last_sv_of_Hp}   $\|P\|_1\le \|H_p^{-1}\|_1\le 2\sqrt{r-1}\frac pq-1$, we can compute the lower bound
\begin{align*}
    R(H^\#)^\top & \ge -  \| W^\dagger (N-N^\#)(I -P(H^\#)^\top)\|_{1,2} 
\ge -  \| W^\dagger\| \|(N-N^\#)\|_{1,2} \|(I -P(H^\#)^\top)\|_{1}
\\& 
\ge  
-  \frac{ 2(1 +\|H_p^{-1}\|_1 )}{\sigma_r(W)}\ve 
\ge  
-  \frac{ 4\sqrt{r-1} }{\sigma_r(W)}\frac pq\ve. 
\end{align*}
Substituting $\ve  = \mac O(\frac {\sigma_r(W^\#)}{r}\frac qp)$ into Lemma~\ref{lem:full_rank_W}, 
\[
\sigma_r(W) \ge \frac {\sigma_r(W^\#)}{\sqrt {r(r-1)}}\frac qp - 2\ve =\Omega\left(
\frac {\sigma_r(W^\#)}{\sqrt {r(r-1)}}\frac qp
\right)
\implies R(H^\#)^\top \ge -\gamma_p\ve =  - \mac O\left( \frac{ r\sqrt r}{\sigma_r(W^\#)}\frac {p^2}{q^2}\ve\right)
\]
and thus 
$(R + \gamma_p\ve ee^\top)(H^\#)^\top\ge 0$ where $\gamma_p\ge 0$. 
Since $H^\#$ is $p$-SSC, 
\[
\cone(R^\top + \gamma_p\ve ee^\top) \cu \cone((H^\#)^\top)^*\cu \mac C_p^*\cu \mac C^*.
\]
In particular, denoting $\wt r_i$ the $i$th row of $R$,   
\[
e^\top \wt r_i + r\gamma_p \ve
=e^\top (\wt r_i +\gamma_p \ve e)\ge \|\wt r_i +\gamma_p \ve e\| \ge \|\wt r_i\| -\sqrt r\gamma_p \ve .
\]
\subsubsection{Proof of Lemma~\ref{lem:R_is_close_to_orthogonal}}
\label{app:proof_of_R_close_to_orthogonal}
 
Given $R^\top = QT$ a QR factorization, with $t_i$ being the columns of $T$, we have $\|\wt r_i\| = \|Qt_i\| = \|t_i\| \ge | t_{i,i}|$, and $|\det(R)| = |\det(T)| = \prod_i |t_{ii}|$. 
Notice that, due to Lemma~\ref{lem:upper_bound_detR}, we can use AM-GM on three different sets of elements $z_i$:
\begin{align*}
    z_i = e^\top \wt r_i   + 2r\gamma_p \ve & & \to  & &\prod_i \|\wt r_i\|^2  \le \prod_i( e^\top \wt r_i   + 2r\gamma_p \ve)^2  \le\left( \sum_i \frac{e^\top \wt r_i + 2r\gamma_p}r\right)^{2r}, \\
        z_i = \|\wt r_i\| & & \to  & &\prod_i \|\wt r_i\|^2  \le \left( \frac{\sum_i\|\wt r_i\| }{r}  \right)^{2r} \le \left( \sum_i \frac{e^\top \wt r_i + 2r\gamma_p}r\right)^{2r}, 
        \\
        z_i = |t_{ii}| & & \to  & &\prod_i t_{i,i}^2\le  \left( \frac{\sum_i|t_{i,i}| }{r}  \right)^{2r} \le  \left( \frac{\sum_i\|\wt r_i\| }{r}  \right)^{2r}\le \left( \sum_i \frac{e^\top \wt r_i + 2r\gamma_p}r\right)^{2r}, 
\end{align*}
and in all three cases the quantities are lower bounded by $|\det(R)|^2 = |\det(T)|^2$ due to the Hadamard theorem. If $A(z_i)$ is the AM and $G(z_i)$ is GM of the respective elements, then 
\[
|\det(R)|^2  \le G(z_i)^{2r} \le A(z_i)^{2r} \le\left( \sum_i \frac{e^\top \wt r_i + 2r\gamma_p}r\right)^{2r}  =\left( 1  + 2r\gamma_p \ve\right)^{2r}.
\]
Due to 
Lemma~\ref{lem:lower_bound_detR}, {  for $\ve  = \mac O(\frac {\sigma_r(W^\#)}{r^2}\frac qp)$} and 
since $\gamma_p = \mac O\left( \frac{ r\sqrt r}{\sigma_r(W^\#)}\frac {p^2}{q^2}\right)$, we have
\[
 1 - \mac O\left(\frac {r^2}{\sigma_r(W^\#)}\frac pq\ve\right)
   \le G(z_i)^{2r} \le A(z_i)^{2r}   \le \left( 1  + \mac O\left( \frac{ r^2\sqrt r}{\sigma_r(W^\#)}\frac {p^2}{q^2}\ve\right) \right)^{2r}
\]
and,  for $\ve  = \mac O(\frac {\sigma_r(W^\#)}{r^3}\frac qp)$, we can use Lemma~\ref{lem:bound_perturbation_power} and take the $2r$-th root of all terms to find
\[
 1 - \mac O\left(\frac {r}{\sigma_r(W^\#)}\frac pq\ve\right)
   \le G(z_i) \le A(z_i)   \le  1  + \mac O\left( \frac{ r^2\sqrt r}{\sigma_r(W^\#)}\frac {p^2}{q^2}\ve\right) 
\]
and thus conclude that
\[
 A(z_i) - G(z_i) = \mac O\left( \frac{ r^2\sqrt r}{\sigma_r(W^\#)}\frac {p^2}{q^2}\ve\right). 
\]
Using Lemma~\ref{lem:AM-GM}, 
\begin{align}
\nonumber
    \max_i|z_i-1| &= \max\{ |z_1-1|, |1-z_r| \}\le  \max\{ z_1-A(z_i) + |A(z_i)-1|, G(z_i)-z_r + |1-G(z_i)| \}\\ \nonumber
    &\le z_1-z_r + \max\{ |A(z_i)-1|,  |1-G(z_i)| \} 
    \\ \nonumber &\le (\sqrt{z_1}+\sqrt{z_r})\sqrt{r(A(z_i)-G(z_i))} 
    + \max\{ |A(z_i)-1|,  |1-G(z_i)| \}
    \\
    &\label{eq:robust_AMGM-application}
    \le \mac O\left( 
    \sqrt{\frac{ x_1r^3\sqrt r}{\sigma_r(W^\#)}\frac {p^2}{q^2}\ve}
    \right)+ 
    \mac O\left( \frac{ r^2\sqrt r}{\sigma_r(W^\#)}\frac {p^2}{q^2}\ve\right)
    \end{align}
Recall that from Lemma~\ref{lem:definition_of_R},  $e^\top Re = r$ and $Re = H^\top VH_p^{-1}e = H^\top V e\ge 0$, so $e^\top \wt r_i = (Re)_i\le r$.  Now from Lemma~\ref{lem:upper_bound_detR}    with $\ve  = \mac O(\frac {\sigma_r(W^\#)}{r^2}\frac qp)$, we find that for all $i$, 
\[
|t_{i,i}| \le 
\|\wt r_i\| \le e^\top \wt r_i + 2r\gamma_p\ve  \le  r + 2r\gamma_p\ve
 = \mac O(r).
 \]
Hence, the same bound holds for all $z_i$, that is, $|z_i| \leq O(r)$. Using     $\ve  = \mac O(\frac {\sigma_r(W^\#)}{r}\frac qp)$, the relation  \eqref{eq:robust_AMGM-application}  can thus be estimated as
\[
\max_i|z_i-1| 
    =  \mac O\left( 
    \sqrt{\frac{ x_1r^3\sqrt r}{\sigma_r(W^\#)}\frac {p^2}{q^2}\ve}
    \right)+ 
    \mac O\left( \frac{ r^2\sqrt r}{\sigma_r(W^\#)}\frac {p^2}{q^2}\ve\right)
    =  \mac O\left( 
    \sqrt{\frac{ r^4\sqrt{r}}{\sigma_r(W^\#)}\frac {p^2}{q^2}\ve}
    \right). 
\] 
If now we suppose    $\ve  = \mac O(\frac {\sigma_r(W^\#)}{r^4\sqrt{r}}\frac {q^2}{p^2})$, then $z_i = \mac O(1)$ and the relation  \eqref{eq:robust_AMGM-application}  reads as 
\[
\max_i|z_i-1| 
    =  \mac O\left( 
    \sqrt{\frac{ x_1r^3\sqrt r}{\sigma_r(W^\#)}\frac {p^2}{q^2}\ve}
    \right)+ 
    \mac O\left( \frac{ r^2\sqrt r}{\sigma_r(W^\#)}\frac {p^2}{q^2}\ve\right)
    =  \mac O\left( 
    \sqrt{\frac{ r^3\sqrt r}{\sigma_r(W^\#)}\frac {p^2}{q^2}\ve}
    \right).
\]
This gives us that $\max_i\Big|\|\wt r_i\| - 1\Big| \leq \mac O\left( 
    \sqrt{\frac{ r^\frac{7}{2}}{\sigma_r(W^\#)}\frac {p^2}{q^2}\ve}\right)$. Moreover,
    \begin{align*}
        \max_i |e^\top \wt r_i - 1| \leq \max_i |(e^\top \wt r_i + 2r\gamma_p \varepsilon) - 1| + |2r\gamma_p \varepsilon| &\leq \mac O\left( 
    \sqrt{\frac{ r^\frac{7}{2}}{\sigma_r(W^\#)}\frac {p^2}{q^2}\ve}\right) + \mac O\left( 
    \frac{ r^\frac{5}{2}}{\sigma_r(W^\#)}\frac {p^2}{q^2} \varepsilon\right) \\ 
    &= \mac O\left( 
    \sqrt{\frac{ r^\frac{7}{2}}{\sigma_r(W^\#)}\frac {p^2}{q^2}\ve}\right).
    \end{align*}
Eventually, letting the $t_i$'s be the columns of the upper triangular matrix $T$, then $Qt_i = \wt r_i$ and in particular for any $i< j$,
    \begin{align*}
        \|\wt r_i-\wt r_j\| &= \|t_i-t_j\| \ge |t_{j,j}|  \ge 1 - \mac O\left( 
    \sqrt{\frac{ r^3\sqrt r}{\sigma_r(W^\#)}\frac {p^2}{q^2}\ve}
    \right).
    \end{align*}

This result can be used to show that  $R^\top$ is close to an orthogonal matrix, since $R = QT$ and $T$ is almost diagonal with diagonal elements close in magnitude to $1$, but that is not necessary to complete the proof of Theorem~\ref{th:main}. 

\subsection{Geometric Intuition}
\label{app:geometric_intuition_th_pSSC}

\subsubsection{Bound on $\ve$ for the disjointness}

We start by introducing some notation. The main objective is to focus on the affine subspace of vectors that have the same entrywise sum of a fixed row $\wt r_k$.
\begin{notation}
    \label{not:geometric_notations}
    $ $
    
\begin{itemize}
    \item  By Lemma~\ref{lem:R_is_close_to_orthogonal}, if $\ve  = \mac O(\frac {\sigma_r(W^\#)}{r^{9/2}}\frac {q^2}{p^2})$ then there exists a parameter $\vf_p\ge 0$ such that
\[
\max\{  |  \|\wt r_k\| -1 |,   |e^\top \wt r_k -1 |  \}\le \vf_p \sqrt \ve, 
\qquad  \min_{i\ne j} \|\wt r_i-\wt r_j\| \ge  1 - \vf_p\sqrt\ve,\qquad 
 \vf_p= \mac O\left( 
    \sqrt{\frac{ r^{7/2}}{\sigma_r(W^\#)}\frac {p^2}{q^2}}
    \right).\]
    \item Let $\mac B_s: = \{z\in \f R^r\ | \ \| z\| <s\}$ be the open ball centered in $0$ with radius $s\ge  0$.
    \item Let $\mac H_\beta := \{z\in \f R^r \ | \ e^\top z = \beta \}$ be the affine subspace of vectors with entry-wise sum equal to $\beta$.
    \item By Lemma~\ref{lem:upper_bound_detR}, we know that for $\ve  = \mac O(\frac {\sigma_r(W^\#)}{r}\frac qp)$, $\wt r_k \in  \mac C_p^*-\gamma_p\ve e$. Keeping in consideration the previous notation, if we fix an index $k$ and define 
\[\mac P := \mac H_{e^\top \wt r_k} \cap (\mac C_p^*-\gamma_p\ve e), \qquad \mac B:= \mac H_{e^\top \wt r_k}\cap \mac B_{1-\vf_p\sqrt \ve},   \]
then necessarily $\wt r_k\in \mac P\setminus\mac B$. 
\item Recall from Lemma~\ref{lem:Cp_Sp_dual_cones_and_inclusions} that
\[
\mac S_p = \left\{x \in \R^r \ \big| \ e^{\top}x \geq p \|x\| \right\}, \qquad \mac S_q = \mac S_p^* \cu \mac C_p^*,
\]
and define
\[
\wt {\mac S} := \lp \mac S_q -\gamma_p\ve e\rp \cap \mac H_{e^\top \wt r_k} \; \cu \; \mac P.
\]
\item For a fixed index $k$ and any index $j$, define
\[ 
\wt e_j  := \lp e^\top \wt r_k + \gamma_p\ve r \rp e_j - \gamma_p\ve e \; \in \;   \mac H_{e^\top \wt r_k} \cap (\mac C_p^*-\gamma_p\ve e) \; = \; \mac P. 
\]
Moreover, let $\wt e_{i,j}:=\frac{\wt e_i+\wt e_j}{2}$ be the middle points. 
\end{itemize} 
\end{notation}
Figure~\ref{fig:p-SSC_identifiability} illustrates this notation.  
\begin{lemma}\label{lem:Pcharacterization}
  Under Notation~\ref{not:geometric_notations}, we have
   \begin{align*}
   \mac  P  
= \conv\left( \wt{\mac S}  \cup \{\wt e_1,\dots ,\wt e_r\} \right).
\end{align*}
\end{lemma}
\begin{proof}
     Let $\beta= e^\top \wt r_k$, by Lemma~\ref{lem:Cp_Sp_dual_cones_and_inclusions}, 
      \begin{align*}
   \mac  P = \mac H_\beta\cap ( \mac C_p^* - \gamma_p\ve e) = 
    \mac H_\beta\cap ( \cone(\mac S_q \cup \{e_1,\dots,e_r\} ) - \gamma_p\ve e).
\end{align*}
Any $v\in \mac P$ can be written as
\[
v = -\gamma_p\ve e + w + \sum_i \lambda_i e_i, \qquad \lambda_i\ge 0, \quad w\in \mac S_q, \quad 
\beta = e^\top v =  -\gamma_p\ve r + e^\top w + \sum_i \lambda_i. 
\]
Notice that $w\in \mac S_q\implies e^\top w\ge q\|w\|\ge 0$, and using  Notation~\ref{not:geometric_notations}, we have $\beta = e^\top \wt r_k \ge 1/2> 0$.  The vector $v$ can thus be rewritten as
\[
v = \frac{e^\top w}{\beta + \gamma_pr \ve}\lp  
 \frac{\beta + \gamma_pr \ve}{e^\top w} w - \gamma_p\ve e 
\rp
+ \sum_i \frac{\lambda_i}{\beta + \gamma_pr \ve} \wt e_i
\in 
\conv\left( \wt{\mac S}  \cup \{\wt e_1,\dots ,\wt e_r\} \right), 
\]
where if $e^\top w=  0$ then $w = 0$ and $v =  -\gamma_p\ve e  + \sum_i \lambda_i e_i =  \sum_i \frac{\lambda_i}{\beta + \gamma_pr \ve} \wt e_i
\in 
\conv\left( \{\wt e_1,\dots ,\wt e_r\} \right)$. This proves that $\mac P \cu \conv\left( \wt{\mac S}  \cup \{\wt e_1,\dots ,\wt e_r\} \right)$. 

To prove the opposite containment, let $v\in \conv\left( \wt{\mac S}  \cup \{\wt e_1,\dots ,\wt e_r\} \right)$ that can be written as 
\[
v = \mu \wt  w + \sum_i \lambda_i \wt e_i, \qquad \mu, \lambda_i\ge 0, \quad 
1 = \mu+ \sum_i \lambda_i, 
\quad \wt w\in \wt {\mac S}. 
\]
Since $e^\top \wt e_i = e^\top \wt w = \beta$, $e^\top v = \beta$ and $v\in \mac H_\beta$. 
From the definitions in Notation~\ref{not:geometric_notations}, we get
\[
v = \mu(w-\gamma_p\ve e)+ \sum_i \lambda_i ((\beta + \gamma_p\ve r)e_i-\gamma_p\ve e) = \lp\mu w + \sum_i\lambda_i(\beta + \gamma_p\ve r)e_i\rp - \gamma_p\ve e, 
\quad  w\in  {\mac S_q}, 
\]
so $v \in  \mac H_\beta\cap ( \cone(\mac S_q \cup \{e_1,\dots,e_r\} ) - \gamma_p\ve e) = \mac P$ and the reserve containment is proved. 

\end{proof}

Let us now show that both $\mac B$ and $\wt{\mac S}$ are spheres in the space $\mac H_{\beta}$ with the same center $\beta e/r$ where $\beta = e^\top \wt r_k$. 

\begin{lemma}
\label{lem:S_B_and_eij}    
Under Notation~\ref{not:geometric_notations}, 
if $\beta = e^\top \wt r_k$, then 
    \begin{align*}
        \mac{\wt S} 
         &=  \left\{\beta e/r + w  \in \R^r \ \big| \ e^{\top}w = 0, \ \ 
     \| w\|^2 \le  ( \beta + \gamma_pr\ve)^2 \left( \frac{1}{q^2} -   \frac 1r\right)
       \right\},\\
         \mac B 
      &=  \left\{\beta e/r + w  \in \R^r \ \big| \ e^{\top}w = 0,\ \ 
     \|w\|^2\le (1-\vf_p\sqrt \ve)^2 -  \beta^2/r  \right\},\\
      \|\wt e_{i,j}-  \beta e/r\|^2
     &=
     \frac{r-2}{2r}(\beta + r\gamma_p\ve)^2,\quad \forall i\ne j.
    \end{align*}
    In particular,  $\wt e_{i,j}\in \wt {\mac S} \iff 2\ge q^2 \ge 1$. 
\end{lemma}
\begin{proof}
   Let us rewrite $\mac {\wt S}$ and $\mac B$ as spheres inside $\mac H_\beta$ both with center $\beta e/r$. 
\begin{align*}
    \mac{\wt S} &= [\mac S_q  - \gamma_p\ve e ] \cap \mac H_\beta\\
    &=  \left\{x - \gamma_p\ve e  \in \R^r \ \big| \ e^{\top}x \geq q \|x\|,\ \ \beta = e^\top (x - \gamma_p\ve e ) \right\}\\
     &=  \left\{v  \in \R^r \ \big| \ e^{\top}(v+\gamma_p\ve e) = \beta + \gamma_pr\ve \geq q \|v+\gamma_p\ve e\| \right\}\\
     &=  \left\{\beta e/r + w  \in \R^r \ \big| \ e^{\top}(\beta e/r + w+\gamma_p\ve e) = \beta + \gamma_pr\ve \geq q \|\beta e/r + w+\gamma_p\ve e\| \right\}\\
     &=  \left\{\beta e/r + w  \in \R^r \ \big| \ e^{\top}w = 0, \ \  \beta + \gamma_pr\ve \geq q \| w+(\gamma_p\ve + \beta /r)e\| \right\}\\
       &=  \left\{\beta e/r + w  \in \R^r \ \big| \ e^{\top}w = 0, \ \ 
      ( \beta + \gamma_pr\ve)^2 \geq q^2 (\| w\|^2 +\|(\gamma_p\ve + \beta /r)e\|^2) \right\}\\
         &=  \left\{\beta e/r + w  \in \R^r \ \big| \ e^{\top}w = 0, \ \ 
     \| w\|^2 \le  ( \beta + \gamma_pr\ve)^2 \left( \frac{1}{q^2} -   \frac 1r\right)
       \right\}, 
\end{align*}
and 
\begin{align*}
    \mac B &=  \mac H_\beta \cap \mac B_{1-\vf_p\sqrt \ve}
    =  \left\{x  \in \R^r \ \big| \ e^{\top}x = \beta,\ \ \|x\|< 1-\vf_p\sqrt \ve \right\}\\
     &=  \left\{\beta e/r + w  \in \R^r \ \big| \ e^{\top}(\beta e/r + w) = \beta,\ \ \|\beta e/r + w\|< 1-\vf_p\sqrt \ve \right\}\\
      &=  \left\{\beta e/r + w  \in \R^r \ \big| \ e^{\top}w = 0,\ \ \|\beta e/r\|^2 + \|w\|^2< (1-\vf_p\sqrt \ve)^2 \right\}\\
      &=  \left\{\beta e/r + w  \in \R^r \ \big| \ e^{\top}w = 0,\ \ 
     \|w\|^2< (1-\vf_p\sqrt \ve)^2 -  \beta^2/r  \right\}.
\end{align*}
Moreover, if $\wt e_{i,j} = (\wt e_i + \wt e_j)/2$ with $i\ne j$, then 
\begin{align*}
    \|\wt e_{i,j}-  \beta e/r\|^2 &= \| (\beta + \gamma_p\ve r) (e_i+e_j)/2 - \gamma_p\ve e  -  \beta e/r\| ^2\\
    &= (\beta + \gamma_p\ve r)^2/2 + r(\gamma_p\ve   +  \beta /r)^2 - 2(\beta + \gamma_p\ve r)(\gamma_p\ve   +  \beta /r)\\
    &=\beta^2/2 + \beta^2/r - 2 \beta^2/r + \ve\left(  \beta\gamma_p r + 2\beta\gamma_p - 4\gamma_p\beta   \right)+ \ve^2 \left(  \gamma_p^2r^2/2 + r\gamma_p^2 -2\gamma_p^2r  \right)
   \\
     &=\beta^2\left(\frac 12 - \frac 1r\right) + \ve \beta\gamma_p \left(   r - 2  \right)+ \ve^2  \gamma_p^2 r \left(  \frac r2 - 1  \right) = 
     \frac{r-2}{2r}(\beta + r\gamma_p\ve)^2.
     \end{align*}
\end{proof}

First, let us prove that in order to ensure that  $\wt{\mac S}$ and $\wt e_{i,j}$ are all contained in $\mac B$, we need that the perturbation $\ve$ must depend on $q-1$. In fact, when $q\to 1$, 
that is, $p^2\to r-1$, we have already seen that only a very small $\ve$ allows for  $\mac P\setminus\mac B$ to be disjoint.

\begin{lemma}
    \label{lem:S_eij_are_in_B}

Under Notation~\ref{not:geometric_notations},     \[
 \sqrt\ve = \mac O\lp 
  \frac{\min\{q,\sqrt 2\} -1}{\sqrt{\frac{ r^{7/2}}{\sigma_r(W^\#)}\frac {p^2}{q^2}}} 
  \rp 
\implies \conv(\{\wt e_{i,j}\}_{i\ne j}, \mac {\wt S}) \cu \mac B, 
\]
and, in particular, if $\beta = e^\top \wt r_k$, then
\[
 (1-\vf_p\sqrt \ve)^2 -  \beta^2/r\ge ( \beta + \gamma_pr\ve)^2 \left( \frac{1}{\min\{q^2,2\}} -   \frac 1r\right). 
\]
\end{lemma}
\begin{proof}
    By Lemma~\ref{lem:S_B_and_eij}, $\conv(\{\wt e_{i,j}\}_{i\ne j}, \mac {\wt S}) \cu \mac B$ if and only if  
     \[
    ( \beta + \gamma_pr\ve)^2 
    \max \left\{ \left( \frac{1}{q^2} -   \frac 1r\right),
    \left( \frac{1}{2} -   \frac 1r\right)\right\}
    +
    \beta^2/r
    \le (1-\vf_p\sqrt \ve)^2  .
    \]
   As the left hand side is increasing in $\beta$, and the relation must hold for any $\beta$ such that $|1-\beta| = |1-e^\top \wt r_k| \le \vf_p\sqrt \ve$, we can substitute $\beta = 1 + \vf_p\sqrt \ve$ to obtain  
\begin{align*}
          ( 1 + \vf_p\sqrt \ve + \gamma_pr\ve)^2 
    \max \left\{ \left( \frac{1}{q^2} -   \frac 1r\right),
    \left( \frac{1}{2} -   \frac 1r\right)\right\}
    +
    \frac{    (1 + \vf_p\sqrt \ve)^2}r
    &\le (1-\vf_p\sqrt \ve)^2,
    \\
     \frac{( 1 + \vf_p\sqrt \ve)^2}{\min\{q^2,2\}} + 
    {(2 + 2\vf_p\sqrt \ve + \gamma_pr\ve)\gamma_pr\ve}
    \left( \frac{1}{\min\{q^2,2\}} -   \frac 1r\right)
    &\le (1-\vf_p\sqrt \ve)^2.
      \end{align*}
From the bound on $\ve$ in Notation~\ref{not:geometric_notations}, we get $\vf_p\sqrt\ve = \mac O(1/\sqrt r)$,  $\gamma_pr\ve = \mac O(1/r^2)$ and $\vf_p^2 = \mac O(\gamma_p r^2)$, so $\vf_p^2 \ve = \mac O(\vf_p\sqrt\ve /\sqrt r)$ and 
we can isolate all the contributions of the order $\ve$ or larger as
\begin{align*}
    \frac{ 1 + 2\vf_p\sqrt \ve}{\min\{q^2,2\}}  + \frac{\vf_p}{\sqrt r} \mac O(\sqrt \ve)
       \le 1-2\vf_p\sqrt \ve
     &  \iff
        \mac O(\vf_p)\sqrt\ve\le \min\{q^2,2\} -1.
       \end{align*}
Since $\min\{q,\sqrt 2\} -1 = \mac O(\min\{q^2,2\} -1)$, 
\[
  \sqrt\ve = \mac O\lp 
  \frac{\min\{q,\sqrt 2\} -1}{\sqrt{\frac{ r^{7/2}}{\sigma_r(W^\#)}\frac {p^2}{q^2}}} 
  \rp \implies 
   \mac O(\vf_p)\sqrt\ve\le \min\{q^2,2\} -1.
\]
\end{proof}

\subsubsection{Decomposition of $\mac P\setminus\mac B$}

\begin{figure}
    \centering
     \includegraphics[width=\linewidth]{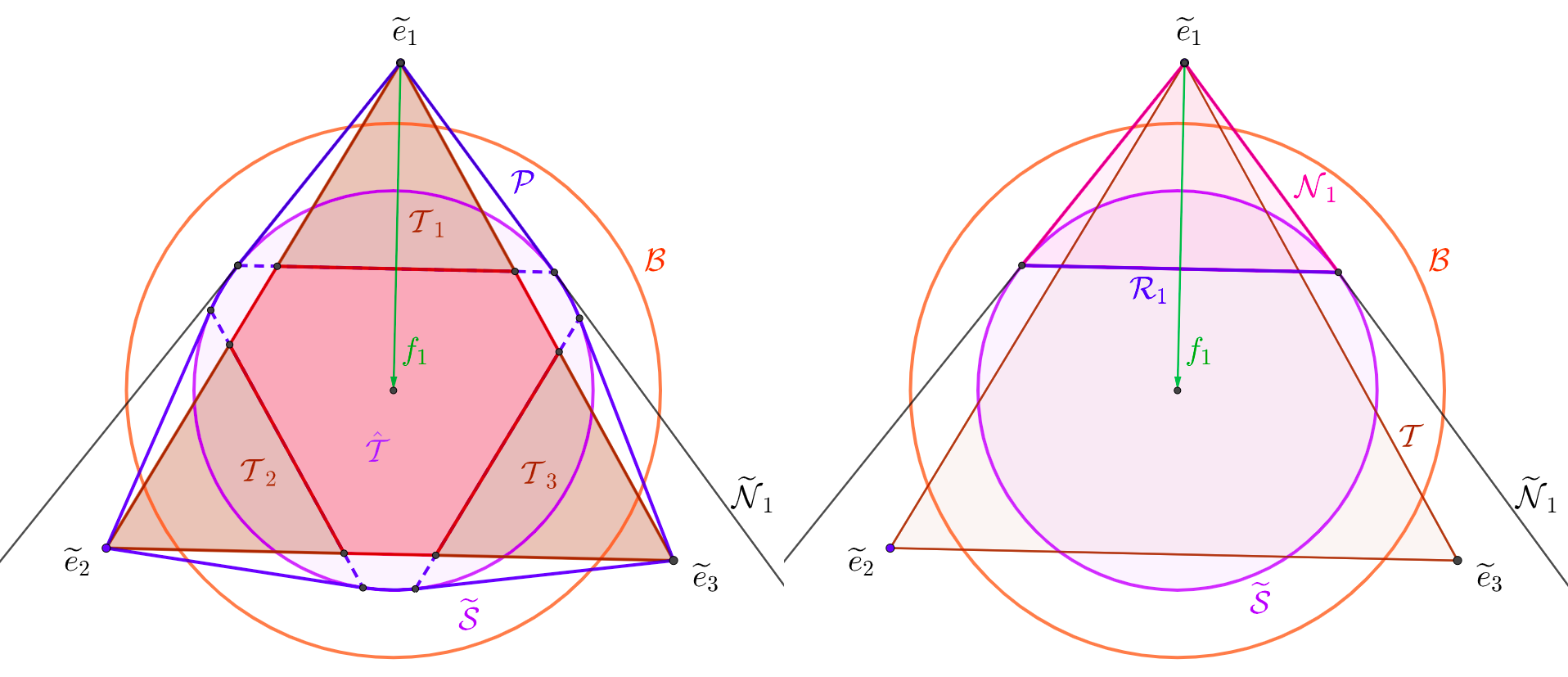}
    \caption{Visualization of the spaces introduced in Notation~\ref{not:geometric_notations2}: $\mac T$, $\wt  {\mac N_i}$, ${\mac N_i}$, ${\mac R_i}$, ${\mac T_i}$, $\hat {\mac T}$.}
    \label{fig:geom2}
\end{figure}

\begin{notation}
    \label{not:geometric_notations2} 
Consider the Notation~\ref{not:geometric_notations}, and let $f_i = \beta e/r-\wt e_i $ where $\beta= e^\top \wt r_k$. If
\[
\alpha: = (\beta + \gamma_p\ve r) \sqrt{1-\frac 1{\min\{q^2,2\}}},
\quad 
 \sqrt\ve = \mac O\lp 
 \lp \min\{q,\sqrt 2\} -1\rp
 \sqrt{\frac{\sigma_r(W^\#)}{ r^{7/2}}\frac {q^2}{p^2}}
  \rp , 
\]
then we define the following spaces, illustrated on Figure~\ref{fig:geom2} for $r=3$.     
\begin{itemize}
    \item  Call $\mac T:= \conv(\{\wt e_{i}\}_{i})$ the convex polyhedra with vertices $\wt e_i$.  Notice that $\mac P = \conv(\mac T,\mac {\wt S})$.
    
     \item The cone $\wt  {\mac N_i}$ is the ice-cream cone with vertex in $\wt e_i$ and central axis in $f_i$.
It is defined so that it is the smallest such cone containing $\mac P$:
\[
    \wt  {\mac N_i} := \left\{  v\in\mac H_\beta \ | \    f_i^\top (v-\wt e_i) \ge \alpha\|v-\wt e_i\|   \right\}. 
    \]
    
    \item The truncated cone $\mac N_i$ is the truncation of $ \wt  {\mac N_i} $ either at the tangency point with $\mac {\wt S}$ or in the $\wt e_{i,j}$, depending on if $q^2$ is larger than $2$ or not: 
\[
    {\mac N_i} := \left\{  v\in\mac H_\beta \ | \    \alpha^2\ge f_i^\top (v-\wt e_i) \ge \alpha\|v-\wt e_i\|   \right\}. 
    \]
    
    \item The space $\mac R_i$ is the part of $  {\mac N_i} $ at the maximum distance from $\wt e_i$: 
\[
          {\mac R_i} := \left\{  v\in\mac H_\beta \ | \    \alpha^2 =  f_i^\top (v-\wt e_i) \ge \alpha\|v-\wt e_i\|   \right\}. 
          \]
          
    \item The space $\mac T_i$ is the part of $  {\mac T} $ close to $\wt e_i$ delimitated by $\mac R_i$: 
\[
          {\mac T_i} := \left\{  v\in\mac T \ | \    \alpha^2\ge f_i^\top (v-\wt e_i)  \right\}. 
          \]
          
\item The space $\hat {\mac T}$ contains the points of $\mac T$ not belonging to any ${\mac T}_i $: 
\[
\hat {\mac T} := \mac T\setminus \left( \cup_i \mac T_i\right). 
\]
\end{itemize} 
\end{notation}

Let us now prove the following properties of the spaces introduced above.

\begin{lemma}
    \label{lem:cones_inclusions} Using Notation~\ref{not:geometric_notations2}, for every $i\ne j$, we have 
    \[
    \|f_i\|^2 =   \frac {r-1}r (\beta + \gamma_p\ve r)^2, \qquad 
    f_i^\top f_j =   -\frac 1r(\beta + \gamma_p\ve r)^2, 
    \]
    and
    \begin{enumerate}
        \item $\mac P\cu  \wt  {\mac N_i}$,
        \item  $\mac R_i\cu \mac B\cap \mac N_i$,
        \item ${\mac T_i} = \mac T\cap \mac N_i$, 
        \item $\hat{\mac T} =  \left\{ v\in\mac T\ |\ v=\sum_i\lambda_i \wt e_i, \ \sum_i\lambda_i = 1, \ 0\le \min_i\lambda_i \le \max_i\lambda_i < \frac{1}{\min\{ q^2,2\}} \right\}  \cu\mac B$.
    \end{enumerate}
\end{lemma}
\begin{proof}
First of all, let us analyze the vector $f_i$. We have $f_i = \beta e/r-\wt e_i = (\beta + \gamma_p\ve r)( e/r-e_i)$. As a consequence,
\[\|f_i\|^2 =    \| (\beta + \gamma_p\ve r)(e_i-  e/r) \| ^2 = \frac {r-1}r (\beta + \gamma_p\ve r)^2\]
and, for any $j\ne i$, 
\[
f_i^\top f_j =   (\beta + \gamma_p\ve r)^2(e_i-  e/r)^\top (e_j-  e/r) = -\frac 1r(\beta + \gamma_p\ve r)^2.
\]

    \textit{1.}   Let us prove that both $\mac T$ and $\mac {\wt S}$ are contained in $ \wt  {\mac N_i}$. From the definition, $\wt e_i\in  \wt  {\mac N_i}$ follows easily.  For $j\ne i$,
    \[
    f_i^\top (\wt e_j- \wt e_i) = f_i^\top (f_i-f_j) = \frac {r-1}r (\beta + \gamma_p\ve r)^2 + \frac 1r(\beta + \gamma_p\ve r)^2 = (\beta + \gamma_p\ve r)^2,
    \]
    \[
    \alpha^2 \|\wt e_j -\wt e_i\|^2 =\alpha^2 \|f_i -f_j\|^2 =2\alpha^2
    \lp
     \frac {r-1}r (\beta + \gamma_p\ve r)^2 +  \frac {1}r (\beta + \gamma_p\ve r)^2
    \rp
    = 2\alpha^2(\beta + \gamma_p\ve r)^2, 
    \]
    but since $2\alpha^2 \le (\beta + \gamma_p\ve r)^2$, we find that $\wt e_j\in  \wt  {\mac N_i}$, so $\mac T\cu  \wt  {\mac N_i}$. Recall from Lemma~\ref{lem:S_B_and_eij} that $v\in \mac {\wt S}$ can also be expressed as $v = \beta e/r +w$ where $e^\top w = 0$ and $ \|w\|^2\le ( \beta + \gamma_pr\ve)^2 \left( \frac{1}{q^2} -   \frac 1r\right)$.  As a consequence, $v\in \wt{\mac N_i}$ as
    \begin{align*}
        [f_i^\top( \wt e_i - w-\beta e/r )]^2 &= [f_i^\top (f_i+w)]^2 = (\|f_i\|^2 + f_i^\top w)^2\\
        &= \|f_i\|^2 (\|f_i\|^2 +2f_i^\top w_i) + (\|w\|^2 + f_i^\top w)^2 - \|w\|^2 (\|w\|^2 + 2f_i^\top w_i)\\
         &=(\|w\|^2 + f_i^\top w)^2 +  
         (\|f_i\|^2 - \|w\|^2)(\|f_i\|^2 + \|w\|^2+ 2f_i^\top w_i)\\
         &\ge 
         (\|f_i\|^2 - \|w\|^2)(\|f_i\|^2 + \|w\|^2+ 2f_i^\top w_i)\\
           &\ge 
         \left[ \frac {r-1}r (\beta + \gamma_p\ve r)^2 - ( \beta + \gamma_pr\ve)^2 \left( \frac{1}{q^2} -   \frac 1r\right)\right ]
         \| f_i+w\|^2 \\
         &= 
        ( \beta + \gamma_pr\ve)^2    \left( 1- \frac{1}{q^2}\right)
        \|v-\wt e_i\|^2 \\&
        \ge 
        ( \beta + \gamma_pr\ve)^2    \left( 1- \frac{1}{\min\{q^2,2\}}\right)
        \|v-\wt e_i\|^2
        = \alpha^2\|v-\wt e_i\|^2.
    \end{align*}
    We can conclude that $\wt{\mac S}\cu \wt{\mac N_i}$. 

    \textit{2.} The relation  $\mac R_i\cu \mac N_i$  is immediate from their definition. Taken now any $v\in \mac R_i$, we have
    \begin{align*}
     \|v-\beta e/r\|^2 &= \|v-\wt e_i - f_i\|^2 
     = \|v-\wt e_i\|^2 + \|f_i\|^2 - 2f_i^\top (v-\wt e_i) 
     \le  \alpha^2 + \|f_i\|^2 - 2\alpha^2 \\
     &=  \frac {r-1}r (\beta + \gamma_p\ve r)^2  -  (\beta + \gamma_p\ve r)^2 \left( 1-\frac 1{\min\{q^2,2\}}\right)\\
     &=  (\beta + \gamma_p\ve r)^2 \left( \frac 1{\min\{q^2,2\} }- \frac {1}r\right)\le  (1-\vf_p\sqrt \ve)^2 -  \beta^2/r,
 \end{align*}
where the last relation holds due to Lemma~\ref{lem:S_eij_are_in_B}.
Since $\mac R_i \cu \mac H_\beta$ then $e^\top (v-\beta e/r) = 0$, thus proving  $\mac R_i\cu \mac B$ using Lemma~\ref{lem:S_B_and_eij}.

 \textit{3.} By definition, $\mac T\cap \mac N_i\cu \mac T_i$, and by 1.  \[
 v\in \mac T_i \implies v\in \mac T\cu  \wt  {\mac N_i}, \, \alpha^2\ge f_i^\top (v-\wt e_i) 
 \implies 
 f_i^\top (v-\wt e_i) \ge \alpha\|v-\wt e_i\| , \, \alpha^2\ge f_i^\top (v-\wt e_i) \implies 
  v\in \mac N_i, 
 \]
 thus proving that $\mac T_i\cu \mac T\cap \mac N_i$.

 \textit{4.} If now $v\in \hat {\mac T}$, then $ v=\sum_i\lambda_i \wt e_i$ and $v\not \in \mac T_i$ for every $i$, where $\lambda_i \ge 0$ and $\sum_i\lambda_i =1$. In particular, $ v=\sum_i\lambda_i \wt e_i \implies \beta e/r-v = \sum_i\lambda_i f_i $  and since the quantity $f_i^\top f_j$ is the same for any $j\ne i$, 
 \begin{align*}
     v\not \in \mac T_i \iff \alpha^2&< f_i^\top (v-\wt e_i) = f_i^\top(f_i +(v-\beta e/r))
     = \|f_i\|^2 -  f_i^\top \sum_j\lambda_j f_j\\
     &= (1-\lambda_i)\|f_i\|^2 - (1-\lambda_i) f_i^\top f_j  = (1-\lambda_i) f_i^\top(f_i - f_j) 
     =  (1-\lambda_i) (\beta + \gamma_p\ve r)^2,\\
        v\not \in \mac T_i \iff \lambda_i &< 1 - \frac{\alpha^2}{(\beta + \gamma_p\ve r)^2 } = 
   \frac{1}{\min\{q^2,2\}}.
 \end{align*}
This proves that 
\[
\hat{\mac T} =  \left\{ v\in\mac T\ \Big|\ v=\sum_i\lambda_i \wt e_i, \ \sum_i\lambda_i = 1, \ 0\le \min_i\lambda_i \le \max_i\lambda_i < \frac{1}{\min\{ q^2,2\}} \right\}.
\]
Call $\ol \lambda:= \frac{1}{\min\{q^2,2\}}$. Since $\ol \lambda\ge 1/2$,  the closure of $\hat{\mac T}$ is a polyhedral convex set with vertices being $ v=\sum_i\lambda_i \wt e_i $ with one $\lambda_i$ equal to $ \ol \lambda$, one equal to $ 1- \ol\lambda$ and all the rest equal to zero. As a consequence, the convex function $\|x- \beta e/r\|^2$ has maximum on $\hat{\mac T}$ in exactly one of its vertices. Using that $1\ge \ol \lambda\ge 1/2$ and Lemma~\ref{lem:S_B_and_eij}, we conclude that
 \begin{align*}
     v\in \hat{\mac T}\implies \|v-\beta e/r\|^2 &\le \left\| \ol\lambda\wt e_i  + (1-\ol\lambda)\wt e_j  -\beta e/r\right\|^2 = 
     \left\| \ol\lambda f_i   + (1-\ol\lambda)f_j \right\|^2\\& = 
     \ol\lambda^2\|f_i\|^2 + (1-\ol\lambda)^2 \|f_i\|^2 + 
    2(1-\ol\lambda)\ol\lambda f_i^\top f_j
    \\& = 
     (\ol\lambda^2+ (1-\ol\lambda)^2) (\|f_i\|^2- f_i^\top f_j)+ 
     f_i^\top f_j 
    \\& = 
   \left(   \ol  \lambda +
   2\ol\lambda^2 -3\ol\lambda + 1 - \frac 1r\right) (\beta + \gamma_p\ve r)^2 
   \\& = 
   \left(   \ol  \lambda - (2\ol \lambda -1)(1-\ol \lambda)  - \frac 1r\right) (\beta + \gamma_p\ve r)^2 
   \\& \le 
   \left(   \ol  \lambda  - \frac 1r\right) (\beta + \gamma_p\ve r)^2 \le  (1-\vf_p\sqrt \ve)^2 -  \beta^2/r.
 \end{align*}
 This is enough to show that  $ \hat {\mac T}\cu \mac B$.
    
\end{proof}

We can finally prove that we can decompose $\mac P\setminus\mac B$ as the union of sets $\mac P_i$, each one contained in $ \mac N_i\setminus\mac B$. 


\begin{theorem}
\label{th:decomposition_of_P/B}    Under Notation~\ref{not:geometric_notations} and Notation~\ref{not:geometric_notations2},    
\[
\mac P\setminus\mac B  
\cu \cup_i (\mac N_i \setminus \mac B).
\]

\end{theorem}
\begin{proof}
    
We want to show that  $\mac P \cu \mac B  \cup  (\cup_i  \mac N_i)$, so that $\mac P\setminus\mac B  
\cu \cup_i (\mac N_i \setminus \mac B)$. Take a vector $v\in \mac P\setminus(\mac B  \cup  (\cup_i  \mac N_i) )$. Since $\mac P = \conv(\wt{\mac S} \cup \mac T)$ and both $\wt{\mac S}, \mac T$ are convex, there must exist $s\in \wt{\mac S}$ and $t\in\mac T$ such that $v = \lambda t + (1-\lambda) s$ with $0\le \lambda\le 1$. 
Since $s\in \wt{\mac S}\cu\mac B$ due to Lemma~\ref{lem:S_eij_are_in_B}, then $t\not \in \mac B$ because otherwise $v\in \mac B$. 
In particular, $t\not\in \hat{\mac T}$ because $\hat{\mac T} \cu \mac B$ due to 4. in Lemma~\ref{lem:cones_inclusions}, so $t\in \mac T\setminus\hat{\mac T} = \cup_j\mac T_j$ and there must exist an index $i$ such that $t\in \mac T_i\cu \mac N_i$.   
Since $\mac T_i\cu \mac N_i$  due to 3. in Lemma~\ref{lem:cones_inclusions}, the vector $s$ cannot belong to $\mac N_i$, otherwise $v\in \mac N_i$.  
Again, due to 1. of the same Lemma, we have $s\in \wt{\mac S}\cu\mac P\cu  \wt{\mac N_i}$, so we conclude that $t\in \mac N_i\cu \wt{\mac N_i}$ and $s\in  \wt{\mac N_i}\setminus\mac N_i$.
In particular,
    \[
    \alpha^2\ge f_i^\top (t-\wt e_i), \qquad \alpha^2< f_i^\top (s-\wt e_i),
    \]
so there exists a vector $w = \mu t + (1-\mu) s$ with $0< \mu\le 1$ such that $\alpha^2 =  f_i^\top (w-\wt e_i)$. Since $\wt{\mac N_i}$ is convex, $w\in \wt{\mac N_i}$ and finally by 2. of Lemma~\ref{lem:cones_inclusions}, $w\in \mac R_i\cu \mac B\cap \mac N_i$. We thus conclude that 
\[
v\in \conv(t,s) = \conv(t,w)\cup \conv(w,s) \cu \mac N_i \cup \mac B,
\]
 a contradiction. 
 \end{proof}

We can join the bounds on $\ve$ found  in Notation~\ref{not:geometric_notations} and Notation~\ref{not:geometric_notations2}
\[
 \sqrt\ve = \mac O\lp 
 \lp \min\{q,\sqrt 2\} -1\rp
 \sqrt{\frac{\sigma_r(W^\#)}{ r^{7/2}}\frac {q^2}{p^2}}
  \rp ,
  \qquad 
  \ve  = \mac O\lp \frac {\sigma_r(W^\#)}{r^{9/2}}\frac {q^2}{p^2}\rp,
\]
into
\[
\ve  = \mac O\lp 
\lp \min\{q,\sqrt 2\} -1\rp^2
\frac {\sigma_r(W^\#)}{r^{9/2}}\frac {q^2}{p^2}\rp.
\]
With this assumption on $\ve$ we can finally bound $\min_j\|\wt r_k-\wt e_j\|^2$.

\begin{lemma}
\label{lem:distance_rk_ei_sscp}  Using  Notation~\ref{not:geometric_notations} and Notation~\ref{not:geometric_notations2}, we have 
\[
   \ve  = \mac O\lp 
\lp \min\{q,\sqrt 2\} -1\rp^2
\frac {\sigma_r(W^\#)}{r^{9/2}}\frac {q^2}{p^2}\rp
\implies 
\min_j\|\wt r_k-\wt e_j\|^2 =  \frac { \ve}{\min\{q^2-1,1\}}
    \mac O\left(  
   \frac{ r^3\sqrt r}{\sigma_r(W^\#)}\frac {p^2}{q^2}
      \right)  .
    \]
\end{lemma}
\begin{proof}
    From Notation~\ref{not:geometric_notations}, $\wt r_k\in\mac P\setminus\mac B$ and, by Theorem~\ref{th:decomposition_of_P/B}, there exists $i$ such that $\wt r_k \in \mac N_i\setminus \mac B$, and  
    \[
\min_j\|\wt r_k-\wt e_j\|^2\le  \|\wt r_k-\wt e_i\|^2
\le 
\max_{v\in \mac N_i\setminus\mac B} \|v-\wt e_i\|^2
\qquad 
    {\mac N_i} = \left\{  v\in\mac H_\beta \ | \    \alpha^2\ge f_i^\top (v-\wt e_i) \ge \alpha\|v-\wt e_i\|   \right\}. 
    \]
For any $t$ define
 \[{\mac N_{i,t}} = \left\{  v\in\mac H_\beta \ | \    \alpha t= f_i^\top (v-\wt e_i) \ge \alpha\|v-\wt e_i\|   \right\}.\]
Notice that $\mac N_{i,t}=\emptyset$ for $t<0$, $\mac N_{i,\alpha} = \mac R_i$ and $\mac N_i = \sqcup_{0\le t\le \alpha} \mac N_{i,t}$,  so
 \[
\min_j\|\wt r_k-\wt e_j\|
\le 
\max_{v\in \mac N_i\setminus\mac B} \|v-\wt e_i\|
\le 
\max_{\alpha \ge t\ge 0 \ :\ \mac N_{i,t}\not\cu\mac B }
\max_{v\in \mac N_{i,t}} \|v-\wt e_i\|
\le 
\max_{\alpha \ge t\ge 0 \ :\ \mac N_{i,t}\not\cu\mac B }
t.
\]
Since $\mac N_{i,t}$ and $\mac B$ are both convex, the condition $\mac N_{i,t}\cu \mac B$ is equivalent to $\partial\mac N_{i,t}\cu \mac B$, that is, for every $v\in\mac H_\beta$, 
\[
 \alpha t= f_i^\top (v-\wt e_i) =  \alpha\|v-\wt e_i\| \implies \| v-\beta e/r\|^2 \le   (1-\vf_p\sqrt \ve)^2 -  \beta^2/r ,  
\]
but $v - \wt e_i = v -\beta e/ r + f_i$, so $\alpha t= f_i^\top (v-\wt e_i) =  \alpha\|v-\wt e_i\| $ coincides with the pair of conditions
\[
\begin{cases}
    \alpha t  = f_i^\top (v-\wt e_i) =  f_i^\top (v-\beta e/r) + \|f_i\|^2, \\
    t^2 = \|v-\wt e_i\|^2 = \| v-\beta e/r\|^2 + \|f_i\|^2  + 2f_i^\top(v-\beta e/r). 
\end{cases}
\]
As a consequence,
\begin{align*}
    \| v-\beta e/r\|^2 & = t^2 - \|f_i\|^2  - 2f_i^\top(v-\beta e/r)
= t^2 - 2\alpha t +  \|f_i\|^2
= (\alpha - t)^2 - \alpha^2+  \|f_i\|^2, 
\end{align*}
and   $\mac N_{i,t}\cu \mac B$ for  $0\le t\le \alpha$ if and only if
\begin{align*}
    t &\ge
    \alpha - \sqrt{(1-\vf_p\sqrt \ve)^2 - \frac {\beta^2}r + \alpha^2 - \|f_i\|^2}
= 
\alpha - \sqrt{ \alpha^2 - \lp\frac {\beta^2}r + \|f_i\|^2 - (1-\vf_p\sqrt \ve)^2 \rp},
\end{align*}
where, by Lemma~\ref{lem:cones_inclusions}  and Lemma~\ref{lem:S_eij_are_in_B}, 
\begin{align*}
    \frac {\beta^2}r + \|f_i\|^2 - (1-\vf_p\sqrt \ve)^2 &\le 
\frac {r-1}r (\beta + \gamma_p\ve r)^2 -( \beta + \gamma_pr\ve)^2 \left( \frac{1}{\min\{q^2,2\}} -   \frac 1r\right) = \alpha^2,\\
\frac {\beta^2}r + \|f_i\|^2 - (1-\vf_p\sqrt \ve)^2 &= 
\frac {\beta^2}r + \frac {r-1}r (\beta + \gamma_p\ve r)^2 - (1-\vf_p\sqrt \ve)^2
\ge 
\frac {\beta^2}r + \frac {r-1}r \beta ^2 - (1-\vf_p\sqrt \ve)^2\ge 0.
\end{align*}
In particular, since $1 - \sqrt{1-y }\le y$ for every $0\le y\le 1$, 
\[
 t \ge \frac{\frac {\beta^2}r + \|f_i\|^2 - (1-\vf_p\sqrt \ve)^2}{\alpha}
 \implies \mac N_{i,t}\cu \mac B
\quad \text{ or }\quad 
\mac N_{i,t}\not\cu \mac B \implies 
 t < \frac{\frac {\beta^2}r + \|f_i\|^2 - (1-\vf_p\sqrt \ve)^2}{\alpha}, 
\]
so that 
 \[
\min_j\|\wt r_k-\wt e_j\|
\le 
\max_{\alpha \ge t\ge 0 \ :\ \mac N_{i,t}\not\cu\mac B }
t
<
\frac{\frac {\beta^2}r + \|f_i\|^2 - (1-\vf_p\sqrt \ve)^2}{\alpha}
=
\frac
  {\frac {\beta^2}r + \frac {r-1}r (\beta + \gamma_p\ve r)^2 - (1-\vf_p\sqrt \ve)^2 }
  {(\beta + \gamma_p\ve r) \sqrt{1-\frac 1{\min\{q^2,2\}}}}
\]
\[
=
\frac
  { (\beta + \gamma_p\ve r)^2 
  -2\gamma_p\ve ( \beta + \gamma_p\ve r)
  +(\gamma_p\ve)^2 r
  - (1-\vf_p\sqrt \ve)^2 }
  {(\beta + \gamma_p\ve r)\sqrt{1-\frac 1{\min\{q^2,2\}}} }
\]\[=
\frac{1}{\sqrt{1-\frac 1{\min\{q^2,2\}}}}
\lp 
\beta + \gamma_p\ve (r - 2) 
+ \frac
  { 
  (\gamma_p\ve)^2 r
  - (1-\vf_p\sqrt \ve)^2 }
  {\beta + \gamma_p\ve r }
  \rp , 
\]
which is increasing in $\beta$ since from the bound on $\ve$ and 
Notation~\ref{not:geometric_notations}, we get $\vf_p\sqrt\ve = \mac O(1/\sqrt r)$,  $\gamma_pr\ve = \mac O(1/r^2)$, so $ (\gamma_p\ve)^2 r
  \ll (1-\vf_p\sqrt \ve)^2$.  Since $\beta \le 1 +\vf_p\sqrt \ve $, we can  substitute $\beta \le 1 +\vf_p\sqrt \ve $, and write
  \begin{align*}
      \min_j\|\wt r_k-\wt e_j\|
&<
\frac{1}{\sqrt{1-\frac 1{\min\{q^2,2\}}}}
\lp 
1 +\vf_p\sqrt \ve+ \gamma_p\ve (r - 2) 
+ \frac
  { 
  (\gamma_p\ve)^2 r
  - (1-\vf_p\sqrt \ve)^2 }
  {1 +\vf_p\sqrt \ve + \gamma_p\ve r }
  \rp \\
  &=
  \frac{1}{\sqrt{1-\frac 1{\min\{q^2,2\}}}}
\lp 
\vf_p\sqrt \ve+ \gamma_p\ve (r - 2) 
+ \frac
  { 3\vf_p\sqrt \ve - \vf_p^2 \ve + \gamma_p\ve r + 
  (\gamma_p\ve)^2 r
   }
  {1 +\vf_p\sqrt \ve + \gamma_p\ve r }
  \rp .
  \end{align*}
  Using again that $\vf_p\sqrt\ve = \mac O(1/\sqrt r)$,  $\gamma_pr\ve = \mac O(1/r^2)$,  $\vf_p\sqrt \ve = \mac O(\sqrt {\gamma_p\ve } r)$, so that $\gamma_p\ve r = \mac O(\sqrt{\gamma_p\ve }/ \sqrt r)$, $\vf_p^2 \ve = \mac O(\vf_p\sqrt\ve /\sqrt r)$, and
    \begin{align*}
      \min_j\|\wt r_k-\wt e_j\|
  &<
  \frac{1}{\sqrt{1-\frac 1{\min\{q^2,2\}}}}
\lp 
\mac O(\sqrt {\gamma_p\ve } r)+ 
\mac O(\sqrt{\gamma_p\ve }/ \sqrt r) 
+ \frac
  {  \mac O(\sqrt {\gamma_p\ve } r) + \mac O(\sqrt{\gamma_p\ve }/ \sqrt r)  + 
  \mac O(\sqrt{\gamma_p\ve }/ r^{7/2}) 
   }
  {1 +\mac O(1/\sqrt r) }
  \rp \\&
  =
   \frac{\mac O(\sqrt {\gamma_p\ve } r)}{\sqrt{1-\frac 1{\min\{q^2,2\}}}}
   =
   \mac O\left( \sqrt{ 
   \frac{ r^3\sqrt r}{\sigma_r(W^\#)}\frac {p^2}{q^2}
   }
   \right)
   \frac{1}{\sqrt{1-\frac 1{\min\{q^2,2\}}}}
   \sqrt \ve\\&
   \le 
   \frac 1{\sqrt{\min\{q^2-1,1\}}}
    \mac O\left( \sqrt{ 
   \frac{ r^3\sqrt r}{\sigma_r(W^\#)}\frac {p^2}{q^2}
   }
   \right)
   \sqrt \ve .
  \end{align*}
 \end{proof}

\subsection{Last steps of the proof}
\label{app:proof_of_th_main}

To prove that two different $\wt r_k$  cannot be close to the same $\wt e_j$, it is sufficient to use the lower bound on $\|\wt r_i-\wt r_j\|$ given by Lemma~\ref{lem:R_is_close_to_orthogonal}.

\begin{corollary}
\label{cor:distance_to_permutation_sscp}     
Using  Notation~\ref{not:geometric_notations} and Notation~\ref{not:geometric_notations2}, if 
     \[
     \ve  = \mac O\lp 
\lp \min\{q,\sqrt 2\} -1\rp^2
\frac {\sigma_r(W^\#)}{r^{9/2}}\frac {q^2}{p^2}\rp, 
     \]
     then there exists a permutation matrix $\Pi\in \f R^{r\times r}$ such that
     \[
     \|R-\Pi\|_{1,2} \le \mac O\left(  \sqrt {
     \frac { \ve}{\min\{q^2-1,1\}}
   \frac{ r^{7/2}}{\sigma_r(W^\#)}\frac {p^2}{q^2}
   }
      \right).
     \]
\end{corollary}
\begin{proof}
By Lemma~\ref{lem:R_is_close_to_orthogonal}, since  $\ve  = \mac O(\frac {\sigma_r(W^\#)}{r^{9/2}}\frac {q^2}{p^2})$, we have 
\[
    \min_{i\ne j} \|\wt r_i-\wt r_j\| \ge  1 - \mac O\lp   \sqrt{\frac{ r^{7/2}}{\sigma_r(W^\#)}\frac {p^2}{q^2}\ve}\rp .
    \]
At the same time, if  $\wt r_i\ne \wt r_j$ are close to the same $\wt e_k$ according to Lemma~\ref{lem:distance_rk_ei_sscp}, then
\[
\|\wt r_i-\wt r_j\|\le \|\wt r_i-\wt e_k\| + \|\wt r_j-\wt e_k\|\le
    \mac O\left(  \sqrt {
     \frac { \ve}{\min\{q^2-1,1\}}
   \frac{ r^{7/2}}{\sigma_r(W^\#)}\frac {p^2}{q^2}
   }
      \right), 
\]
which is impossible. As a consequence, each $\wt r_i$ is close to a different $\wt e_k$ and to the associated $e_k$ as
\begin{align*}
    \|\wt e_k-e_k\|& = \| \lp e^\top \wt r_k + \gamma_p\ve r -1\rp e_k - \gamma_p\ve e \|
    \le |e^\top \wt r_k -1| + 2\gamma_p\ve r \le \vf_p\sqrt\ve + 2\gamma_p\ve r\\
    &\le \mac O\lp \sqrt {\gamma_p\ve}r  \rp +\mac O(\sqrt {\gamma_p\ve}/\sqrt r)
    =\mac O\lp  \sqrt{\frac{ r^{7/2}}{\sigma_r(W^\#)}\frac {p^2}{q^2}\ve }\rp, 
\end{align*}
and therefore
\begin{align*}
    \|\wt r_i-e_k\| \le  \|\wt r_i-\wt e_k\|  +  \|\wt e_k -e_k\|  \le \mac O\left(  \sqrt {
     \frac { \ve}{\min\{q^2-1,1\}}
   \frac{ r^{7/2}}{\sigma_r(W^\#)}\frac {p^2}{q^2}
   }
      \right).
\end{align*}
In particular, there must exists a permutation matrix $\Pi\in \f R^{r\times r}$ such that 
   \[
     \|R-\Pi\|_{1,2} \le 
     \max_k \min_j\|\wt r_k - \wt e_j\|
     \le
     \mac O\left(  \sqrt {
     \frac { \ve}{\min\{q^2-1,1\}}
   \frac{ r^{7/2}}{\sigma_r(W^\#)}\frac {p^2}{q^2}
   }
      \right).
     \] 
\end{proof}

Due to Lemma~\ref{lem:norm_inequalities}, Lemma~\ref{lem:definition_of_R} and Corollary~\ref{cor:distance_to_permutation_sscp}, we get
\begin{align*}   \min_\Pi \|W^\#-W^*\Pi\|_{1,2} &\le  \min_\Pi \left( \|W^*\|\|R-\Pi\|_{1,2}  + \|M\|_{1,2}   \right) \\
    &\le    \|W^*\| \cdot  \mac O\left(  \sqrt {
     \frac { \ve}{\min\{q^2-1,1\}}
   \frac{ r^{7/2}}{\sigma_r(W^\#)}\frac {p^2}{q^2}
   }
      \right)  +4r\ve.   
\end{align*}
If $\Pi$ is the permutation matrix satisfying Corollary~\ref{cor:distance_to_permutation_sscp}, then  Lemma~\ref{lem:norm_inequalities} and Lemma~\ref{lem:definition_of_R} show that
\begin{align*}
  \|W^\#\|&\le \|W^*\|\|R\| + \|M\| 
  \le \sqrt r\|W^*\|\|R\|_{1,2} + \sqrt r\|M\|_{1,2}  \\
  & \le \sqrt r\|W^*\|(1 + \|R-\Pi\|_{1,2}) + 2r\sqrt r\ve 
  = \mac O(\sqrt r)\|W^*\| + \mac O(\sigma_r(W^\#))\\
  \implies \frac{\|W^*\|}{\sigma_r(W^\#)} 
  & \ge 
  \frac 1{\mac O(\sqrt r)}
  \left(\frac{\|W^\#\|}{\sigma_r(W^\#)} - \mac O(1)\right)
  = \Omega\left( \frac1{\sqrt r} \right)\\
   \implies 4r\ve &\le 4r\sqrt r \ve \cdot \mac O\left(  \frac{\|W^*\|}{\sigma_r(W^\#)} \right) = 
   \|W^*\|\cdot \mac O\left(     \frac{r\sqrt r}{\sigma_r(W^\#)}  \ve \right) 
   \\ & 
   \hspace{3.9cm} =
    \|W^*\| \cdot  \mac O\left(  \sqrt {
     \frac { \ve}{\min\{q^2-1,1\}}
   \frac{ r^{7/2}}{\sigma_r(W^\#)}\frac {p^2}{q^2}
   }
      \right), 
\end{align*}
so that 
\[
 \min_\Pi \|W^\#-W^*\Pi\|_{1,2} \le     \|W^*\| \cdot  \mac O\left(  \sqrt {
     \frac { \ve}{\min\{q^2-1,1\}}
   \frac{ r^{7/2}}{\sigma_r(W^\#)}\frac {p^2}{q^2}
   }
      \right).
\]
Notice that from Lemma~\ref{lem:definition_of_R} , $ W^* = R^{-1}(W^{\#} - M) $, and from above $\|M\|  =  \mac O(\sigma_r(W^\#))$, so
\begin{align*}
   \|W^*\| \le \|R^{-1}\| \lp \| W^{\#}\| + \|M\|\rp  = 
   \frac{1}{\sigma_r(R)} \lp \| W^{\#}\| +   \mac O(\sigma_r(W^\#)) \rp
   =  \mac O\lp   \frac{\| W^{\#}\|}{\sigma_r(R)}\rp
\end{align*}
but $R$ is close to the permutation matrix $\Pi$, so
\[
\sigma_r(R) \ge \sigma_r(\Pi) - \|\Pi - R\| \ge 1 -  \mac O\left( \sqrt r \sqrt {
     \frac { \ve}{\min\{q^2-1,1\}}
   \frac{ r^{7/2}}{\sigma_r(W^\#)}\frac {p^2}{q^2}
   }
      \right) = 1 - \mac O(1) = \Omega(1), 
\]
leading to $\|W^*\| = \mac O(\|W^\#\|)$ and finally to

\[
 \min_\Pi \|W^\#-W^*\Pi\|_{1,2} \le     \|W^\#\| \cdot  \mac O\left(  \sqrt {
     \frac { \ve}{\min\{q^2-1,1\}}
   \frac{ r^{7/2}}{\sigma_r(W^\#)}\frac {p^2}{q^2}
   }
      \right).
\]


\end{document}